\documentclass{article}
\usepackage{style/unites}
\AtBeginDocument{\fontsize{12}{14}\selectfont}

\usepackage[utf8]{inputenc}
\usepackage[T1]{fontenc}
\usepackage{microtype}

\usepackage{amsmath}
\usepackage{amssymb}
\usepackage{amsfonts}
\usepackage{amsthm}
\usepackage{mathtools}
\usepackage{mathrsfs}
\usepackage{physics}
\usepackage{braket}
\usepackage{slashed}
\usepackage{nicefrac}
\usepackage{textcomp}
\usepackage{dsfont}
\usepackage{bbm}
\usepackage{bm}

\usepackage{graphicx}
\usepackage{subcaption}
\usepackage[export]{adjustbox}
\usepackage{float}
\usepackage{booktabs}
\usepackage{dcolumn}
\newcolumntype{d}[1]{D{.}{.}{#1}}
\usepackage{bigstrut, tabularx, multirow, makecell, diagbox}
\usepackage{colortbl}
\usepackage{tabularray}
\UseTblrLibrary{booktabs}
\usepackage{threeparttable}
\usepackage{tablefootnote}
\usepackage{fontawesome5}

\usepackage{placeins}
\usepackage{caption}
\usepackage{footnote}
\usepackage{enumitem}
\usepackage{multicol}
\usepackage{xspace}
\usepackage{titletoc}
\usepackage{titlesec}
\usepackage[bottom]{footmisc}
\usepackage{setspace}

\usepackage{wrapfig}
\usepackage{tikz}
\usepackage{quantikz}
\usepackage{dashbox}
\usepackage{mdframed}
\usepackage{marvosym}
\usepackage{pifont}
\usepackage{CJK}
\usepackage{url}

\usepackage[table,x11names]{xcolor}
\usepackage[most]{tcolorbox}
\tcbuselibrary{breakable}
\usetikzlibrary{decorations.pathreplacing, fit}

\definecolor{primaryblue}{HTML}{0066CC}
\definecolor{accentcyan}{HTML}{00D4AA}
\definecolor{warmorange}{HTML}{FF6B35}
\definecolor{deepgray}{HTML}{2C3E50}
\definecolor{lightgray}{HTML}{F8F9FA}
\definecolor{gradientstart}{HTML}{667eea}
\definecolor{gradientend}{HTML}{764ba2}

\definecolor{citecolor}{HTML}{0071bc}
\definecolor{citeblue}{RGB}{0, 113, 188}
\definecolor{linkcolor}{HTML}{9A4D92}
\definecolor{firebrick}{rgb}{0.698,0.133,0.133}

\definecolor{paleviolet}{HTML}{E1EEFC}
\definecolor{CarolinaUltraLight}{HTML}{E7F4FC}
\definecolor{lightgrey}{RGB}{247, 247, 247}
\definecolor{shadecolor}{HTML}{EFEFEF}
\definecolor{lightyellow}{rgb}{1.0, 0.95, 0.7}
\definecolor{lightblue}{rgb}{0.90, 0.95, 1.0}
\definecolor{light-gray}{gray}{0.95}

\definecolor{darkgrey}{rgb}{0.5, 0.5, 0.5}
\definecolor{darkgreen}{rgb}{0, 0.5, 0}
\definecolor{mydarkblue}{rgb}{0,0.08,0.45}
\definecolor{mydarkblue2}{rgb}{0.133, 0.133, 0.698}
\definecolor{echodrk}{HTML}{0099cc}
\definecolor{mymauve}{rgb}{0.58,0,0.82}
\definecolor{midnightblue}{rgb}{0.1,0.1,0.44}
\definecolor{oxfordblue}{rgb}{0.0,0.13,0.28}
\definecolor{prussianblue}{rgb}{0.0,0.19,0.33}
\definecolor{coolteal}{rgb}{0, 0.45, 0.45}
\definecolor{olive}{rgb}{0.1, 0.3, 0}
\definecolor{mypurple}{rgb}{0.5,0,0.5}
\definecolor{almond}{rgb}{0.94, 0.87, 0.8}

\definecolor{blue_ampEncoding}{HTML}{DAE8FC}
\definecolor{green_encoder}{HTML}{D5E8D4}
\definecolor{purple_decoder}{HTML}{E1D5E7}
\definecolor{yellow_measure}{HTML}{FFF2CC}
\definecolor{gray_block}{HTML}{F5F5F5}
\definecolor{pink_dru}{HTML}{FAD9D5}
\definecolor{orange_v}{HTML}{FAD7AC}

\definecolor{colorA}{rgb}{1,0,0}
\definecolor{colorB}{rgb}{0,0.3,1}
\definecolor{colorC}{rgb}{0.9,0.8,0.2}
\definecolor{colorD}{rgb}{0,0.65,0}
\definecolor{lesslightgray}{rgb}{0.5,0.5,0.5}
\definecolor{fundamental}{RGB}{55, 110, 111}
\definecolor{Gred}{RGB}{219, 50, 54}
\definecolor{ToCgreen}{RGB}{0, 128, 0}
\definecolor{Sepia}{RGB}{112, 66, 20}
\definecolor{Dblue}{rgb}{0,0.08,0.45}
\definecolor{Blue}{rgb}{0, 0, 0.8}
\definecolor{blue}{rgb}{0,0,1}
\definecolor{UNCblue!10}{rgb}{0.84,0.91,0.98}
\definecolor{RowAlt}{rgb}{0.98,0.98,0.99}

\definecolor{CarolinaBlue}{HTML}{7BAFD4}        %
\definecolor{CarolinaLightBlue}{HTML}{B3D4E5}   %
\definecolor{CarolinaUltraLight}{HTML}{E8F4F8}  %
\definecolor{CarolinaText}{HTML}{1C2B33}        %

\usepackage[pagebackref=true,breaklinks=true,colorlinks,hyperfootnotes=false]{hyperref}
\hypersetup{
  colorlinks,
  citecolor=citeblue,
  linkcolor=firebrick,
  urlcolor=firebrick
}
\usepackage[nameinlink,capitalize,noabbrev]{cleveref}

\titlespacing\section{0pt}{4pt plus 4pt minus 2pt}{-2pt plus 2pt minus 2pt}
\titlespacing\subsection{0pt}{2pt plus 4pt minus 2pt}{-2pt plus 2pt minus 2pt}
\titlespacing\subsubsection{0pt}{2pt plus 4pt minus 2pt}{-2pt plus 2pt minus 2pt}

\makeatletter
\def\th@remark{%
  \thm@headfont{\bfseries}%
  \normalfont %
  \thm@preskip\topsep \divide\thm@preskip\tw@
  \thm@postskip\thm@preskip
}
\makeatother

\theoremstyle{definition}

\newtheorem{theorem}{Theorem}[section]
\tcolorboxenvironment{theorem}{
  breakable,
  colback=black!10,
  colframe=white,
  width=\linewidth, 
  enlarge left by=0pt,
  enlarge right by=0pt,
  boxsep=5pt,
  boxrule=0pt,
  left=0pt,right=0pt,top=0pt,bottom=0pt,
  arc=8pt,
  before skip=\topsep,
  after skip=\topsep
}

\newtheorem{lemma}{Lemma}[section]
\tcolorboxenvironment{lemma}{
  breakable,
  colback=black!10,
  colframe=white,
  width=\linewidth,
  enlarge left by=0pt,
  enlarge right by=0pt,
  boxsep=5pt,
  boxrule=0pt,
  left=0pt,right=0pt,top=0pt,bottom=0pt,
  arc=8pt,
  before skip=\topsep,
  after skip=\topsep
}

\tcolorboxenvironment{corollary}{
  breakable,
  colback=black!10,
  colframe=white,
  width=\linewidth,
  enlarge left by=0pt,
  enlarge right by=0pt,
  boxsep=5pt,
  boxrule=0pt,
  left=0pt,right=0pt,top=0pt,bottom=0pt,
  arc=8pt,
  before skip=\topsep,
  after skip=\topsep
}

\newtheorem{proposition}{Proposition}[section]
\tcolorboxenvironment{proposition}{
  breakable,
  colback=black!10,
  colframe=white,
  width=\linewidth,
  enlarge left by=0pt,
  enlarge right by=0pt,
  boxsep=5pt,
  boxrule=0pt,
  left=0pt,right=0pt,top=0pt,bottom=0pt,
  arc=8pt,
  before skip=\topsep,
  after skip=\topsep
}

\tcolorboxenvironment{definition}{
  breakable,
  colback=black!10,
  colframe=white,
  width=\linewidth,
  enlarge left by=0pt,
  enlarge right by=0pt,
  boxsep=5pt,
  boxrule=0pt,
  left=0pt,right=0pt,top=0pt,bottom=0pt,
  arc=8pt,
  before skip=\topsep,
  after skip=\topsep
}

\tcolorboxenvironment{assumption}{
  breakable,
  colback=black!10,
  colframe=white,
  width=\linewidth,
  enlarge left by=0pt,
  enlarge right by=0pt,
  boxsep=5pt,
  boxrule=0pt,
  left=0pt,right=0pt,top=0pt,bottom=0pt,
  arc=8pt,
  before skip=\topsep,
  after skip=\topsep
}

\tcolorboxenvironment{claim}{
  breakable,
  colback=black!10,
  colframe=white,
  width=\linewidth,
  enlarge left by=0pt,
  enlarge right by=0pt,
  boxsep=5pt,
  boxrule=0pt,
  left=0pt,right=0pt,top=0pt,bottom=0pt,
  arc=8pt,
  before skip=\topsep,
  after skip=\topsep
}

\tcolorboxenvironment{problem}{
  breakable,
  colback=black!10,
  colframe=white,
  width=\linewidth,
  enlarge left by=0pt,
  enlarge right by=0pt,
  boxsep=5pt,
  boxrule=0pt,
  left=0pt,right=0pt,top=0pt,bottom=0pt,
  arc=8pt,
  before skip=\topsep,
  after skip=\topsep
}

\tcolorboxenvironment{question}{
  breakable,
  colback=black!10,
  colframe=white,
  width=\linewidth,
  enlarge left by=0pt,
  enlarge right by=0pt,
  boxsep=5pt,
  boxrule=0pt,
  left=0pt,right=0pt,top=0pt,bottom=0pt,
  arc=8pt,
  before skip=\topsep,
  after skip=\topsep
}

\newtcolorbox{titleblock}{
  enhanced,
  frame hidden,
  colback=CarolinaUltraLight,
  colframe=CarolinaUltraLight,
  boxrule=0pt,
  arc=10pt,
  left=14pt,
  right=14pt,
  top=14pt,
  bottom=14pt,
  width=\linewidth,
  before skip=12pt plus 4pt,
  after skip=12pt plus 4pt,
  grow to left by=1.5pt,
  grow to right by=1.5pt,
  before upper={
    \setlength{\parindent}{0cm}
    \setlength{\parskip}{0.5cm}
  }
}

\crefname{theorem}{Theorem}{Theorems}
\crefname{proposition}{Proposition}{Propositions}
\crefname{lemma}{Lemma}{Lemmas}
\crefname{corollary}{Corollary}{Corollaries}
\crefname{definition}{Definition}{Definitions}
\crefname{assumption}{Assumption}{Assumptions}
\crefname{remark}{Remark}{Remarks}
\crefname{problem}{Problem}{Problems}
\crefname{property}{Property}{property}
\crefname{question}{Question}{Questions}

\numberwithin{equation}{section}
\numberwithin{theorem}{section}
\numberwithin{proposition}{section}
\numberwithin{definition}{section}
\numberwithin{lemma}{section}
\numberwithin{assumption}{section}
\numberwithin{remark}{section}

\def\1{\bm{1}}

\makeatletter
\let\save@mathaccent\mathaccent
\newcommand*\if@single[3]{%
    \setbox0\hbox{${\mathaccent"0362{#1}}^H$}%
    \setbox2\hbox{${\mathaccent"0362{\kern0pt#1}}^H$}%
    \ifdim\ht0=\ht2 #3\else #2\fi
}
\newcommand*\rel@kern[1]{\kern#1\dimexpr\macc@kerna}
\newcommand*\widebar[1]{\@ifnextchar^{{\wide@bar{#1}{0}}}{\wide@bar{#1}{1}}}
\newcommand*\wide@bar[2]{\if@single{#1}{\wide@bar@{#1}{#2}{1}}{\wide@bar@{#1}{#2}{2}}}
\newcommand*\wide@bar@[3]{%
    \begingroup
    \def\mathaccent##1##2{%
        \let\mathaccent\save@mathaccent
        \if#32 \let\macc@nucleus\first@char \fi
        \setbox\z@\hbox{$\macc@style{\macc@nucleus}_{}$}%
        \setbox\tw@\hbox{$\macc@style{\macc@nucleus}{}_{}$}%
        \dimen@\wd\tw@
        \advance\dimen@-\wd\z@
        \divide\dimen@ 3
        \@tempdima\wd\tw@
        \advance\@tempdima-\scriptspace
        \divide\@tempdima 10
        \advance\dimen@-\@tempdima
        \ifdim\dimen@>\z@ \dimen@0pt\fi
        \rel@kern{0.6}\kern-\dimen@
        \if#31
        \overline{\rel@kern{-0.6}\kern\dimen@\macc@nucleus\rel@kern{0.4}\kern\dimen@}%
        \advance\dimen@0.4\dimexpr\macc@kerna
        \let\final@kern#2%
        \ifdim\dimen@<\z@ \let\final@kern1\fi
        \if\final@kern1 \kern-\dimen@\fi
        \else
        \overline{\rel@kern{-0.6}\kern\dimen@#1}%
        \fi
    }%
    \macc@depth\@ne
    \let\math@bgroup\@empty \let\math@egroup\macc@set@skewchar
    \mathsurround\z@ \frozen@everymath{\mathgroup\macc@group\relax}%
    \macc@set@skewchar\relax
    \let\mathaccentV\macc@nested@a
    \if#31
    \macc@nested@a\relax111{#1}%
    \else
    \def\gobble@till@marker##1\endmarker{}%
    \futurelet\first@char\gobble@till@marker#1\endmarker
    \ifcat\noexpand\first@char A\else
    \def\first@char{}%
    \fi
    \macc@nested@a\relax111{\first@char}%
    \fi
    \endgroup
    }
\makeatother

\DeclareMathAlphabet{\mathsfit}{\encodingdefault}{\sfdefault}{m}{sl}
\SetMathAlphabet{\mathsfit}{bold}{\encodingdefault}{\sfdefault}{bx}{n}

\renewcommand{\arraystretch}{1.15}
\newtheorem{ResearchQ}{RQ}
\tcolorboxenvironment{ResearchQ}{
  breakable,
  colback=black!10,
  colframe=white,%
  width=\dimexpr\linewidth\relax,%
  enlarge left by=0pt,%
  enlarge right by=0pt,%
  boxsep=5pt,%
  boxrule=0pt,
  left=0pt,right=0pt,top=0pt,bottom=0pt,
  arc=8pt
  before skip=\topsep,
  after skip=\topsep
}
\crefname{ResearchQ}{ResearchQ}{ResearchQs}

\usepackage{listings}
\lstdefinestyle{plainmarkdown}{
  language={},
  basicstyle=\fontsize{12}{14}\selectfont\ttfamily,
  breaklines=true,
  breakatwhitespace=true,
  showstringspaces=false,
  showspaces=false,
  showtabs=false,
  columns=fullflexible,
  keepspaces=false,
  backgroundcolor=\color{gray!8},
  frame=single,
  rulecolor=\color{gray!40},
  framerule=0.8pt,
  xleftmargin=3pt,
  xrightmargin=3pt,
  framextopmargin=3pt,
  framexbottommargin=3pt,
  breakindent=0pt,
  postbreak=\mbox{\textcolor{red}{$\hookrightarrow$}\space},
}

\lstdefinestyle{markdownhighlight}{
  basicstyle=\ttfamily\footnotesize,
  breaklines=true,
  breakatwhitespace=true,
  showstringspaces=false,
  showspaces=false,
  showtabs=false,
  columns=fullflexible,
  keepspaces=false,
  backgroundcolor=\color{gray!8},
  frame=single,
  rulecolor=\color{gray!40},
  framerule=0.8pt,
  xleftmargin=3pt,
  xrightmargin=3pt,
  framextopmargin=3pt,
  framexbottommargin=3pt,
  breakindent=0pt,
  postbreak=\mbox{\textcolor{red}{$\hookrightarrow$}\space},
  escapeinside={(*@}{@*)},
}

\renewcommand{\arraystretch}{1.15}
\definecolor{UNCblue!10}{rgb}{0.84,0.91,0.98}
\definecolor{RowAlt}{rgb}{0.98,0.98,0.99}

\begin{document}

\makeatletter
\def\blfootnote{\gdef\@thefnmark{}\@footnotetext}
\makeatother

\makeatletter
\pagestyle{fancy}
\fancyhf{}
\renewcommand{\headrulewidth}{1pt}
\chead{\small\bf TRUST: A Decentralized Framework for Auditing Large Language Model Reasoning

}
\cfoot{\thepage}
\thispagestyle{fancy}
\makeatother

\makeatletter
\def\icmldate#1{\gdef\@icmldate{#1}}
\icmldate{\today}
\makeatother

\makeatletter
\fancypagestyle{fancytitlepage}{
  \fancyhead{}
  \lhead{\includegraphics[height=0.8cm]{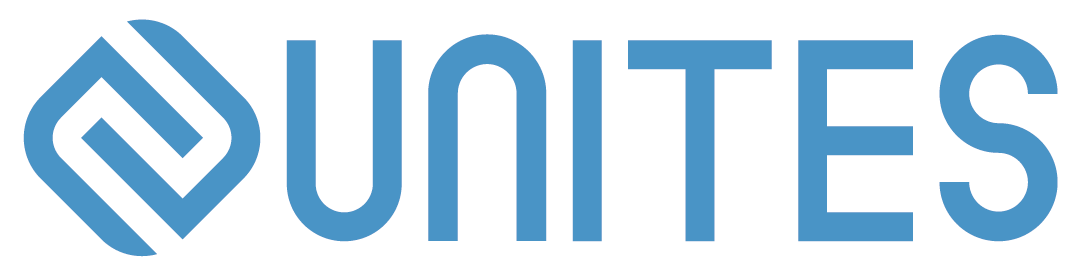}}
  \rhead{\it \@icmldate}
  \cfoot{}
}
\makeatother

\thispagestyle{fancytitlepage}

\vspace*{0.5em}

\noindent
\begin{titleblock}
    {\setlength{\parskip}{0cm}
     \raggedright
     {\setstretch{1.2}
      \LARGE\bfseries
      TRUST: A Decentralized Framework for Auditing Large Language Model Reasoning

      \par}
    }
    \vskip 0.2cm
    
    \begin{icmlauthorlist}
\mbox{Morris Yu-Chao Huang$^{\,1\,\textrm{\Letter}}$},
\mbox{Zhen Tan$^{\,2\,}$}, 
\mbox{Mohan Zhang$^{\,1\,}$}, 
\mbox{Pingzhi Li$^{\,1\,}$}, 
\mbox{Zhuo Zhang$^{\,3\,}$}, 
and \mbox{Tianlong Chen$^{\,1\,\textrm{\Letter}}$}
\end{icmlauthorlist}

$^{1\,}$UNITES Lab, University of North Carolina at Chapel Hill  
\quad 
$^{2\,}$Arizona State University
\quad 
$^{3\,}$Columbia University

$^{\textrm{\Letter}}$ Corresponding Author: \{morris, tianlong\}@cs.unc.edu

    \vskip 0.2cm

Large Language Models generate complex reasoning chains that reveal their decision-making, yet verifying the faithfulness and harmlessness of these intermediate steps remains a critical unsolved problem. Existing auditing methods are centralized, opaque, and hard to scale, creating significant risks for deploying proprietary models in high-stakes domains. We identify four core challenges: 
(1) \textit{Robustness}: Centralized auditors are single points of failure, prone to bias or attacks. 
(2) \textit{Scalability}: Reasoning traces are too long for manual verification.
(3) \textit{Opacity}: Closed auditing undermines public trust. 
(4) \textit{Privacy}: Exposing full reasoning risks model theft or distillation. 
We propose TRUST, a transparent, decentralized auditing framework that overcomes these limitations via: 
(1) A consensus mechanism among diverse auditors, guaranteeing correctness under up to $30\%$ malicious participants. 
(2) A hierarchical DAG decomposition of reasoning traces, enabling scalable, parallel auditing. 
(3) A blockchain ledger that records all verification decisions for public accountability.
(4) Privacy-preserving segmentation, sharing only partial reasoning steps to protect proprietary logic. We provide theoretical guarantees for the security and economic incentives of the TRUST framework. Experiments across multiple LLMs (GPT-OSS, DeepSeek-r1, Qwen) and reasoning tasks (math, medical, science, humanities) show TRUST effectively detects reasoning flaws and remains robust against adversarial auditors. Our work pioneers decentralized AI auditing, offering a practical path toward safe and trustworthy LLM deployment.

\end{titleblock}

\clearpage
\vspace{1em}
{\LARGE \bfseries Table of Contents}

{
\setlength{\parskip}{-0em}
\startcontents[sections]
\printcontents[sections]{}{1}{}
}

\setlength{\parskip}{.5em}
\clearpage

\section{Introduction}\label{sec:intro}
The capabilities of large language models (LLMs) have expanded from text generation to complex, multi-step reasoning, leading to the development of Large Reasoning Models (LRMs) that produce explicit reasoning traces~\citep{wei2022chain}. While offering a view of a model's logical flow, this explicit reasoning also exposes potential flaws, including logical errors, a lack of faithfulness to the model's true internal state~\citep{turpin2023language}, and safety vulnerabilities. The verification of these intermediate steps is a critical prerequisite for the safe and reliable deployment of LRMs in high-stakes domains such as medicine~\citep{singhal2023large}, law~\citep{chalkidis2021lexglue}, and finance~\citep{wang2023fingpt}. The urgency of this task is underscored by emerging regulatory frameworks like the EU AI Act~\citep{com2021laying} and the NIST AI RMF~\citep{ai2023artificial}, which mandate rigorous documentation and monitoring~\citep{oecd2019principles}. 
However, existing auditing methods are misaligned with this paradigm, as their centralized, opaque, and unscalable nature creates unacceptable risks. Specifically, they either rely on a single trusted entity, cannot process the volume and complexity of reasoning traces, or force a dangerous trade-off between public transparency and the protection of proprietary models.

The inadequacy of current auditing systems stems from four interconnected challenges. A primary issue is a lack of \textbf{Robustness} since systems relying on a single auditor, whether a human expert or another LLM, constitute a ``single point of failure'' and are vulnerable to targeted attacks like prompt injection~\citep{zou2023universal,perez2022ignore} and susceptible to systemic biases~\citep{bender2021dangers,liang2022holistic}. Compounding this issue is a severe \textbf{Scalability} bottleneck. The volume and combinatorial complexity of reasoning traces from modern LRMs, especially those employing branching search~\citep{lightman2023let,yao2023tree}, make comprehensive manual verification practically and economically infeasible, a fact evidenced by the massive human effort required for existing process supervision datasets~\citep{bai2022training,lightman2023let}. Furthermore, the \textbf{Opacity} of internal auditing processes at proprietary model providers erodes public trust and prevents independent verification of safety claims, conflicting with established principles of transparent reporting~\citep{bommasani2023foundation,mitchell2019model}. In parallel, addressing opacity creates a critical tension with \textbf{Privacy}, since exposing complete reasoning traces for public audit risks the theft of valuable intellectual property through model distillation~\citep{carlini2021extracting} and increases the surface area for extracting sensitive training data~\citep{nasr2023scalable}.

Addressing these simultaneous challenges of robustness, scalability, opacity, and privacy demands a new approach to the auditing paradigm. Our work is guided by the following research questions:
\begin{ResearchQ}
\textit{How can we design an auditing system that is robust to malicious participants and systemic bias without relying on a central trusted authority?}
\end{ResearchQ}
\begin{ResearchQ}
\textit{How can this system scale to audit complex reasoning traces while preserving the intellectual property of the model provider and ensuring public transparency?} 
\end{ResearchQ}
Answering these questions naturally leads to a framework that integrates decentralized consensus, privacy-preserving protocols, and a novel representation for reasoning itself.

\begin{figure}[htp!]
  \centering
  \includegraphics[width=0.6\textwidth]{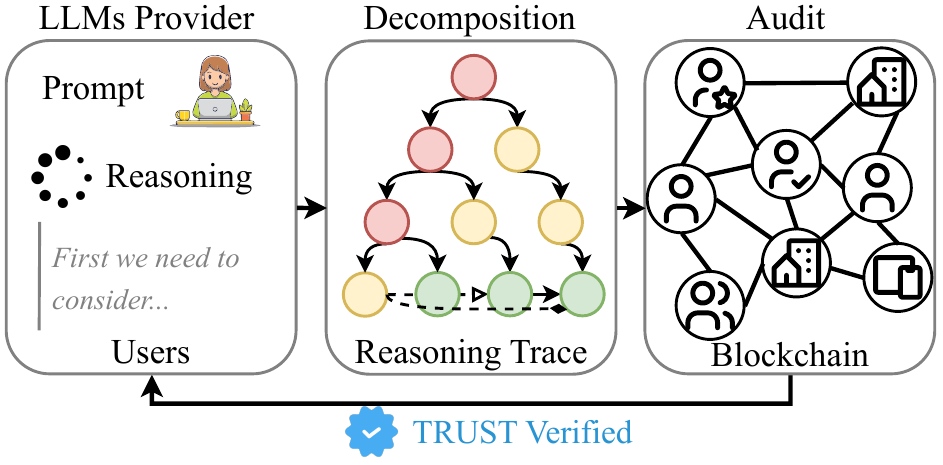}
  \caption{ 
  Reasoning traces (left) are decomposed into hierarchical segments (middle) and verified by a heterogeneous, multi-tier auditor network (right). 
  Audit outcomes are recorded on-chain, while raw trace is stored off-chain for privacy.}
  \label{fig:main_illustration}
\end{figure}
We introduce TRUST, a decentralized framework for auditing LLM reasoning (see \cref{fig:main_illustration}). 
To achieve \textbf{robustness}, TRUST establishes a consensus mechanism among a diverse, multi-tier set of auditors, drawing on principles from Byzantine Fault Tolerant systems~\citep{castro1999practical,lamport2019byzantine} to provably guarantee audit correctness even with a significant fraction of malicious participants. For \textbf{scalability}, the framework introduces a novel decomposition method that transforms reasoning traces into \textit{Hierarchical Directed Acyclic Graphs (HDAGs)}, a structured representation that permits parallel verification of atomic reasoning steps by a distributed network. To jointly address \textbf{opacity} and \textbf{privacy}, all verification decisions are recorded on a transparent blockchain ledger for public auditability, while the protocol preserves confidentiality by distributing only partial, disconnected trace segments to individual auditors, protecting proprietary logic from reconstruction.

The design of TRUST is supported by rigorous theoretical guarantees for security and economic viability. We prove a \textit{Safety-Profitability Guarantee} theorem, which formally ensures that under configurable statistical and economic conditions, the system achieves a target audit safety level while making honest participation profitable and malicious behavior result in a net loss. 
Our empirical validation spans multiple state-of-the-art LLMs (e.g., GPT-OSS~\citep{openai2025gptoss}, DeepSeek-r1~\citep{deepseekr1}, Qwen~\citep{yang2025qwen3}) and diverse reasoning tasks, and incorporates human-in-the-loop experiments with expert auditors to validate the multi-tier design. 
The results demonstrate that TRUST is highly effective at identifying reasoning flaws and is significantly more resilient to coordinated attacks than centralized baselines. 

In summary, our main contributions are:
\begin{itemize}[leftmargin=*]
    \item We introduce TRUST, the first decentralized auditing system for reasoning traces that achieve privacy-preserving verification without exposing the proprietary model.

    \item We develop a systematic approach to decompose Chain-of-Thought reasoning into Hierarchy Directed Acyclic Graphs (HDAGs) that enable modular verification coupled with a multi-tier verification, routing simple problems to automated validators and complex problems to human experts.

    \item We develop theoretically grounded incentive mechanisms which ensure that honest auditors profit while malicious actor incur losses, providing the foundation necessary for sustainable real-world deployment at scale.
    
    \item We conduct comprehensive experiments on diverse datasets (e.g., MMLU-Pro, GSM8K) and models (e.g., GPT-OSS, DeepSeek-r1), including human-in-the-loop studies, to demonstrate the effectiveness and robustness of TRUST against centralized baselines.
\end{itemize}

\section{Related Works}\label{sec:related}

\textbf{Reasoning Model Verification.} 
Chain-of-Thought (CoT) prompting has revolutionized LLM reasoning by exposing intermediate steps~\citep{wei2022chain}, evolving into sophisticated tree-based search methods~\citep{yao2023tree} and Large 
Reasoning Models that treat reasoning as a primary objective~\citep{jaech2024openai,guo2025deepseek}. However, these advances lack systematic verification mechanisms for generated reasoning traces, particularly for privacy-preserving and decentralized auditing~\citep{peng2025survey}.

\textbf{Auditing and Evaluation.} 
Current auditing approaches range from centralized ``LLM-as-a-judge'' methods~\citep{zheng2023judging} to Process Reward Models that provide step-by-step supervision~\citep{lightman2023let}. 
Recent work addresses service-level integrity through cryptographic verification~\citep{sun2025coin} and detection of model substitution~\citep{cai2025you}. While addressing inference integrity, these approaches lack unified frameworks for scalable semantic auditing with decentralized consensus.

\textbf{Decentralized Verification.} 
Foundational work in Byzantine Fault Tolerant consensus~\citep{castro1999practical} and Zero-Knowledge Proofs for ML~\citep{chen2024zkml,sun2024zkllm} provides primitives for verifiable computation. 
However, existing approaches focus on computational correctness rather than semantic quality verification through human-in-the-loop consensus processes.

Our work synthesizes these directions by introducing the first framework for decentralized, privacy-preserving semantic auditing of reasoning traces at scale.
Due to space limitation, we provide more comprehensive related works in \cref{sec:more_related}.

\section{Decentralized Auditing for Large Reasoning Models}\label{sec:method}

\begin{figure}[t!]
    \centering
    \includegraphics[width=\textwidth]{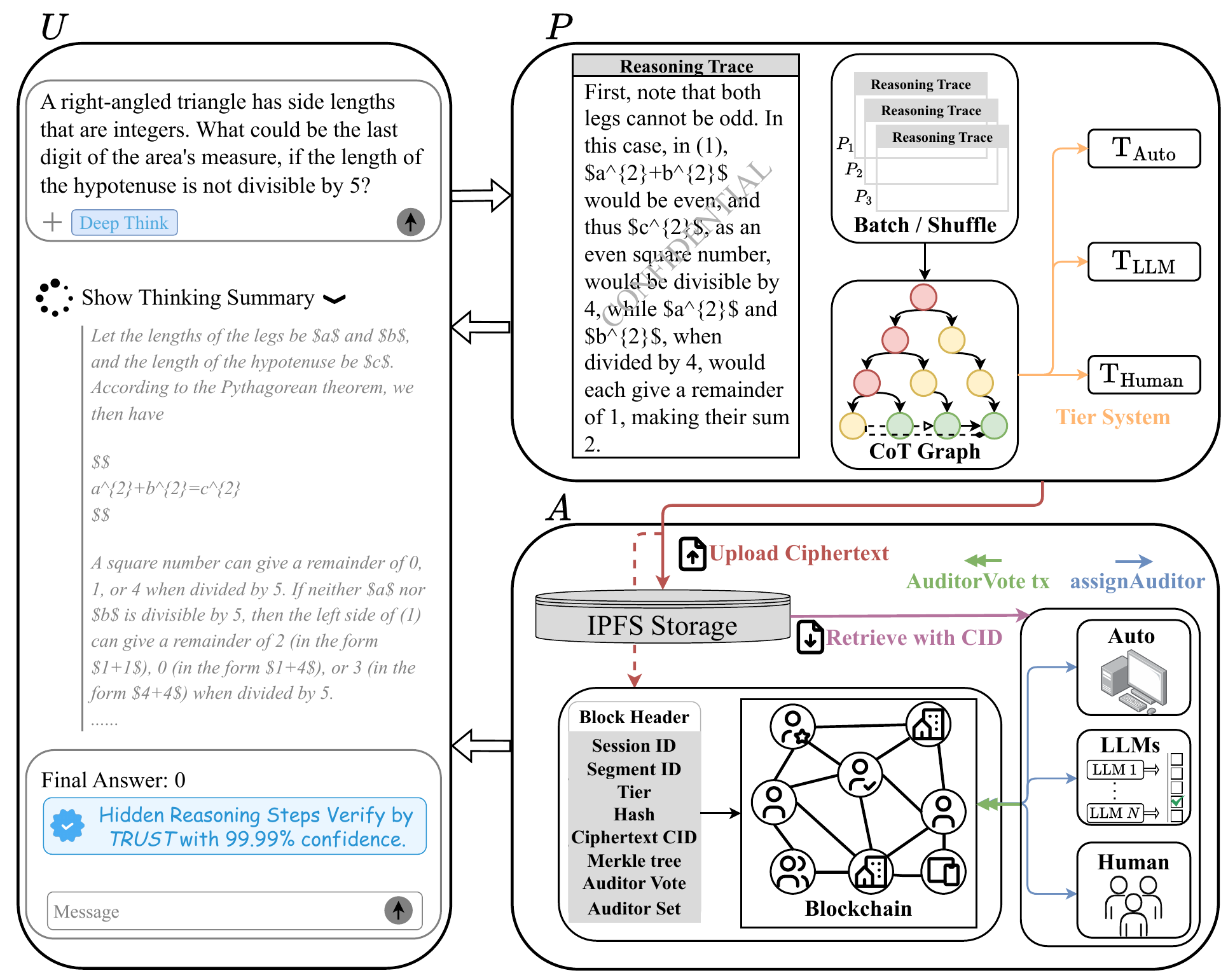}
    \caption{
    {Overview of TRUST.} 
    The TRUST framework decomposes Chain-of-Thought reasoning traces into hierarchical directed acyclic graphs (DAGs) for modular verification across a three-tier auditor system. 
    The process begins with a reasoning query (left panel) that generates intermediate reasoning steps, which are then decomposed into graph components and distributed across automated computers, LLM-based, and human auditors. 
    TRUST utilizes IPFS for decentralized storage of reasoning traces and blockchain technology for immutable audit records, vote aggregation, and consensus mechanisms. 
    Auditors verify reasoning segments independently, while cryptographic protocols ensure the privacy preservation of proprietary model internals. 
    The final verification result provides confidence guarantees for reasoning trace faithfulness and correctness for the end user.}
    \label{fig:main}
\end{figure}
As illustrated in \cref{fig:main_illustration}, \textsc{TRUST} can integrate either human or LLMs to audit \emph{faithfulness}, \emph{harmlessness}, and \emph{logical consistency} of chain-of-thought (CoT) reasoning. By operating on intermediate traces rather than final outputs alone, \textsc{TRUST} enables earlier and more comprehensive detection of reasoning flaws. \textsc{TRUST} features the following key innovations:

\begin{itemize}[leftmargin=*]
    \item \textbf{Batch \& Segmentation.} 
    Reasoning traces from multiple providers are \emph{batched} to anonymize source identity and mitigate provider-specific bias. 
    Traces are then \emph{segmented} into minimal, auditable units and stored as content-addressed objects in decentralized storage. 
    Segmentation protects proprietary logic: each auditor only sees the segment(s) they are assigned, preventing full-trace reconstruction.
    
    \item \textbf{Auditing \& Consensus.} 
      Heterogeneous auditors (computational checkers, LLMs, and humans) independently evaluate assigned segments. 
      Votes are submitted via a cryptographic \emph{commit--reveal} protocol: in the commit phase, auditors submit hashed votes; in the reveal phase, they disclose votes for verification against commitments. 
      Segment-level quorums validate local steps; a trace-level aggregator combines weighted segment outcomes to reach the final decision.

    \item \textbf{Blockchain \& Decentralized Storage.}
    A blockchain layer provides immutable audit trails and trustless consensus using a Proof-of-Stake (PoS)-style mechanism adapted for AI auditing. 
    Smart contracts orchestrate session lifecycle, auditor assignment (by stake and expertise), commit--reveal voting, and performance-based rewards/slashing. 
    Reasoning content is stored off-chain on IPFS; the blockchain records metadata, vote commitments, and final outcomes.
\end{itemize}

We formalize the three key participant parties in the TRUST ecosystem—\textbf{Provider}, \textbf{Auditor}, and \textbf{User}. 
TRUST supports both Business-to-Business (B2B) and Business-to-Consumer (B2C) settings, where the reasoning trace provider may be either a proprietary LLM vendor or an individual customer.
We provide the illustration of the parties involved in TRUST in \cref{fig:main}.
\begin{itemize}[leftmargin=*]
\item \textbf{Provider ($P$)}: A proprietary LRM vendor or customer submitting reasoning traces for audit. 
\item \textbf{Auditor ($A$)}: An anonymous seat (computer, LLM, or human) that verifies assigned segments.
\item \textbf{User ($U$)}: An end-user who consumes audited outputs and provenance via APIs or dashboards.
\end{itemize}

Given a reasoning trace with $S$ segments (including CoT and tool calls), \textsc{TRUST} maps the trace to a \emph{Hierarchical Directed Acyclic Graph (HDAG)} with five abstraction levels: \emph{Goal}, \emph{Strategy}, \emph{Tactic}, \emph{Step}, and \emph{Operation}.
This representation is \emph{problem-agnostic} (math, science, programming, general reasoning, etc.) and enables scalable, parallel verification because most nodes are independently auditable. 
Each node carries metadata (ID, summary, complexity, auditor type, and dependencies), and edges encode relationships (\emph{decomposes\_to}, \emph{depends\_on}, \emph{enables}, \emph{validates}, \emph{contradicts}, etc.). 
Formally, each segment $s\in\{1,\dots,S\}$ is assigned to a primary auditor type $\in\{\mathrm{\textbf{C}omputer},\mathrm{\textbf{L}LM},\mathrm{\textbf{H}uman}\}$:
\begin{align*}
  \underbrace{\text{Segment }1}_{\mathrm{\textbf{C}omputer}},\quad
  \underbrace{\text{Segment }2}_{\mathrm{\textbf{L}LM}},\quad
  \underbrace{\text{Segment }3}_{\mathrm{\textbf{H}uman}},\;\dots,\;
  \underbrace{\text{Segment }S}_{\mathrm{type}\in \{\mathrm{\textbf{C}, \textbf{L}, \textbf{H}}\}}.
\end{align*}

\subsection{Hierarchical Directed Acyclic Graphs (HDAGs)}
Prior work on CoT decomposition, such as DLCoT, introduced automatic frameworks for breaking down long reasoning traces into structured segments, primarily to generate high-quality data for model distillation~\citep{luo2025decon}. 
These works observe that CoTs can follow linear, tree, or more general network structures. 
DLCoT, for instance, applies macro-structure parsing to divide CoTs into four parts—\emph{Problem Restatement}, \emph{Approach Exploration}, \emph{Verification}, and \emph{Summary}—before further segmenting the approach and verification stages into stepwise units. 
Other lines of research~\citep{kothapalli2025cot} focus on extracting causal structures from token-level processing functions.

In contrast, we propose a general, problem-agnostic approach: decomposing CoTs into \emph{Hierarchical Directed Acyclic Graphs} (HDAGs). 
Our hierarchy consists of five abstraction levels: \emph{Goal}, \emph{Strategy}, \emph{Tactic}, \emph{Step}, and \emph{Operation}. 
This abstraction provides two key advantages. 
First, it is broadly applicable across domains—mathematics, science, engineering, and open-domain reasoning. 
Second, it enables scalable verification, since most nodes are independently auditable and can be naturally mapped to different auditor types (e.g., computer programs, LLMs, or human experts). 
An illustration is provided in \cref{fig:hdag}, where each node is annotated by difficulty and type (basic reasoning step, tool usage, or fact/premise). 
This hierarchical decomposition mirrors neural circuits in the frontal cortex, which process reasoning through multi-level evidence integration~\citep{morteza2019hierarchical}. 
Just as the brain organizes reasoning hierarchically rather than linearly, our HDAG design enables different reasoning components to be audited at the appropriate granularity. 
Edges capture logical relationships between nodes, including \emph{dependencies} (depends on, enables), \emph{structural links} (decomposes, refines), \emph{validation} (validates, exemplifies), and \emph{conflicts} (contradicts).

\begin{figure}[htp]
    \centering
    \includegraphics[width=0.35\textwidth]{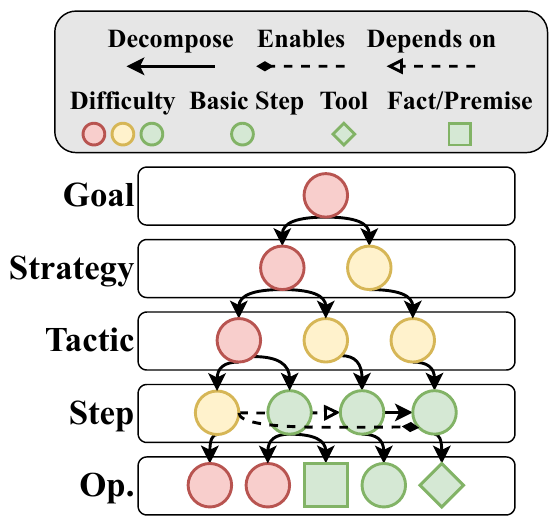}
    \caption{
    {Example of an HDAG.}
    Node color indicates the level of difficulty assigned to different types of auditors. 
    Node shape denotes the type of step and
    edges represent relationships between nodes.
    }
    \label{fig:hdag}
\end{figure}
Concretely, TRUST constructs HDAGs in five steps:
\begin{enumerate}[leftmargin=*, label=\textbf{Step \arabic*:}]
    \item \textbf{Identify Abstraction Levels.}  
    The raw problem statement, reasoning trace (with tool usage), and final output are parsed into semantic hierarchy levels.

    \item \textbf{Segment Within Each Level.}  
    Each level is further divided into granular units with associated metadata (IDs, complexity, summaries). 
    Difficulty annotations guide later auditor assignment.

    \item \textbf{Extract Relationships.}  
    Logical dependencies between segments are mapped into relations (\emph{decomposes\_to}, \emph{depends\_on}, \emph{enables}, \emph{validates}, etc.).

    \item \textbf{Assign Auditor Types.}  
    Segments are routed to auditor types from $\{\mathrm{Human}, \mathrm{Computer}, \mathrm{LLM}\}$, based on complexity and modality.

    \item \textbf{Refine and Construct HDAG.}  
    Segments and relationships are synthesized into a final auditable HDAG, with quality assurance checks.
\end{enumerate}

Further implementation details and examples are provided in \cref{app:implementation,sec:hdag_ex1,sec:hdag_ex2}.

\subsection{Auditing \& Consensus}

Reasoning traces vary in both complexity and type of reasoning step. 
In TRUST, each segment is routed to the most appropriate auditor type—\emph{Human}, \emph{LLM}, or \emph{Computer}—to ensure accurate and efficient verification. 
Decentralized participants either contribute computational resources for deterministic checks (e.g., arithmetic, tool calls) or provide expertise in evaluating semantic coherence and faithfulness. 
This heterogeneous auditor pool improves both accuracy (by matching segments to suitable verifiers) and robustness (by reducing vulnerability to malicious or biased auditors). 
We analyze the consensus process at three levels: \emph{seat}, \emph{segment}, and \emph{trace}.

\begin{enumerate}[leftmargin=*,itemsep=0.3em]

\item \textbf{Seat layer.}  
Within a segment $s$, each of the $k_{t(s)}$ auditor seats votes independently. 
Computer seats are assumed noiseless, while LLM and human seats have nonzero error rates $\epsilon_t$. 
Human seats may additionally be adversarial with probability $\rho_{\mathrm H}$.

\item \textbf{Segment layer.}  
For the segment $s$, define the segment pass indicator
$ B_s=\mathbf{1}\!\left[\#\{\text{correct votes}\}\ge q_{t(s)}\right], $
where $q_{t}=\lceil\tau\,k_t\rceil$ is the quorum threshold for type $t$. 
The exact pass probability for a segment of type $t$ with parameters $(k_t,\epsilon_t,\rho_t)$ (with $\rho_{\mathrm C}=\rho_{\mathrm L}=0$):
\begin{align}
  p_t
  = \Pr[B_s=1]
  = \sum_{m=0}^{k_t}\binom{k_t}{m}\rho_t^m(1-\rho_t)^{k_t-m}
    \sum_{c=q_t}^{k_t-m}
      \binom{k_t-m}{c}
      (1-\epsilon_t)^c\,\epsilon_t^{\,k_t-m-c},
\end{align}
where $m$ malicious seats vote incorrectly, among the $k_t-m$ honest seats, $c$ cast correct votes, $\epsilon_t$ is the error rate for type $t$, and $\rho_{\mathrm{H}}$ is the human adversarial probability

\item \textbf{Trace layer.}  
To aggregate across all $S$ segments, we assign weights $w_{t(s)}$ and define
$
W=\sum_{s=1}^{S} w_{t(s)}\,B_s, 
\quad W_\beta=\beta\sum_{s}w_{t(s)}
$
, where $w_t$ is the weight for segment of type $t$, and $W_\beta$ is the trace-level quorum threhold ($\beta \in (0, 1)$).
We then bound the failure probability $\Pr[W<W_\beta]$ using Hoeffding and Chernoff inequalities:
\begin{align}\label{eq:hoeffding-chernoff}
\Pr\!\bigl[ W < W_\beta \bigr]
\le
\underbrace{%
  \exp\!\Bigl[
    -2\,(\mu_{\text{vote}} - W_\beta)^{2}\bigl/\sigma_{\max}^{2}
  \Bigr]
}_{\text{Hoeffding}}
\land
\underbrace{%
  \min_{\lambda>0}
    \exp\!\Bigl(
      \lambda W_\beta
      +\sum_{s=1}^{S}
         \ln\!\bigl(
           p_s\,e^{-\lambda w_s} + (1-p_s)
         \bigr)
    \Bigr)
}_{\text{Chernoff}} .
\end{align}
Figure~\ref{fig:fail_bound_theory_exact} compares these bounds with the exact solution under representative parameters ($\epsilon_{\mathrm{C}}=0$, $\epsilon_{\mathrm{L}}=0.05$, $\epsilon_{\mathrm{H}}=0.30$, $\rho_{\mathrm{H}}=0.1$).
The full derivation on seat, segment, and trace levels results are provided in \cref{sec:appex_theory} and notation summarize in \cref{tab:notation}.

\begin{figure}[htp!]
    \centering
    \begin{subfigure}[b]{0.4\textwidth}
        \includegraphics[width=\textwidth]{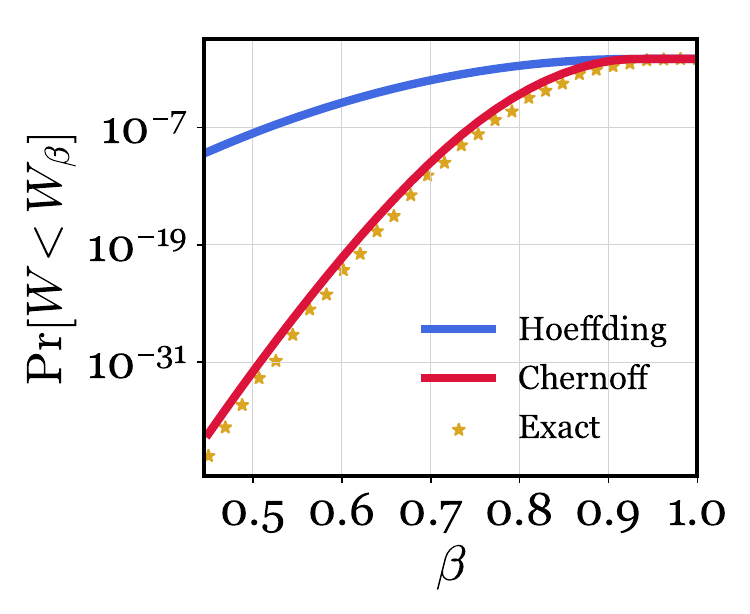}
        \label{fig:failure_vs_beta}
    \end{subfigure}
    \begin{subfigure}[b]{0.4\textwidth}
        \includegraphics[width=\textwidth]{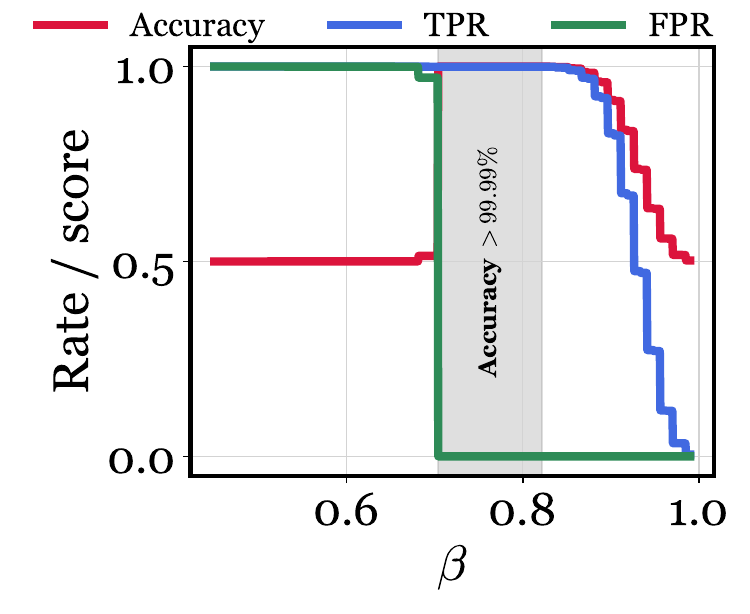}
        \label{fig:beta_vs_metrics}
    \end{subfigure}
    \caption{The parameters are $\epsilon_{\mathrm{C}}=0$, $\epsilon_{\mathrm{L}}=0.05$, $\epsilon_{\mathrm{H}}=0.30$, and $\rho_{\mathrm{H}}=0.1$. 
    (Left) Comparison of probability of failure of Hoeffding and Chernoff bounds and exact solution in \eqref{eq:hoeffding-chernoff} as a function of trace-level quorum threshold $\beta$.
    (Right) The true positive rate (TPR), false positive rate (FPR), and accuracy with different values of trace-level quorum threshold.
    The grey shaded area indicates the width of the trace-level quorum that achieves greater accuracy than $99.99\%$.
    }
    \label{fig:fail_bound_theory_exact}
\end{figure}
    
\end{enumerate}

\subsection{Economics Analysis}\label{sec:economics}

In this section, we provide the economic analysis of the TRUST framework on reputation, slashing, reward, statistical, and economic guarantees.
\paragraph{Reputation-Weighted Slashing and Rewards.}
Each human auditor seat~$i$ maintains a reputation score $r_i(t)\in[0,1]$, updated after every segment as
\(
r_i(t+1) = (1-\gamma)\,r_i(t) + \gamma\,\mathbf{1}[\text{vote correct}],
\)
where $\gamma\in(0,1]$ controls adaptation speed.  
Incorrect votes trigger a slashing probability
\(
p_{\mathrm{slash}}(r) = p_{\min}+(p_{\max}-p_{\min})(1-r),
\)
with $0<p_{\min}<p_{\max}\le1$, penalizing low-reputation seats more heavily.  
The per-segment payoff $X_i\in\{-P,\,0,\,R\}$ is defined as:  
$R$ for a correct vote, $0$ for an incorrect vote without slashing, and $-P$ for a slashed incorrect vote.  
For an honest seat with error rate $\epsilon_{\mathrm H}$, the expected payoff is
\(
\mu_{\mathrm H}(r) := \mathbb E[X_i] = (1-\epsilon_{\mathrm H})R - \epsilon_{\mathrm H}\,P\,p_{\mathrm{slash}}(r).
\)

Take parameters $R=6$, $P=8$, $p_{\min}=0.2$, $p_{\max}=0.5$, $\delta=0.2$, $\lambda=60$, and $\epsilon_{\mathrm{H}}=0.30$ for example, an honest seat achieves an expected per-segment payoff of $\mu_{\min}=0.7\times 6-0.3\times 8\times 0.5=3.0$, with variance $\sigma_{\mathrm{H}}^2=25.8$ and worst-case increment $b=6$. 
A malicious seat, by contrast, suffers an expected loss of $\mathbb{E}[X_{\mathrm{mal}}]=-0.5\times 8=-4.0$, with variance $16$ and worst-case increment $b=8$. 
Over a 24-hour window ($T=24$) with $1440$ segments, tail bounds from \cref{thm:safty_profit} show that the probability of an honest auditor ending with nonpositive payoff is at most $\exp\!\left(-\big(60\times 24\times 3^{2}\big)\,/\,\big(2\times 25.8+(2/3)\times 6\times 3\big)\right)\approx e^{-204}<10^{-80}$, while the probability of a malicious auditor breaking even or better is at most $\exp\!\left(-\big(60\times 24\times (0.2\cdot 8)^2\big)\,/\,\big(2\times 25.8+(2/3)\times 6\times 1.6\big)\right)\approx e^{-63.6}<10^{-27}$.

\section{Experimental Verification}\label{sec:exp}
In this section, we provide verification on TRUST on the annotated CoT dataset, open-source model-generated CoTs for de-bias, and safety and privacy results.

\subsection{Correctness and Faithfulness} 
We use 200 samples from the MMLU-Pro-CoT-Train dataset~\citep{mmlu_pro_cot_train}, which provides ground truth annotations for individual reasoning steps and final answers. 
This allows us to systematically evaluate the correctness and faithfulness of audits at both the step and trace levels. 
\begin{figure}[htpb!]
    \centering
    \includegraphics[width=0.40\textwidth]{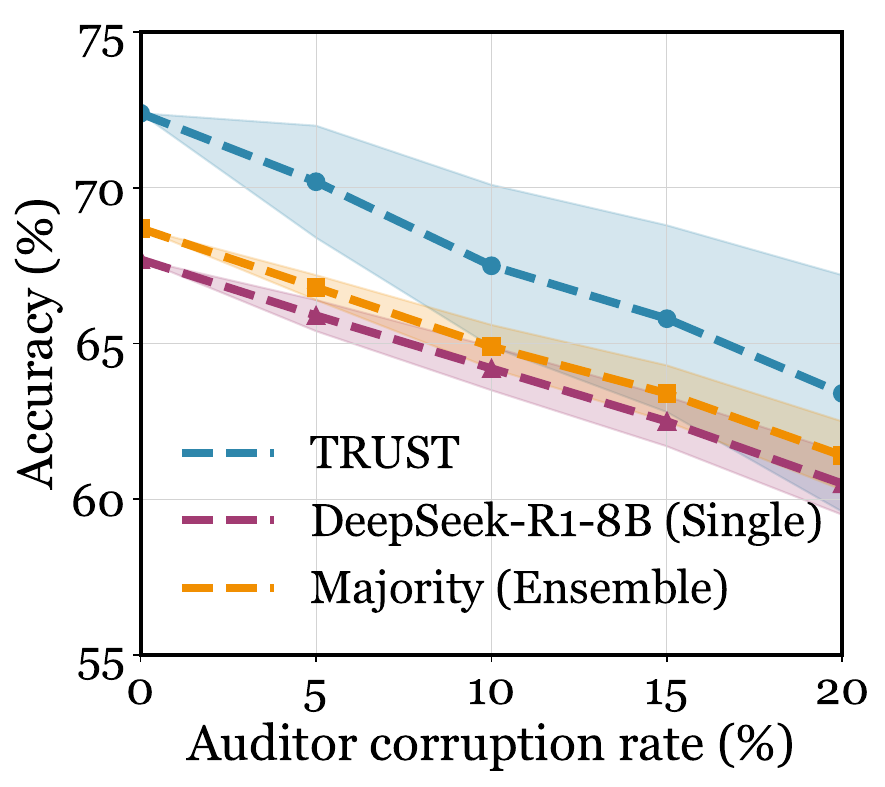}
    \caption{Correctness of Single, Ensemble (Centralized) with decentralized TRUST framework.}
    \label{fig:exp1}
\end{figure}
We compare TRUST against centralized approaches, including (i) single-LLM auditors (DeepSeek-R1-8B, Qwen2.5-7B, Mistral-7B, GPT-OSS-20B, LLaMA-3B) and (ii) ensemble-based voting schemes (majority, supermajority, weighted, unanimous). 
To stress-test robustness, we simulate auditor corruption by systematically flipping a proportion of segment-level votes, with corruption rates ranging from $5\%$ to $20\%$. 

Figure~\ref{fig:exp1} and Table~\ref{tab:trust_performance} summarize the results. 
At baseline (no corruption), TRUST achieves the highest accuracy ($72.4\%$), outperforming both single auditors (e.g., DeepSeek-R1-8B at $67.7\%$) and ensemble methods (e.g., majority voting at $68.7\%$). 
As corruption increases, all methods degrade, but TRUST degrades more gracefully: accuracy remains above $63\%$ even at $20\%$ corruption, while centralized ensembles drop below $61\%$ and single auditors fall closer to $60\%$. 
The performance gap widens with higher corruption rates, highlighting TRUST’s resilience to adversarial or biased auditors. 
\begin{table*}[htp!]
\centering
\caption{Performance comparison of TRUST (decentralized) vs.\ centralized approaches across corruption rates. Best is \textbf{bold}, second-best is \underline{underlined}.}
\label{tab:trust_performance}
\begingroup
\small
\setlength{\tabcolsep}{2pt}
\renewcommand{\arraystretch}{1.2}

\begin{tabularx}{\linewidth}{@{} l X c c c c c @{}}
\toprule
\textbf{Category} & \textbf{Method} & \textbf{Baseline} & \textbf{5\% Corr.} & \textbf{10\% Corr.} & \textbf{15\% Corr.} & \textbf{20\% Corr.} \\
\midrule
\rowcolor{UNCblue!10}\textbf{Decentralized} & TRUST
  & $\mathbf{72.4_{\pm 0.0}}$
  & $\mathbf{70.2_{\pm 1.8}}$
  & $\mathbf{67.5_{\pm 2.6}}$
  & $\mathbf{65.8_{\pm 3.0}}$
  & $\mathbf{63.4_{\pm 3.8}}$ \\
\midrule
\multirow[c]{12}{*}{\textbf{Centralized}} 
  & \multicolumn{6}{c}{\centering\textit{Ensemble Models}} \\
\cmidrule(lr){2-7}
& Majority Voting
  & $\underline{68.7}_{\pm 0.0}$
  & $\underline{66.8}_{\pm 0.4}$
  & $64.9_{\pm 0.7}$
  & $63.4_{\pm 0.9}$
  & $61.4_{\pm 1.1}$ \\
& Supermajority
  & $\underline{68.7}_{\pm 0.0}$
  & $66.8_{\pm 0.5}$
  & $\underline{65.0}_{\pm 0.7}$
  & $\underline{63.2}_{\pm 0.9}$
  & $\underline{61.2}_{\pm 0.9}$ \\
& Weighted Voting
  & $68.1_{\pm 0.0}$ & $66.4_{\pm 0.6}$ & $64.5_{\pm 0.7}$ & $62.7_{\pm 1.1}$ & $60.9_{\pm 0.9}$ \\
& Unanimous
  & $45.6_{\pm 0.0}$ & $46.1_{\pm 0.6}$ & $46.5_{\pm 0.9}$ & $46.8_{\pm 1.1}$ & $47.4_{\pm 1.0}$ \\
\cmidrule(lr){2-7}
& \multicolumn{6}{c}{\centering\textit{Single Models}} \\
\cmidrule(lr){2-7}
& DeepSeek-R1-8B
  & $67.7_{\pm 0.0}$ & $65.9_{\pm 0.5}$ & $64.2_{\pm 0.7}$ & $62.5_{\pm 0.8}$ & $60.5_{\pm 1.0}$ \\
& Qwen2.5-7B
  & $67.4_{\pm 0.0}$ & $65.7_{\pm 0.6}$ & $64.1_{\pm 0.7}$ & $62.1_{\pm 0.9}$ & $60.5_{\pm 1.0}$ \\
& Mistral-7B
  & $66.8_{\pm 0.0}$ & $65.2_{\pm 0.6}$ & $63.6_{\pm 0.8}$ & $61.8_{\pm 1.1}$ & $60.1_{\pm 1.1}$ \\
& DeepSeek-R1-1.5B
  & $64.1_{\pm 0.0}$ & $62.9_{\pm 0.5}$ & $61.2_{\pm 0.8}$ & $59.7_{\pm 1.1}$ & $58.5_{\pm 0.8}$ \\
& GPT-OSS-20B
  & $63.8_{\pm 0.0}$ & $62.5_{\pm 0.6}$ & $60.9_{\pm 0.7}$ & $59.7_{\pm 1.1}$ & $58.4_{\pm 1.0}$ \\
& LLaMA-3B
  & $52.1_{\pm 0.0}$ & $51.9_{\pm 0.6}$ & $51.7_{\pm 0.7}$ & $51.4_{\pm 1.0}$ & $51.3_{\pm 1.1}$ \\
\bottomrule
\end{tabularx}
\endgroup
\end{table*}

\subsection{Safety and Profitability} 
A central design goal of TRUST is to guarantee both \emph{statistical safety}—ensuring that the probability of a failed audit remains vanishingly small—and \emph{economic sustainability}—ensuring that honest auditors are consistently rewarded while malicious ones suffer provable losses. 

\cref{fig:reputation_profit} illustrates these dynamics empirically. 
On the left of Figure~\ref{fig:reputation_profit}, reputation scores naturally separate: honest auditors are reinforced with high reputation, while malicious and random guessers quickly lose credibility. 
On the right, profit trajectories diverge: honest participants earn steadily increasing rewards, while guessers and malicious seats accumulate losses due to repeated slashing. 
These empirical trends are formally supported by the \emph{Safety–Profitability Guarantee} (Theorem~\ref{thm:safty_profit_main}), which proves that, under appropriate statistical and economic parameters, honest auditors almost surely remain profitable while malicious participants incur provable long-term losses. 
The detailed derivation of these guarantees is provided in \cref{sec:appex_theory}.

\begin{figure}[htp!]
    \centering
    \begin{subfigure}[b]{0.4\textwidth}
        \includegraphics[width=\textwidth]{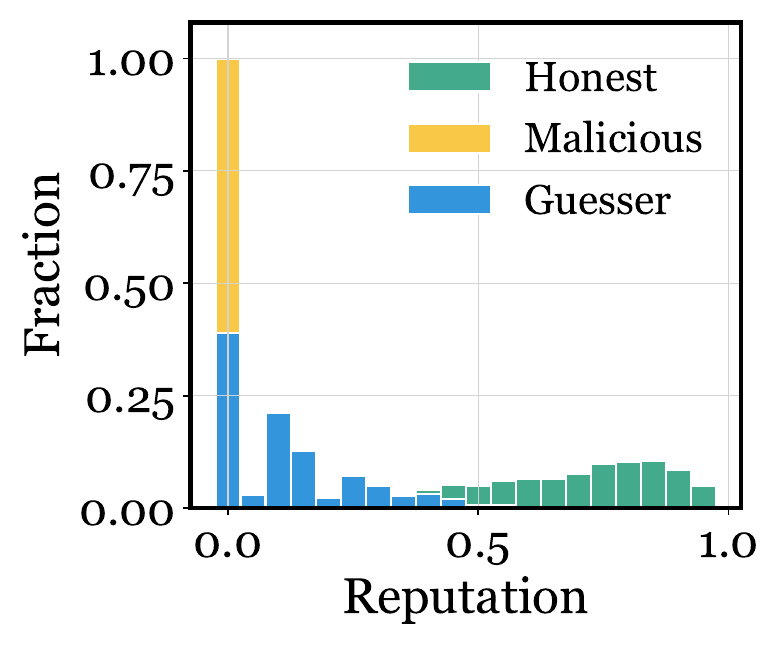}
    \end{subfigure}
    \begin{subfigure}[b]{0.4\textwidth}
        \includegraphics[width=\textwidth]{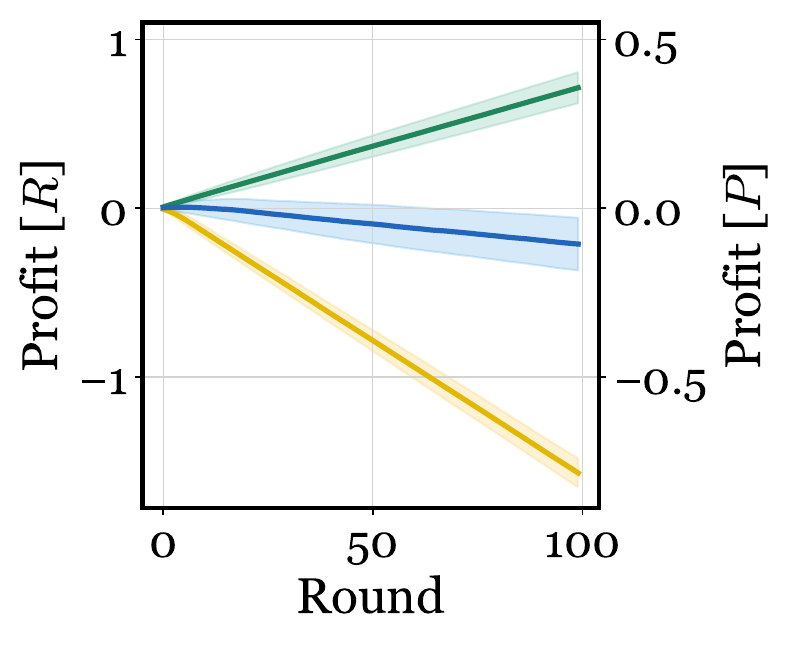}
    \end{subfigure}
    \caption{The parameters $\epsilon_{\mathrm{C}}=0$, $\epsilon_{\mathrm{L}}=0.05$, $\epsilon_{\mathrm{H}}=0.30$, and $\rho_{\mathrm{H}}=0.1$. 
    (Left) Repupation scores.
    (Right) Profit curves.
    }
    \label{fig:reputation_profit}
\end{figure}

\begin{theorem}[Safety–Profitability Guarantee]
\label{thm:safty_profit_main}
Fix a horizon \(T>0\), a target trace-failure probability
$\epsilon_{\mathrm{target}}\in(0,1)$ and a design constant
$\delta\in(0,1)$.
We have the following two dials to control the safety-profitability.

\begin{itemize}[leftmargin=*, itemsep=0.25em]
\item \textbf{Statistical dial.}  
Let \((k_t,q_t,w_t,\beta)\) be the vote parameters.
Write \(\mu_{\text{vote}}:=\mathbb E[W]\) and
\(\sigma_{\text{vote}}^{2}:=\sup\!\operatorname{Var}(W)\) for
\emph{one trace}.  Require
\[
  \mu_{\text{vote}}-W_\beta
  \;\ge\;
  \sqrt{\tfrac12\,
        \sigma_{\text{vote}}^{2}\,
        \ln\!\frac{\lambda T}{\epsilon_{\mathrm{target}}}}.
\]

\item \textbf{Economic dial.}  
Choose \((R,P,p_{\min},p_{\max})\) such that
\begin{align}
R
>\frac{\epsilon_{\mathrm H}}{1-\epsilon_{\mathrm H}}\,P\,p_{\max},
\quad
p_{\min}
\ge\frac{\delta}{1-\alpha},
\quad
\alpha
:=\frac{P p_{\max}}{R+P p_{\max}}.
\end{align}
\end{itemize}

With the expected minimum earn per round
\begin{align}
  \displaystyle
  \mu_{\min}:=(1-\epsilon_{\mathrm H})R-\epsilon_{\mathrm H}P p_{\max}>0,
\end{align}
the following hold:

\begin{enumerate}[label=(\alph*), leftmargin=*, itemsep=0.25em]
\item \textbf{Statistical safety.}
\begin{align}
  \Pr[\text{trace fails in }[0,T]]\le\epsilon_{\mathrm{target}}.
\end{align} 
\item \textbf{Honest profitability.}
      \begin{align}
        \Pr[U_{\mathrm{hon}}(T)\le0]
        \;\le\;
        \exp\!\Bigl[
          -\frac{\lambda T\,\mu_{\min}^{2}}
                 {2\sigma_{\mathrm{H}}^{2}+\tfrac23 b\mu_{\min}}
        \Bigr].
      \end{align}

\item \textbf{Malicious loss.}
      \begin{align}
        \Pr[U_{\mathrm{mal}}(T)\ge0]
        \;\le\;
        \exp\!\Bigl[
          -\frac{\lambda T\,(\delta P)^{2}}
                 {2\sigma_{\mathrm{H}}^{2}+\tfrac23 b\delta P}
        \Bigr],
        \quad
        \mathbb E[U_{\mathrm{mal}}(T)]
        \;\le\;
        -\lambda T\,\delta P.
      \end{align}
\end{enumerate}
\end{theorem}

\subsection{Bias Mitigation} 
\begin{table}[t]
\centering
\small
\caption{Comparison of auditing methods on reasoning trace verification. 
Accuracy in \%. Bias score = (self-approval)$-$(other-approval). 
Positive bias = favoritism, Negative bias = criticism.}
\label{tab:method_acc_bias}
\setlength{\tabcolsep}{4pt}
\renewcommand{\arraystretch}{1.2}
\begin{tabularx}{0.65\textwidth}{X c c}
\toprule
\textbf{Method} & \textbf{Acc. (\%)} & \textbf{Bias} \\
\midrule
\rowcolor{UNCblue!10}
\textbf{TRUST (Decentralized)} & \textbf{34.1} & \textbf{--} \\
\midrule
\textbf{Ensemble} & & \\
\rowcolor{gray!20} Supermajority & 30.5 & -- \\
\rowcolor{gray!20} Majority & 26.8 & -- \\
\rowcolor{gray!20} Weighted & 16.9 & -- \\
\midrule
\textbf{Single LLMs} & & \\
GPT-OSS-20B & 60.9 & $+44.4$ \\
Qwen-7B     & 38.4 & -- \\
DeepSeek R1-1.5B & 21.9 & $-11.3$ \\
Llama-3B    & 16.2 & -- \\
Mistral-7B  & 15.9 & -- \\
DeepSeek R1-8B & 15.2 & -- \\
\rowcolor{gray!20}\textbf{Average} & 28.1 & \textbf{+5.5} \\
\bottomrule
\end{tabularx}
\end{table}

Auditing systems are vulnerable to bias, where auditors may favor reasoning traces produced by their own model family or penalize outputs from competing models. 
This creates two common failure modes: (i) \emph{self-favoritism}, where a model systematically approves its own reasoning, and (ii) \emph{self-criticism}, where a model disproportionately rejects its own outputs. 

TRUST is designed to mitigate such bias through three architectural features: 
(1) \emph{Segment-level decomposition}, which breaks reasoning traces into atomic units; 
(2) \emph{Multi-tier consensus}, combining human, LLM, and automated auditors; and 
(3) \emph{Anonymous evaluation}, hiding the source model of each segment.

We construct a benchmark of 200 questions across four domains. 
For each, CoT traces from DeepSeek-R1-1.5B and GPT-OSS-20B are evaluated under three regimes: single auditors, ensemble auditors, and TRUST. 
In \cref{tab:method_acc_bias}, single auditors vary in accuracy (15.2–60.9\%) but show bias (avg.~+5.5). 
Ensembles remove bias but perform poorly (16.9–30.5\%). 
TRUST breaks this tradeoff, achieving higher accuracy (34.1\%) without bias.

\subsection{Human Experiment on Three-Tier Auditing with TRUST Variants}

To evaluate TRUST in a multi-tier auditing setting, we recruited 15 PhD students in computer science as human auditors (Example provided in \cref{sec:human_audit}). 
We constructed a small-scale benchmark of 10 math problems from the \texttt{openai/gsm8k} dataset~\citep{cobbe2021math}, generating CoT traces with DeepSeek-R1-8B and GPT-OSS-20B. 
Auditors participated within the TRUST framework alongside LLM and automated auditors, and we compared performance against both single-LLM auditors and a centralized human-only audit.

Table~\ref{tab:cot_audit_accuracy} reports F1 and Brier scores. 
Single LLM auditors perform poorly, reflecting both limited correctness when providing a whole trace to audit. 
Centralized human auditors achieve moderately stronger results (F1 = $0.77$, Brier = $0.34$), showing clear advantages over single LLMs, but are time-consuming and prone to having systematic error and malicious actors. 
In contrast, TRUST achieves a much higher F1 of $0.89$ and a substantially lower Brier score of $0.074$, demonstrating both accurate and well-calibrated auditing. 
These results highlight the effectiveness of combining human, LLM, and automated auditors under TRUST, outperforming both centralized human audits and single-model baselines.
For TRUST variants, random segment break CoT randomly between 2-6 sentences; fixed-length short segment breaks into two sentences per segment; fixed-length long segment breaks into 5-7 sentences per segment.
All three variants are using GPT-OSS-20B.
\begin{table}[htbp]
\setlength{\tabcolsep}{3pt}
\renewcommand{\arraystretch}{1.2}
\centering
\caption{Human Experiment on Three-Tier Auditing: Performance comparison across single LLM auditors, centralized human auditing, and TRUST variants. 15 PhD students audited 10 GSM8K math problems with CoT traces from DeepSeek-R1-8B and GPT-OSS-20B. Metrics are F1 score (higher better) and Brier score (lower better). Single LLMs perform poorly, centralized human audit achieves F1=0.77, while TRUST w/ HDAG substantially outperforms all baselines (F1=0.89, Brier=0.074) through multi-tier consensus. TRUST variants using random or fixed-length segmentation show degraded performance (F1=0.40), validating the importance of hierarchical decomposition.}
\label{tab:cot_audit_accuracy}
\resizebox{\textwidth}{!}{%
\begin{tabular}{@{}l*{11}{c}@{}}
\toprule
{Metric} & \multicolumn{6}{c}{{Single LLM Auditors}} &  & \multicolumn{4}{c}{{TRUST Variants}} \\
\cmidrule(lr){2-7} \cmidrule(lr){9-12}
 & \makecell{{DeepSeek}\\{R1-8B}} & 
   \makecell{{Qwen2.5}\\{7B}} & 
   \makecell{{Mistral}\\{7B}} & 
   \makecell{{DeepSeek}\\{R1-1.5B}} & 
   \makecell{{GPT-OSS}\\{20B}} & 
   \makecell{{LLaMA}\\{3B}} & 
   \makecell{{Human}\\{Audit}} & 
   \makecell{{TRUST w/}\\{HDAG}} &
   \makecell{{TRUST w/}\\{Random Seg}} &
   \makecell{{TRUST w/}\\{Fix Seg (Short)}} &
   \makecell{{TRUST w/}\\{Fix Seg (Long)}} \\
\midrule
F1           & 0.50 & 0.50 & 0.50 & 0.40 & 0.50 & 0.30 & 0.77 & \textbf{0.89} & 0.40 & 0.40 & 0.40 \\
Brier Score  & 0.500 & 0.486 & 0.500 & 0.500 & 0.544 & 0.890 & 0.34 & \textbf{0.074} & 0.49 & 0.49 & 0.49 \\
\bottomrule
\end{tabular}%
}
\end{table}

\section{Discussion}\label{sec:discuss}
In this section we identify two challenges in \cref{sec:latency,sec:free_rider} for our TRUST framework and their potential solutions.
\subsection{Latency}\label{sec:latency}
A primary limitation of the TRUST framework is the inherent computational latency introduced by it's HDAG reasoning trace decomposition.
Our framework's requirement for multiple LLM calls across five step HDAG construction create significant overhead compare with single pass segmentation and heuristic HDAG decomposition.
Several avenues exist to mitigate this potential bottleneck when we try to deploy in real-time system.
\paragraph{Heuristic HDAG Decomposition.}
The current implementation uses a general-purpose LLM-based decomposition pipeline (Steps 1-5 in \cref{sec:method}) that processes each reasoning trace individually. 
A more efficient approach would employ domain-specific heuristic decomposition algorithms that can act as default templates for common reasoning patterns. 
For instance, mathematical proofs typically follow 
Goal $\rightarrow$ Strategy (proof technique) $\rightarrow$ Tactic (algebraic manipulation) $\rightarrow$ Step (individual equations) $\rightarrow$ Operation (calculations), while code debugging follows Goal $\rightarrow$ Strategy (error localization) $\rightarrow$ Tactic (hypothesis testing) $\rightarrow$ Step (trace analysis) $\rightarrow$ Operation (variable inspection). 
Pre-computed decomposition templates for these patterns could reduce the 5-step LLM pipeline to a single pattern-matching operation, dramatically reducing latency for routine verification tasks.

\paragraph{Specialized Decomposition Models.}
Rather than relying on general-purpose LLMs like GPT-OSS-20B for decomposition (which can take 30-90 seconds per step), the framework could employ lightweight specialized models trained specifically for reasoning trace segmentation. 
A distilled decomposition model (e.g., a fine-tuned 1-3B parameter model) could perform Steps 1-5 in short period of time, reducing end-to-end latency while maintaining the structural guarantees for high quality decomposition.

\paragraph{Adaptive Complexity Routing.}
Not all reasoning traces require full 5-level HDAG decomposition. 
Simple arithmetic problems may only need Goal $\rightarrow$ Operation verification, while complex ethical reasoning demands the complete hierarchy. 
An adaptive routing mechanism that predicts decomposition depth based on problem characteristics (e.g., domain, query length, expected computation steps) could bypass unnecessary abstraction levels, reducing average latency on mixed workloads.

\paragraph{Latency in Distributed Auditing.}
A primary practical limitation of TRUST is the inherent computational latency introduced by its multi-tier audit framework, which requires sequential processing through HDAG decomposition and distributed auditor assignment. 
End-to-end audit latency can range from seconds for simple operations to minutes for complex reasoning requiring human expert review, creating bottlenecks for real-time applications. 
TRUST potentially addresses this through three complementary strategies: 
First, auditor parallelization with time-bound responses over-provisions auditor assignments by assigning $k_t + \Delta$ auditors per segment with strict time limits, accepting the first $k_t$ valid responses and discarding late submissions. 
This ``auditor racing'' trades modest computational redundancy for substantial latency reduction, as the system waits only for the fastest $k_t$ auditors rather than all assigned auditors, masking individual response time variance. 
Second, adaptive LLM substitution dynamically replaces slow human auditors with faster LLM auditors when human queues exceed target latency budgets—low-complexity segments (deterministic Operations, straightforward Tactics) escalate to LLM auditors providing sub-second verification, reserving scarce human attention for genuinely difficult segments requiring domain expertise. 
Third, optimistic verification with deferred finalization returns preliminary results immediately based on fast automated auditors (Computer and LLM tiers) with confidence bounds, enabling downstream applications to proceed while full human consensus finalizes asynchronously in the background, triggering rollback only if discrepancies emerge. 
This approach is particularly effective for reasoning traces where early segments show high confidence, allowing applications to stream partial results during ongoing verification. 
The latency-security tradeoff is configurable per deployment: high-stakes applications (medical diagnosis, financial auditing) enforce strict synchronous verification with extended timeouts, while lower-stakes applications (educational tutoring, creative writing) accept optimistic results with statistical confidence bounds. 

\subsection{Free Rider}\label{sec:free_rider}
The decentralized nature of TRUST introduces economic vulnerabilities most notably the free-rider problem. 
In this context, free-riders are auditors who collect rewards without performing genuine verification work—either by copying votes from honest auditors or by submitting random/superficial evaluations that happen to align with consensus.

\paragraph{Rubber-Stamp Free-Riders.}
A rubber-stamp free-riders employ simple yet damaging strategy: systematically approve all segments without performing genuine verification, exploiting the statistical fact that most LLM reasoning segments are actually correct (i.e. 70-85\% base rate depending on task complexity).
Rubber-stamp free-riders undermine error detection by creating systemic approval bias might present in the real-world reasoning traces. 
For example when a significant fraction of auditors rubber-stamp, the false positive rate for genuinely incorrect segments can nearly double, catastrophically degrading the system's ability to catch errors.
The economic incentive structure reinforces the problem, rubber-stamp free-riders requires zero computational effort while exploiting the base correctness rate to maintain superficially acceptable accuracy, making it profitable for rational actors until their reputation decays to critically low levels.
A potential solution incorporate in TRUST is through strategic honeypot injection combined with behavioral pattern analysis: the system introduces segments with known errors at controlled rates, spanning three difficulty tiers—obvious errors (arithmetic contradictions), subtle errors (sign flips, off-by-one mistakes, unit confusions), and domain-specific errors (medical contraindications, legal statute misapplications).
Rubber-stampers reveal themselves through approval rate consistency: honest auditors show difficulty-dependent approval patterns that decline with segment complexity, while rubber-stampers maintain near-constant high approval across all types and difficulty levels.

\paragraph{Passive Free-Riders.}
Auditors who submit votes after observing partial consensus signals from earlier voters. 
The commit-reveal protocol \cref{sec:method} mitigates this by requiring all auditors to commit cryptographically to their votes before any reveals occur.

\section{Conclusion and Future Works}\label{sec:conclusion}
We introduce TRUST, the first decentralized framework for auditing the reasoning traces of Large Reasoning Models that simultaneously addresses robustness, scalability, opacity, and privacy challenges. TRUST offers an end-to-end pipeline that integrates three key components: a Hierarchical Directed Acyclic Graph (HDAG) decomposition method that breaks Chain-of-Thought reasoning into five abstraction levels; a multi-tier consensus mechanism that routes verification tasks to automated checkers, LLMs, and human experts based on complexity; and a blockchain-based infrastructure with cryptographic privacy preservation that ensures transparent audit trails while protecting proprietary model internals. It supports verification across diverse reasoning domains and enhances transparency through decentralized consensus and immutable audit records.

Our experiments demonstrate TRUST's effectiveness across correctness, bias mitigation, and human-in-the-loop evaluation using multiple datasets and state-of-the-art models. TRUST consistently outperforms centralized ensemble methods and single auditors while maintaining graceful degradation under adversarial conditions. These results highlight its robustness against corruption and effectiveness in eliminating systematic bias while preserving accuracy.

In addition, our theoretical framework provides formal guarantees that honest auditors profit while malicious actors incur losses, creating sustainable economic incentives. Overall, TRUST pioneers decentralized AI auditing as a practical pathway toward safe and accountable deployment of reasoning-capable AI systems in high-stakes domains. It makes transparent oversight of proprietary AI systems accessible without compromising intellectual property rights.

\textbf{Future work} will involve developing more sophisticated graph decomposition methods to capture richer reasoning dependencies, integrating adaptive auditor assignment strategies that leverage task-specific expertise, and extending the framework to support dynamic, interactive reasoning settings. We also plan to explore cross-model reasoning consistency verification for multi-agent scenarios and investigate integration with federated learning frameworks to accelerate trustworthy AI deployment.

\newpage
\bibliography{999_reference}

\begin{thebibliography}{74}
\providecommand{\natexlab}[1]{#1}
\providecommand{\url}[1]{\texttt{#1}}
\expandafter\ifx\csname urlstyle\endcsname\relax
  \providecommand{\doi}[1]{doi: #1}\else
  \providecommand{\doi}{doi: \begingroup \urlstyle{rm}\Url}\fi

\bibitem[AI(2023)]{ai2023artificial}
AI, N.
\newblock Artificial intelligence risk management framework (ai rmf 1.0).
\newblock \emph{URL: https://nvlpubs. nist. gov/nistpubs/ai/nist. ai}, pp.\  100--1, 2023.

\bibitem[Bai et~al.(2022)Bai, Jones, Ndousse, Askell, Chen, DasSarma, Drain, Fort, Ganguli, Henighan, et~al.]{bai2022training}
Bai, Y., Jones, A., Ndousse, K., Askell, A., Chen, A., DasSarma, N., Drain, D., Fort, S., Ganguli, D., Henighan, T., et~al.
\newblock Training a helpful and harmless assistant with reinforcement learning from human feedback.
\newblock \emph{arXiv preprint arXiv:2204.05862}, 2022.

\bibitem[Bender et~al.(2021)Bender, Gebru, McMillan-Major, and Shmitchell]{bender2021dangers}
Bender, E.~M., Gebru, T., McMillan-Major, A., and Shmitchell, S.
\newblock On the dangers of stochastic parrots: Can language models be too big?
\newblock In \emph{Proceedings of the 2021 ACM conference on fairness, accountability, and transparency}, pp.\  610--623, 2021.

\bibitem[Besta et~al.(2024)Besta, Blach, Kubicek, Gerstenberger, Podstawski, Gianinazzi, Gajda, Lehmann, Niewiadomski, Nyczyk, et~al.]{besta2024graph}
Besta, M., Blach, N., Kubicek, A., Gerstenberger, R., Podstawski, M., Gianinazzi, L., Gajda, J., Lehmann, T., Niewiadomski, H., Nyczyk, P., et~al.
\newblock Graph of thoughts: Solving elaborate problems with large language models.
\newblock In \emph{Proceedings of the AAAI conference on artificial intelligence}, volume~38, pp.\  17682--17690, 2024.

\bibitem[Bommasani et~al.(2023)Bommasani, Klyman, Longpre, Kapoor, Maslej, Xiong, Zhang, and Liang]{bommasani2023foundation}
Bommasani, R., Klyman, K., Longpre, S., Kapoor, S., Maslej, N., Xiong, B., Zhang, D., and Liang, P.
\newblock The foundation model transparency index.
\newblock \emph{arXiv preprint arXiv:2310.12941}, 2023.

\bibitem[Cai et~al.(2025)Cai, Shi, Zhao, and Song]{cai2025you}
Cai, W., Shi, T., Zhao, X., and Song, D.
\newblock Are you getting what you pay for? auditing model substitution in llm apis.
\newblock \emph{arXiv preprint arXiv:2504.04715}, 2025.

\bibitem[Carlini et~al.(2021)Carlini, Tramer, Wallace, Jagielski, Herbert-Voss, Lee, Roberts, Brown, Song, Erlingsson, et~al.]{carlini2021extracting}
Carlini, N., Tramer, F., Wallace, E., Jagielski, M., Herbert-Voss, A., Lee, K., Roberts, A., Brown, T., Song, D., Erlingsson, U., et~al.
\newblock Extracting training data from large language models.
\newblock In \emph{30th USENIX security symposium (USENIX Security 21)}, pp.\  2633--2650, 2021.

\bibitem[Castro et~al.(1999)Castro, Liskov, et~al.]{castro1999practical}
Castro, M., Liskov, B., et~al.
\newblock Practical byzantine fault tolerance.
\newblock In \emph{OsDI}, volume~99, pp.\  173--186, 1999.

\bibitem[Chalkidis et~al.(2021)Chalkidis, Jana, Hartung, Bommarito, Androutsopoulos, Katz, and Aletras]{chalkidis2021lexglue}
Chalkidis, I., Jana, A., Hartung, D., Bommarito, M., Androutsopoulos, I., Katz, D.~M., and Aletras, N.
\newblock Lexglue: A benchmark dataset for legal language understanding in english.
\newblock \emph{arXiv preprint arXiv:2110.00976}, 2021.

\bibitem[Chen et~al.(2024)Chen, Waiwitlikhit, Stoica, and Kang]{chen2024zkml}
Chen, B.-J., Waiwitlikhit, S., Stoica, I., and Kang, D.
\newblock Zkml: An optimizing system for ml inference in zero-knowledge proofs.
\newblock In \emph{Proceedings of the Nineteenth European Conference on Computer Systems}, pp.\  560--574, 2024.

\bibitem[Chen et~al.(2022)Chen, Ma, Wang, and Cohen]{chen2022program}
Chen, W., Ma, X., Wang, X., and Cohen, W.~W.
\newblock Program of thoughts prompting: Disentangling computation from reasoning for numerical reasoning tasks.
\newblock \emph{arXiv preprint arXiv:2211.12588}, 2022.

\bibitem[Chen et~al.(2025)Chen, Benton, Radhakrishnan, Uesato, Denison, Schulman, Somani, Hase, Wagner, Roger, et~al.]{chen2025reasoning}
Chen, Y., Benton, J., Radhakrishnan, A., Uesato, J., Denison, C., Schulman, J., Somani, A., Hase, P., Wagner, M., Roger, F., et~al.
\newblock Reasoning models don't always say what they think.
\newblock \emph{arXiv preprint arXiv:2505.05410}, 2025.

\bibitem[Cobbe et~al.(2021{\natexlab{a}})Cobbe, Kosaraju, Bavarian, Chen, Jun, Kaiser, Plappert, Tworek, Hilton, Nakano, Hesse, and Schulman]{cobbe2021gsm8k}
Cobbe, K., Kosaraju, V., Bavarian, M., Chen, M., Jun, H., Kaiser, L., Plappert, M., Tworek, J., Hilton, J., Nakano, R., Hesse, C., and Schulman, J.
\newblock Training verifiers to solve math word problems.
\newblock \emph{arXiv preprint arXiv:2110.14168}, 2021{\natexlab{a}}.

\bibitem[Cobbe et~al.(2021{\natexlab{b}})Cobbe, Kosaraju, Bavarian, Chen, Jun, Kaiser, Plappert, Tworek, Hilton, Nakano, Hesse, and Schulman]{cobbe2021math}
Cobbe, K., Kosaraju, V., Bavarian, M., Chen, M., Jun, H., Kaiser, L., Plappert, M., Tworek, J., Hilton, J., Nakano, R., Hesse, C., and Schulman, J.
\newblock Training verifiers to solve math word problems, 2021{\natexlab{b}}.
\newblock URL \url{https://arxiv.org/abs/2110.14168}.

\bibitem[COM(2021)]{com2021laying}
COM, E.
\newblock Laying down harmonised rules on artificial intelligence (artificial intelligence act) and amending certain union legislative acts.
\newblock \emph{Proposal for a regulation of the European parliament and of the council}, 2021.

\bibitem[Costan \& Devadas(2016)Costan and Devadas]{costan2016intel}
Costan, V. and Devadas, S.
\newblock Intel sgx explained.
\newblock \emph{Cryptology ePrint Archive}, 2016.

\bibitem[Creswell et~al.(2022)Creswell, Shanahan, and Higgins]{creswell2022selection}
Creswell, A., Shanahan, M., and Higgins, I.
\newblock Selection-inference: Exploiting large language models for interpretable logical reasoning.
\newblock \emph{arXiv preprint arXiv:2205.09712}, 2022.

\bibitem[DeepSeek-AI et~al.(2025)DeepSeek-AI, Guo, Yang, Zhang, Song, Zhang, Xu, Zhu, Ma, Wang, Bi, Zhang, Yu, Wu, Wu, Gou, Shao, Li, Gao, Liu, Xue, Wang, Wu, Feng, Lu, Zhao, Deng, Zhang, Ruan, Dai, Chen, Ji, Li, Lin, Dai, Luo, Hao, Chen, Li, Zhang, Bao, Xu, Wang, Ding, Xin, Gao, Qu, Li, Guo, Li, Wang, Chen, Yuan, Qiu, Li, Cai, Ni, Liang, Chen, Dong, Hu, Gao, Guan, Huang, Yu, Wang, Zhang, Zhao, Wang, Zhang, Xu, Xia, Zhang, Zhang, Tang, Li, Wang, Li, Tian, Huang, Zhang, Wang, Chen, Du, Ge, Zhang, Pan, Wang, Chen, Jin, Chen, Lu, Zhou, Chen, Ye, Wang, Yu, Zhou, Pan, Li, Zhou, Wu, Ye, Yun, Pei, Sun, Wang, Zeng, Zhao, Liu, Liang, Gao, Yu, Zhang, Xiao, An, Liu, Wang, Chen, Nie, Cheng, Liu, Xie, Liu, Yang, Li, Su, Lin, Li, Jin, Shen, Chen, Sun, Wang, Song, Zhou, Wang, Shan, Li, Wang, Wei, Zhang, Xu, Li, Zhao, Sun, Wang, Yu, Zhang, Shi, Xiong, He, Piao, Wang, Tan, Ma, Liu, Guo, Ou, Wang, gong, Zou, He, Xiong, Luo, You, Liu, Zhou, Zhu, Xu, Huang, Li, Zheng, Zhu, Ma, Tang, Zha, Yan, Ren, Ren, Sha, Fu, Xu, Xie, Zhang,
  Hao, Ma, Yan, Wu, Gu, Zhu, Liu, Li, Xie, Song, Pan, Huang, Xu, Zhang, and Zhang]{deepseekr1}
DeepSeek-AI, Guo, D., Yang, D., Zhang, H., Song, J., Zhang, R., Xu, R., Zhu, Q., Ma, S., Wang, P., Bi, X., Zhang, X., Yu, X., Wu, Y., Wu, Z.~F., Gou, Z., Shao, Z., Li, Z., Gao, Z., Liu, A., Xue, B., Wang, B., Wu, B., Feng, B., Lu, C., Zhao, C., Deng, C., Zhang, C., Ruan, C., Dai, D., Chen, D., Ji, D., Li, E., Lin, F., Dai, F., Luo, F., Hao, G., Chen, G., Li, G., Zhang, H., Bao, H., Xu, H., Wang, H., Ding, H., Xin, H., Gao, H., Qu, H., Li, H., Guo, J., Li, J., Wang, J., Chen, J., Yuan, J., Qiu, J., Li, J., Cai, J.~L., Ni, J., Liang, J., Chen, J., Dong, K., Hu, K., Gao, K., Guan, K., Huang, K., Yu, K., Wang, L., Zhang, L., Zhao, L., Wang, L., Zhang, L., Xu, L., Xia, L., Zhang, M., Zhang, M., Tang, M., Li, M., Wang, M., Li, M., Tian, N., Huang, P., Zhang, P., Wang, Q., Chen, Q., Du, Q., Ge, R., Zhang, R., Pan, R., Wang, R., Chen, R.~J., Jin, R.~L., Chen, R., Lu, S., Zhou, S., Chen, S., Ye, S., Wang, S., Yu, S., Zhou, S., Pan, S., Li, S.~S., Zhou, S., Wu, S., Ye, S., Yun, T., Pei, T., Sun, T., Wang, T., Zeng, W.,
  Zhao, W., Liu, W., Liang, W., Gao, W., Yu, W., Zhang, W., Xiao, W.~L., An, W., Liu, X., Wang, X., Chen, X., Nie, X., Cheng, X., Liu, X., Xie, X., Liu, X., Yang, X., Li, X., Su, X., Lin, X., Li, X.~Q., Jin, X., Shen, X., Chen, X., Sun, X., Wang, X., Song, X., Zhou, X., Wang, X., Shan, X., Li, Y.~K., Wang, Y.~Q., Wei, Y.~X., Zhang, Y., Xu, Y., Li, Y., Zhao, Y., Sun, Y., Wang, Y., Yu, Y., Zhang, Y., Shi, Y., Xiong, Y., He, Y., Piao, Y., Wang, Y., Tan, Y., Ma, Y., Liu, Y., Guo, Y., Ou, Y., Wang, Y., gong, Y., Zou, Y., He, Y., Xiong, Y., Luo, Y., You, Y., Liu, Y., Zhou, Y., Zhu, Y.~X., Xu, Y., Huang, Y., Li, Y., Zheng, Y., Zhu, Y., Ma, Y., Tang, Y., Zha, Y., Yan, Y., Ren, Z.~Z., Ren, Z., Sha, Z., Fu, Z., Xu, Z., Xie, Z., Zhang, Z., Hao, Z., Ma, Z., Yan, Z., Wu, Z., Gu, Z., Zhu, Z., Liu, Z., Li, Z., Xie, Z., Song, Z., Pan, Z., Huang, Z., Xu, Z., Zhang, Z., and Zhang, Z.
\newblock Deepseek-r1: Incentivizing reasoning capability in llms via reinforcement learning, 2025.
\newblock URL \url{https://arxiv.org/abs/2501.12948}.

\bibitem[Gao et~al.(2023)Gao, Madaan, Zhou, Alon, Liu, Yang, Callan, and Neubig]{gao2023pal}
Gao, L., Madaan, A., Zhou, S., Alon, U., Liu, P., Yang, Y., Callan, J., and Neubig, G.
\newblock Pal: Program-aided language models.
\newblock In \emph{International Conference on Machine Learning}, pp.\  10764--10799. PMLR, 2023.

\bibitem[Ghodsi et~al.(2017)Ghodsi, Gu, and Garg]{ghodsi2017safetynets}
Ghodsi, Z., Gu, T., and Garg, S.
\newblock Safetynets: Verifiable execution of deep neural networks on an untrusted cloud.
\newblock \emph{Advances in Neural Information Processing Systems}, 30, 2017.

\bibitem[Guo et~al.(2025)Guo, Yang, Zhang, Song, Wang, Zhu, Xu, Zhang, Ma, Bi, et~al.]{guo2025deepseek}
Guo, D., Yang, D., Zhang, H., Song, J., Wang, P., Zhu, Q., Xu, R., Zhang, R., Ma, S., Bi, X., et~al.
\newblock Deepseek-r1 incentivizes reasoning in llms through reinforcement learning.
\newblock \emph{Nature}, 645\penalty0 (8081):\penalty0 633--638, 2025.

\bibitem[Huang \& Chang(2022)Huang and Chang]{huang2022towards}
Huang, J. and Chang, K. C.-C.
\newblock Towards reasoning in large language models: A survey.
\newblock \emph{arXiv preprint arXiv:2212.10403}, 2022.

\bibitem[Jacovi et~al.(2024)Jacovi, Bitton, Bohnet, Herzig, Honovich, Tseng, Collins, Aharoni, and Geva]{jacovi2024chain}
Jacovi, A., Bitton, Y., Bohnet, B., Herzig, J., Honovich, O., Tseng, M., Collins, M., Aharoni, R., and Geva, M.
\newblock A chain-of-thought is as strong as its weakest link: A benchmark for verifiers of reasoning chains.
\newblock \emph{arXiv preprint arXiv:2402.00559}, 2024.

\bibitem[Jaech et~al.(2024)Jaech, Kalai, Lerer, Richardson, El-Kishky, Low, Helyar, Madry, Beutel, Carney, et~al.]{jaech2024openai}
Jaech, A., Kalai, A., Lerer, A., Richardson, A., El-Kishky, A., Low, A., Helyar, A., Madry, A., Beutel, A., Carney, A., et~al.
\newblock Openai o1 system card.
\newblock \emph{arXiv preprint arXiv:2412.16720}, 2024.

\bibitem[Johannes~Welbl(2017)]{SciQ}
Johannes~Welbl, Nelson F.~Liu, M.~G.
\newblock Crowdsourcing multiple choice science questions.
\newblock 2017.

\bibitem[Khalifa et~al.(2025)Khalifa, Agarwal, Logeswaran, Kim, Peng, Lee, Lee, and Wang]{khalifa2025process}
Khalifa, M., Agarwal, R., Logeswaran, L., Kim, J., Peng, H., Lee, M., Lee, H., and Wang, L.
\newblock Process reward models that think.
\newblock \emph{arXiv preprint arXiv:2504.16828}, 2025.

\bibitem[Kojima et~al.(2022)Kojima, Gu, Reid, Matsuo, and Iwasawa]{kojima2022large}
Kojima, T., Gu, S.~S., Reid, M., Matsuo, Y., and Iwasawa, Y.
\newblock Large language models are zero-shot reasoners.
\newblock \emph{Advances in neural information processing systems}, 35:\penalty0 22199--22213, 2022.

\bibitem[Kothapalli et~al.(2025)Kothapalli, Firooz, and Sanjabi]{kothapalli2025cot}
Kothapalli, V., Firooz, H., and Sanjabi, M.
\newblock Cot-icl lab: A synthetic framework for studying chain-of-thought learning from in-context demonstrations.
\newblock \emph{arXiv preprint arXiv:2502.15132}, 2025.

\bibitem[Lab(2024)]{mmlu_pro_cot_train}
Lab, U.-M.~L.
\newblock Mmlu-pro-cot-train-labeled.
\newblock \url{https://huggingface.co/datasets/UW-Madison-Lee-Lab/MMLU-Pro-CoT-Train-Labeled}, 2024.
\newblock Accessed: 2025-09-24.

\bibitem[Lamport et~al.(2019)Lamport, Shostak, and Pease]{lamport2019byzantine}
Lamport, L., Shostak, R., and Pease, M.
\newblock The byzantine generals problem.
\newblock In \emph{Concurrency: the works of leslie lamport}, pp.\  203--226. 2019.

\bibitem[Lanham et~al.(2023)Lanham, Chen, Radhakrishnan, Steiner, Denison, Hernandez, Li, Durmus, Hubinger, Kernion, et~al.]{lanham2023measuring}
Lanham, T., Chen, A., Radhakrishnan, A., Steiner, B., Denison, C., Hernandez, D., Li, D., Durmus, E., Hubinger, E., Kernion, J., et~al.
\newblock Measuring faithfulness in chain-of-thought reasoning.
\newblock \emph{arXiv preprint arXiv:2307.13702}, 2023.

\bibitem[Liang et~al.(2022)Liang, Bommasani, Lee, Tsipras, Soylu, Yasunaga, Zhang, Narayanan, Wu, Kumar, et~al.]{liang2022holistic}
Liang, P., Bommasani, R., Lee, T., Tsipras, D., Soylu, D., Yasunaga, M., Zhang, Y., Narayanan, D., Wu, Y., Kumar, A., et~al.
\newblock Holistic evaluation of language models.
\newblock \emph{arXiv preprint arXiv:2211.09110}, 2022.

\bibitem[Lightman et~al.(2023)Lightman, Kosaraju, Burda, Edwards, Baker, Lee, Leike, Schulman, Sutskever, and Cobbe]{lightman2023let}
Lightman, H., Kosaraju, V., Burda, Y., Edwards, H., Baker, B., Lee, T., Leike, J., Schulman, J., Sutskever, I., and Cobbe, K.
\newblock Let's verify step by step.
\newblock In \emph{The Twelfth International Conference on Learning Representations}, 2023.

\bibitem[Ling et~al.(2023)Ling, Fang, Li, Huang, Lee, Memisevic, and Su]{ling2023deductive}
Ling, Z., Fang, Y., Li, X., Huang, Z., Lee, M., Memisevic, R., and Su, H.
\newblock Deductive verification of chain-of-thought reasoning.
\newblock \emph{Advances in Neural Information Processing Systems}, 36:\penalty0 36407--36433, 2023.

\bibitem[Liu et~al.(2021)Liu, Xie, and Zhang]{liu2021zkcnn}
Liu, T., Xie, X., and Zhang, Y.
\newblock Zkcnn: Zero knowledge proofs for convolutional neural network predictions and accuracy.
\newblock In \emph{Proceedings of the 2021 ACM SIGSAC Conference on Computer and Communications Security}, pp.\  2968--2985, 2021.

\bibitem[Luo et~al.(2025)Luo, Song, Zhang, Liu, Wang, Chen, Su, and Zheng]{luo2025decon}
Luo, Y., Song, Y., Zhang, X., Liu, J., Wang, W., Chen, G., Su, W., and Zheng, B.
\newblock Deconstructing long chain-of-thought: A structured reasoning optimization framework for long cot distillation.
\newblock \emph{arXiv preprint arXiv:2503.16385}, 2025.

\bibitem[Mitchell et~al.(2019)Mitchell, Wu, Zaldivar, Barnes, Vasserman, Hutchinson, Spitzer, Raji, and Gebru]{mitchell2019model}
Mitchell, M., Wu, S., Zaldivar, A., Barnes, P., Vasserman, L., Hutchinson, B., Spitzer, E., Raji, I.~D., and Gebru, T.
\newblock Model cards for model reporting.
\newblock In \emph{Proceedings of the conference on fairness, accountability, and transparency}, pp.\  220--229, 2019.

\bibitem[Nakano et~al.(2021)Nakano, Hilton, Balaji, Wu, Ouyang, Kim, Hesse, Jain, Kosaraju, Saunders, et~al.]{nakano2021webgpt}
Nakano, R., Hilton, J., Balaji, S., Wu, J., Ouyang, L., Kim, C., Hesse, C., Jain, S., Kosaraju, V., Saunders, W., et~al.
\newblock Webgpt: Browser-assisted question-answering with human feedback.
\newblock \emph{arXiv preprint arXiv:2112.09332}, 2021.

\bibitem[Nasr et~al.(2023)Nasr, Carlini, Hayase, Jagielski, Cooper, Ippolito, Choquette-Choo, Wallace, Tram{\`e}r, and Lee]{nasr2023scalable}
Nasr, M., Carlini, N., Hayase, J., Jagielski, M., Cooper, A.~F., Ippolito, D., Choquette-Choo, C.~A., Wallace, E., Tram{\`e}r, F., and Lee, K.
\newblock Scalable extraction of training data from (production) language models.
\newblock \emph{arXiv preprint arXiv:2311.17035}, 2023.

\bibitem[NVIDIA et~al.(2025)NVIDIA, :, Basant, Khairnar, Paithankar, Khattar, Renduchintala, Malte, Bercovich, Hazare, Rico, Ficek, Kondratenko, Shaposhnikov, Bukharin, Taghibakhshi, Barton, Mahabaleshwarkar, Shen, Tao, Guan, Shors, Mandarwal, Mehta, Venkatesan, Sharabiani, Aithal, Poojary, Dattagupta, Buddharaju, Zhu, Simkin, Kartal, Rouhani, Chen, Ginsburg, Norick, Yu, Catanzaro, Wang, Truong, Mungekar, Patel, Alexiuk, Munley, Parisien, Su, Afrimi, Korzekwa, Rohrer, Gitman, Mosallanezhad, Narayanan, Rekesh, Yared, Pykhtar, Ahn, Riach, Long, Ning, Chung, Galinkin, Bakhturina, Prasad, Shen, Qian, Elisha, Sharma, Ross, Ngo, Sahota, Wang, Shin, Huang, Cunningham, Gitman, Moshkov, Jung, Kautz, Scowcroft, Casper, Zhang, Zeng, Zhang, Xue, Huang, Conway, Kamalu, Cohen, Jennings, Vialard, Yi, Parmar, Briski, Cheung, Luna, Wyss, Santhanam, Kong, Pawelec, Anik, Li, Ahmadian, McAfee, Sleiman, Derczynski, Vega, de~Melo, Sreedhar, Chochowski, Cai, Kliegl, Stepniewska-Dziubinska, Novikov, Samadi, Price, Boubdir, Boone,
  Evans, Bien, Zawalski, Martinez, Chrzanowski, Shoeybi, Patwary, Dhameja, Assaf, Habibi, Bhatia, Pope, Tajbakhsh, Juluru, Rybakov, Hrinchuk, Kuchaiev, Olabiyi, Ribalta, Subramanian, Chadha, Molchanov, Dykas, Jin, Bialecki, Januszewski, Thalasta, Gaikwad, Varshney, Gundecha, Tredak, Mahabadi, Patel, El-Yaniv, Rajan, Cheruvu, Shahbazyan, Borkar, Gala, Waleffe, Zhang, Hewett, Prenger, Jain, Kriman, Satheesh, Kaji, Yurick, Muralidharan, Narenthiran, Bak, Sameni, Han, Ramasamy, Ghosh, Sreenivas, Thomas, Diao, Gopal, Prabhumoye, Toshniwal, Ding, Singh, Jain, Majumdar, Singhal, Alborghetti, Akter, Kong, Moon, Hliwiak, Asida, Wang, Konuk, Vashishth, Poon, Karpas, Noroozi, Srinivasan, Korthikanti, Fugro, Kalluru, Kurin, Lavrukhin, Ahmad, Du, Byeon, Lu, Dong, Karnati, Choi, Zhang, Lin, Fu, Suhara, Dong, Li, Zhu, and Chen]{nvidia2025nvidianemotronnano2}
NVIDIA, :, Basant, A., Khairnar, A., Paithankar, A., Khattar, A., Renduchintala, A., Malte, A., Bercovich, A., Hazare, A., Rico, A., Ficek, A., Kondratenko, A., Shaposhnikov, A., Bukharin, A., Taghibakhshi, A., Barton, A., Mahabaleshwarkar, A.~S., Shen, A., Tao, A., Guan, A., Shors, A., Mandarwal, A., Mehta, A., Venkatesan, A., Sharabiani, A., Aithal, A., Poojary, A., Dattagupta, A., Buddharaju, B., Zhu, B., Simkin, B., Kartal, B., Rouhani, B.~D., Chen, B., Ginsburg, B., Norick, B., Yu, B., Catanzaro, B., Wang, C., Truong, C., Mungekar, C., Patel, C., Alexiuk, C., Munley, C., Parisien, C., Su, D., Afrimi, D., Korzekwa, D., Rohrer, D., Gitman, D., Mosallanezhad, D., Narayanan, D., Rekesh, D., Yared, D., Pykhtar, D., Ahn, D., Riach, D., Long, E., Ning, E., Chung, E., Galinkin, E., Bakhturina, E., Prasad, G., Shen, G., Qian, H., Elisha, H., Sharma, H., Ross, H., Ngo, H., Sahota, H., Wang, H., Shin, H.~C., Huang, H., Cunningham, I., Gitman, I., Moshkov, I., Jung, J., Kautz, J., Scowcroft, J.~P., Casper, J.,
  Zhang, J., Zeng, J., Zhang, J., Xue, J., Huang, J., Conway, J., Kamalu, J., Cohen, J., Jennings, J., Vialard, J.~V., Yi, J., Parmar, J., Briski, K., Cheung, K., Luna, K., Wyss, K., Santhanam, K., Kong, K., Pawelec, K., Anik, K., Li, K., Ahmadian, K., McAfee, L., Sleiman, L., Derczynski, L., Vega, L., de~Melo, M.~R., Sreedhar, M.~N., Chochowski, M., Cai, M., Kliegl, M., Stepniewska-Dziubinska, M., Novikov, M., Samadi, M., Price, M., Boubdir, M., Boone, M., Evans, M., Bien, M., Zawalski, M., Martinez, M., Chrzanowski, M., Shoeybi, M., Patwary, M., Dhameja, N., Assaf, N., Habibi, N., Bhatia, N., Pope, N., Tajbakhsh, N., Juluru, N.~K., Rybakov, O., Hrinchuk, O., Kuchaiev, O., Olabiyi, O., Ribalta, P., Subramanian, P., Chadha, P., Molchanov, P., Dykas, P., Jin, P., Bialecki, P., Januszewski, P., Thalasta, P., Gaikwad, P., Varshney, P., Gundecha, P., Tredak, P., Mahabadi, R.~K., Patel, R., El-Yaniv, R., Rajan, R., Cheruvu, R., Shahbazyan, R., Borkar, R., Gala, R., Waleffe, R., Zhang, R., Hewett, R.~J., Prenger,
  R., Jain, S., Kriman, S., Satheesh, S., Kaji, S., Yurick, S., Muralidharan, S., Narenthiran, S., Bak, S., Sameni, S., Han, S., Ramasamy, S., Ghosh, S., Sreenivas, S.~T., Thomas, S., Diao, S., Gopal, S., Prabhumoye, S., Toshniwal, S., Ding, S., Singh, S., Jain, S., Majumdar, S., Singhal, S., Alborghetti, S., Akter, S.~N., Kong, T., Moon, T., Hliwiak, T., Asida, T., Wang, T., Konuk, T., Vashishth, T., Poon, T., Karpas, U., Noroozi, V., Srinivasan, V., Korthikanti, V., Fugro, V., Kalluru, V., Kurin, V., Lavrukhin, V., Ahmad, W.~U., Du, W., Byeon, W., Lu, X., Dong, X., Karnati, Y., Choi, Y., Zhang, Y., Lin, Y., Fu, Y., Suhara, Y., Dong, Z., Li, Z., Zhu, Z., and Chen, Z.
\newblock Nvidia nemotron nano 2: An accurate and efficient hybrid mamba-transformer reasoning model, 2025.
\newblock URL \url{https://arxiv.org/abs/2508.14444}.

\bibitem[Nye et~al.(2021)Nye, Andreassen, Gur-Ari, Michalewski, Austin, Bieber, Dohan, Lewkowycz, Bosma, Luan, et~al.]{nye2021show}
Nye, M., Andreassen, A.~J., Gur-Ari, G., Michalewski, H., Austin, J., Bieber, D., Dohan, D., Lewkowycz, A., Bosma, M., Luan, D., et~al.
\newblock Show your work: Scratchpads for intermediate computation with language models.
\newblock 2021.

\bibitem[OECD(2019)]{oecd2019principles}
OECD, F.-T. C. A.~N.
\newblock Principles on artificial intelligence, 2019.

\bibitem[OpenAI(2025)]{openai2025gptoss}
OpenAI.
\newblock gpt-oss-120b \& gpt-oss-20b model card, 2025.
\newblock URL \url{https://arxiv.org/abs/2508.10925}.

\bibitem[Pal et~al.(2022)Pal, Umapathi, and Sankarasubbu]{medmcqa}
Pal, A., Umapathi, L.~K., and Sankarasubbu, M.
\newblock Medmcqa: A large-scale multi-subject multi-choice dataset for medical domain question answering.
\newblock In Flores, G., Chen, G.~H., Pollard, T., Ho, J.~C., and Naumann, T. (eds.), \emph{Proceedings of the Conference on Health, Inference, and Learning}, volume 174 of \emph{Proceedings of Machine Learning Research}, pp.\  248--260. PMLR, 07--08 Apr 2022.
\newblock URL \url{https://proceedings.mlr.press/v174/pal22a.html}.

\bibitem[Pasquini et~al.(2025)Pasquini, Kornaropoulos, and Ateniese]{pasquini2025llmmap}
Pasquini, D., Kornaropoulos, E.~M., and Ateniese, G.
\newblock $\{$LLMmap$\}$: Fingerprinting for large language models.
\newblock In \emph{34th USENIX Security Symposium (USENIX Security 25)}, pp.\  299--318, 2025.

\bibitem[Peng et~al.(2025)Peng, Wang, Zhao, Liao, Lin, Liu, Cao, Shi, Yang, and Zhang]{peng2025survey}
Peng, Z., Wang, T., Zhao, C., Liao, G., Lin, Z., Liu, Y., Cao, B., Shi, L., Yang, Q., and Zhang, S.
\newblock A survey of zero-knowledge proof based verifiable machine learning.
\newblock \emph{arXiv preprint arXiv:2502.18535}, 2025.

\bibitem[Perez \& Ribeiro(2022)Perez and Ribeiro]{perez2022ignore}
Perez, F. and Ribeiro, I.
\newblock Ignore previous prompt: Attack techniques for language models.
\newblock \emph{arXiv preprint arXiv:2211.09527}, 2022.

\bibitem[Rajpurkar et~al.(2016)Rajpurkar, Zhang, Lopyrev, and Liang]{squad}
Rajpurkar, P., Zhang, J., Lopyrev, K., and Liang, P.
\newblock {SQ}u{AD}: 100,000+ questions for machine comprehension of text.
\newblock In Su, J., Duh, K., and Carreras, X. (eds.), \emph{Proceedings of the 2016 Conference on Empirical Methods in Natural Language Processing}, pp.\  2383--2392, Austin, Texas, November 2016. Association for Computational Linguistics.
\newblock \doi{10.18653/v1/D16-1264}.
\newblock URL \url{https://aclanthology.org/D16-1264}.

\bibitem[Sabt et~al.(2015)Sabt, Achemlal, and Bouabdallah]{sabt2015trusted}
Sabt, M., Achemlal, M., and Bouabdallah, A.
\newblock Trusted execution environment: What it is, and what it is not.
\newblock In \emph{2015 IEEE Trustcom/BigDataSE/Ispa}, volume~1, pp.\  57--64. IEEE, 2015.

\bibitem[Sarafyazd \& Jazayeri(2019)Sarafyazd and Jazayeri]{morteza2019hierarchical}
Sarafyazd, M. and Jazayeri, M.
\newblock Hierarchical reasoning by neural circuits in the frontal cortex.
\newblock \emph{Science}, 364\penalty0 (6441):\penalty0 eaav8911, 2019.
\newblock \doi{10.1126/science.aav8911}.
\newblock URL \url{https://www.science.org/doi/abs/10.1126/science.aav8911}.

\bibitem[Schick et~al.(2023)Schick, Dwivedi-Yu, Dess{\`\i}, Raileanu, Lomeli, Hambro, Zettlemoyer, Cancedda, and Scialom]{schick2023toolformer}
Schick, T., Dwivedi-Yu, J., Dess{\`\i}, R., Raileanu, R., Lomeli, M., Hambro, E., Zettlemoyer, L., Cancedda, N., and Scialom, T.
\newblock Toolformer: Language models can teach themselves to use tools.
\newblock \emph{Advances in Neural Information Processing Systems}, 36:\penalty0 68539--68551, 2023.

\bibitem[Schnabl et~al.(2025)Schnabl, Hugenroth, Marino, and Beresford]{schnabl2025attestable}
Schnabl, C., Hugenroth, D., Marino, B., and Beresford, A.~R.
\newblock Attestable audits: Verifiable ai safety benchmarks using trusted execution environments.
\newblock \emph{arXiv preprint arXiv:2506.23706}, 2025.

\bibitem[Shi et~al.(2024)Shi, Wang, Yin, Chen, Chang, and Hsieh]{shi2024red}
Shi, Z., Wang, Y., Yin, F., Chen, X., Chang, K.-W., and Hsieh, C.-J.
\newblock Red teaming language model detectors with language models.
\newblock \emph{Transactions of the Association for Computational Linguistics}, 12:\penalty0 174--189, 2024.

\bibitem[Singhal et~al.(2023)Singhal, Azizi, Tu, Mahdavi, Wei, Chung, Scales, Tanwani, Cole-Lewis, Pfohl, et~al.]{singhal2023large}
Singhal, K., Azizi, S., Tu, T., Mahdavi, S.~S., Wei, J., Chung, H.~W., Scales, N., Tanwani, A., Cole-Lewis, H., Pfohl, S., et~al.
\newblock Large language models encode clinical knowledge.
\newblock \emph{Nature}, 620\penalty0 (7972):\penalty0 172--180, 2023.

\bibitem[South et~al.(2024)South, Camuto, Jain, Nguyen, Mahari, Paquin, Morton, and Pentland]{south2024verifiable}
South, T., Camuto, A., Jain, S., Nguyen, S., Mahari, R., Paquin, C., Morton, J., and Pentland, A.
\newblock Verifiable evaluations of machine learning models using zksnarks.
\newblock \emph{arXiv preprint arXiv:2402.02675}, 2024.

\bibitem[Sun et~al.(2025)Sun, Wang, Tian, Liu, Shen, He, He, Ye, Wang, and Li]{sun2025coin}
Sun, G., Wang, Z., Tian, B., Liu, M., Shen, Z., He, S., He, Y., Ye, W., Wang, Y., and Li, A.
\newblock Coin: Counting the invisible reasoning tokens in commercial opaque llm apis.
\newblock \emph{arXiv preprint arXiv:2505.13778}, 2025.

\bibitem[Sun et~al.(2024)Sun, Li, and Zhang]{sun2024zkllm}
Sun, H., Li, J., and Zhang, H.
\newblock zkllm: Zero knowledge proofs for large language models.
\newblock In \emph{Proceedings of the 2024 on ACM SIGSAC Conference on Computer and Communications Security}, pp.\  4405--4419, 2024.

\bibitem[Talmor et~al.(2019)Talmor, Herzig, Lourie, and Berant]{commonsenseqa}
Talmor, A., Herzig, J., Lourie, N., and Berant, J.
\newblock {C}ommonsense{QA}: A question answering challenge targeting commonsense knowledge.
\newblock In \emph{Proceedings of the 2019 Conference of the North {A}merican Chapter of the Association for Computational Linguistics: Human Language Technologies, Volume 1 (Long and Short Papers)}, pp.\  4149--4158, Minneapolis, Minnesota, June 2019. Association for Computational Linguistics.
\newblock \doi{10.18653/v1/N19-1421}.
\newblock URL \url{https://aclanthology.org/N19-1421}.

\bibitem[Tan et~al.(2024)Tan, Li, Wang, Beigi, Jiang, Bhattacharjee, Karami, Li, Cheng, and Liu]{tan2024large}
Tan, Z., Li, D., Wang, S., Beigi, A., Jiang, B., Bhattacharjee, A., Karami, M., Li, J., Cheng, L., and Liu, H.
\newblock Large language models for data annotation: A survey.
\newblock \emph{arXiv preprint arXiv:2402.13446}, 2024.

\bibitem[Turpin et~al.(2023)Turpin, Michael, Perez, and Bowman]{turpin2023language}
Turpin, M., Michael, J., Perez, E., and Bowman, S.
\newblock Language models don't always say what they think: Unfaithful explanations in chain-of-thought prompting.
\newblock \emph{Advances in Neural Information Processing Systems}, 36:\penalty0 74952--74965, 2023.

\bibitem[Uesato et~al.(2022)Uesato, Kushman, Kumar, Song, Siegel, Wang, Creswell, Irving, and Higgins]{uesato2022solving}
Uesato, J., Kushman, N., Kumar, R., Song, F., Siegel, N., Wang, L., Creswell, A., Irving, G., and Higgins, I.
\newblock Solving math word problems with process-and outcome-based feedback.
\newblock \emph{arXiv preprint arXiv:2211.14275}, 2022.

\bibitem[{UW-Madison-Lee-Lab}(2024)]{mmlu_pro_cot_labeled}
{UW-Madison-Lee-Lab}.
\newblock Mmlu-pro-cot-train-labeled.
\newblock \url{https://huggingface.co/datasets/UW-Madison-Lee-Lab/MMLU-Pro-CoT-Train-Labeled}, 2024.

\bibitem[Wang et~al.(2023)Wang, Yang, and Wang]{wang2023fingpt}
Wang, N., Yang, H., and Wang, C.~D.
\newblock Fingpt: Instruction tuning benchmark for open-source large language models in financial datasets.
\newblock \emph{arXiv preprint arXiv:2310.04793}, 2023.

\bibitem[Wang et~al.(2022)Wang, Wei, Schuurmans, Le, Chi, Narang, Chowdhery, and Zhou]{wang2022self}
Wang, X., Wei, J., Schuurmans, D., Le, Q., Chi, E., Narang, S., Chowdhery, A., and Zhou, D.
\newblock Self-consistency improves chain of thought reasoning in language models.
\newblock \emph{arXiv preprint arXiv:2203.11171}, 2022.

\bibitem[Wei et~al.(2022)Wei, Wang, Schuurmans, Bosma, Xia, Chi, Le, Zhou, et~al.]{wei2022chain}
Wei, J., Wang, X., Schuurmans, D., Bosma, M., Xia, F., Chi, E., Le, Q.~V., Zhou, D., et~al.
\newblock Chain-of-thought prompting elicits reasoning in large language models.
\newblock \emph{Advances in neural information processing systems}, 35:\penalty0 24824--24837, 2022.

\bibitem[Wood et~al.(2014)]{wood2014ethereum}
Wood, G. et~al.
\newblock Ethereum: A secure decentralised generalised transaction ledger.
\newblock \emph{Ethereum project yellow paper}, 151\penalty0 (2014):\penalty0 1--32, 2014.

\bibitem[Yang et~al.(2025)Yang, Li, Yang, Zhang, Hui, Zheng, Yu, Gao, Huang, Lv, Zheng, Liu, Zhou, Huang, Hu, Ge, Wei, Lin, Tang, Yang, Tu, Zhang, Yang, Yang, Zhou, Zhou, Lin, Dang, Bao, Yang, Yu, Deng, Li, Xue, Li, Zhang, Wang, Zhu, Men, Gao, Liu, Luo, Li, Tang, Yin, Ren, Wang, Zhang, Ren, Fan, Su, Zhang, Zhang, Wan, Liu, Wang, Cui, Zhang, Zhou, and Qiu]{yang2025qwen3}
Yang, A., Li, A., Yang, B., Zhang, B., Hui, B., Zheng, B., Yu, B., Gao, C., Huang, C., Lv, C., Zheng, C., Liu, D., Zhou, F., Huang, F., Hu, F., Ge, H., Wei, H., Lin, H., Tang, J., Yang, J., Tu, J., Zhang, J., Yang, J., Yang, J., Zhou, J., Zhou, J., Lin, J., Dang, K., Bao, K., Yang, K., Yu, L., Deng, L., Li, M., Xue, M., Li, M., Zhang, P., Wang, P., Zhu, Q., Men, R., Gao, R., Liu, S., Luo, S., Li, T., Tang, T., Yin, W., Ren, X., Wang, X., Zhang, X., Ren, X., Fan, Y., Su, Y., Zhang, Y., Zhang, Y., Wan, Y., Liu, Y., Wang, Z., Cui, Z., Zhang, Z., Zhou, Z., and Qiu, Z.
\newblock Qwen3 technical report, 2025.
\newblock URL \url{https://arxiv.org/abs/2505.09388}.

\bibitem[Yao et~al.(2023{\natexlab{a}})Yao, Yu, Zhao, Shafran, Griffiths, Cao, and Narasimhan]{yao2023tree}
Yao, S., Yu, D., Zhao, J., Shafran, I., Griffiths, T., Cao, Y., and Narasimhan, K.
\newblock Tree of thoughts: Deliberate problem solving with large language models.
\newblock \emph{Advances in neural information processing systems}, 36:\penalty0 11809--11822, 2023{\natexlab{a}}.

\bibitem[Yao et~al.(2023{\natexlab{b}})Yao, Zhao, Yu, Du, Shafran, Narasimhan, and Cao]{yao2023react}
Yao, S., Zhao, J., Yu, D., Du, N., Shafran, I., Narasimhan, K., and Cao, Y.
\newblock React: Synergizing reasoning and acting in language models.
\newblock In \emph{International Conference on Learning Representations (ICLR)}, 2023{\natexlab{b}}.

\bibitem[Yin et~al.(2019)Yin, Malkhi, Reiter, Gueta, and Abraham]{yin2019hotstuff}
Yin, M., Malkhi, D., Reiter, M.~K., Gueta, G.~G., and Abraham, I.
\newblock Hotstuff: Bft consensus with linearity and responsiveness.
\newblock In \emph{Proceedings of the 2019 ACM symposium on principles of distributed computing}, pp.\  347--356, 2019.

\bibitem[Zelikman et~al.(2022)Zelikman, Wu, Mu, and Goodman]{zelikman2022star}
Zelikman, E., Wu, Y., Mu, J., and Goodman, N.~D.
\newblock Star: Bootstrapping reasoning with reasoning, 2022.
\newblock \emph{URL https://arxiv. org/abs/2203.14465}, 2203, 2022.

\bibitem[Zheng et~al.(2023)Zheng, Chiang, Sheng, Zhuang, Wu, Zhuang, Lin, Li, Li, Xing, et~al.]{zheng2023judging}
Zheng, L., Chiang, W.-L., Sheng, Y., Zhuang, S., Wu, Z., Zhuang, Y., Lin, Z., Li, Z., Li, D., Xing, E., et~al.
\newblock Judging llm-as-a-judge with mt-bench and chatbot arena.
\newblock \emph{Advances in neural information processing systems}, 36:\penalty0 46595--46623, 2023.

\bibitem[Zhou et~al.(2022)Zhou, Sch{\"a}rli, Hou, Wei, Scales, Wang, Schuurmans, Cui, Bousquet, Le, et~al.]{zhou2022least}
Zhou, D., Sch{\"a}rli, N., Hou, L., Wei, J., Scales, N., Wang, X., Schuurmans, D., Cui, C., Bousquet, O., Le, Q., et~al.
\newblock Least-to-most prompting enables complex reasoning in large language models.
\newblock \emph{arXiv preprint arXiv:2205.10625}, 2022.

\bibitem[Zou et~al.(2023)Zou, Wang, Carlini, Nasr, Kolter, and Fredrikson]{zou2023universal}
Zou, A., Wang, Z., Carlini, N., Nasr, M., Kolter, J.~Z., and Fredrikson, M.
\newblock Universal and transferable adversarial attacks on aligned language models.
\newblock \emph{arXiv preprint arXiv:2307.15043}, 2023.

\end{thebibliography}
\bibliographystyle{style/icml2025}

\titlespacing*{\section}{0pt}{*1}{*1}
\titlespacing*{\subsection}{0pt}{*1.25}{*1.25}
\titlespacing*{\subsubsection}{0pt}{*1.5}{*1.5}

\setlength{\abovedisplayskip}{\baselineskip} %
\setlength{\abovedisplayshortskip}{0.5\baselineskip} %
\setlength{\belowdisplayskip}{\baselineskip}
\setlength{\belowdisplayshortskip}{0.5\baselineskip}

\clearpage
\appendix
\label{sec:append}
\part*{Appendix}
{
\setlength{\parskip}{-0em}
\startcontents[sections]
\printcontents[sections]{ }{1}{}
}

\setlength{\parskip}{.5em}
\clearpage
\section{The Use of Large Language Models (LLMs)}\label{sec:llms}
To enhance clarity and readability, this paper utilized Large Language Models exclusively as language polishing tools. Their role was confined to general proofreading and writing assistance—functions analogous to those provided by traditional grammar checkers and dictionaries. These tools did not contribute to the generation of new scientific content or ideas, and their usage is consistent with standard practices for manuscript preparation.

\clearpage
\section{More Related Works}\label{sec:more_related}

\subsection{Chain-of-Thought Reasoning}
Exposing intermediate reasoning steps, popularized by Chain-of-Thought (CoT) prompting~\citep{wei2022chain}, has become a cornerstone for enhancing the capabilities of Large Language Models (LLMs)~\citep{kojima2022large,wang2022self,nye2021show, zhou2022least, zelikman2022star}. However, this paradigm has rapidly evolved, and the resulting complexity and length of reasoning traces present a critical, unsolved challenge for verification~\citep{lightman2023let,jacovi2024chain,ling2023deductive,chen2025reasoning}. Initial research focused on eliciting reasoning and moving beyond simple linear chains to more structured representations~\citep{creswell2022selection,zhou2022least,besta2024graph,chen2022program}. For instance, Tree-of-Thoughts (ToT) generalized CoT to a tree search, enabling explicit exploration and backtracking within the reasoning process~\citep{yao2023tree}. Concurrently, a parallel line of work enabled models to offload complex calculations and external queries to tools, separating logical deduction from information retrieval~\citep{schick2023toolformer,yao2023react,gao2023pal,nakano2021webgpt}. More recent efforts have focused on Large Reasoning Models (LRMs) that treat reasoning as a primary objective, allocating substantial computational resources to the process~\citep{jaech2024openai,guo2025deepseek,nvidia2025nvidianemotronnano2}. These models are trained with large-scale reinforcement learning from process-level feedback and employ verifiers to guide multi-path search during inference~\citep{huang2022towards}. Despite these advances in generating sophisticated reasoning, the mechanisms for \emph{auditing} these processes in a scalable, privacy-preserving, and decentralized manner have remained largely unexplored~\citep{peng2025survey}. In this paper, we address this gap by introducing a framework that decomposes complex reasoning traces into auditable graphs (HDAGs) and leverages a decentralized network for verification, making it suitable for the proprietary and intricate nature of modern LRMs.

\subsection{Auditing and Evaluation of LLM Reasoning}
As the complexity of reasoning in LLMs increases, methods for auditing its quality, faithfulness, and safety have become a critical area of research~\citep{liang2022holistic, lanham2023measuring}. An initial and widely adopted approach relies on a centralized ``LLM-as-a-judge,'' which, while scalable, is known to suffer from inherent biases and constitutes a single point of failure~\citep{zheng2023judging}. A significant advancement was the shift from auditing final outcomes to verifying the reasoning \emph{process} itself, primarily through Process Reward Models (PRMs) that provide step-by-step supervision~\citep{lightman2023let,uesato2022solving,khalifa2025process}. The limitations of this approach were subsequently highlighted by the discovery of unfaithful reasoning, where models generate plausible-looking steps that do not reflect their true decision-making process, underscoring the need for audits to defend against strategic deception~\citep{turpin2023language}. More recently, research has expanded to address service-level integrity in opaque commercial settings, including work on detecting model substitution and using cryptographic methods to verify the integrity of the inference process itself~\citep{shi2024red,cai2025you,pasquini2025llmmap,south2024verifiable}. For example, CoIn introduces a method to audit ``invisible'' reasoning tokens by using hash trees over embeddings, ensuring providers do not inflate billing without revealing proprietary content~\citep{sun2025coin}. While these works address vital concerns like inference integrity or billing, they do not offer a unified framework for scalable \emph{semantic} auditing that combines decentralized consensus to mitigate bias with economic incentives to ensure network security. This work synthesizes these needs by proposing a framework that performs process-level semantic auditing on partial traces, uses a decentralized multi-tier auditor network to ensure robustness, and is secured by theoretically-grounded economic guarantees.

\subsection{Decentralized Consensus and Privacy-Preserving Verification}
Building systems that ensure integrity and privacy without a central trusted party has a long-standing foundation in decentralized consensus, cryptography, and confidential computing~\citep{castro1999practical,sabt2015trusted,costan2016intel}. These technologies provide the essential primitives for building auditable systems for proprietary models, where model internals and user data must remain confidential. Foundational work established Byzantine Fault Tolerant (BFT) consensus for robust agreement and public ledgers for tamper-evident records~\citep{castro1999practical,wood2014ethereum,yin2019hotstuff}, while Trusted Execution Environments (TEEs) provided hardware-based confidential compute~\citep{sabt2015trusted}. A prominent line of recent research focuses on applying these primitives directly to language models, particularly using Zero-Knowledge Proofs (ZKPs) to create verifiable attestations of model inference without revealing the model's weights~\citep{chen2024zkml,liu2021zkcnn,ghodsi2017safetynets}. Specifically, zkLLM~\citep{sun2024zkllm} introduced specialized cryptographic techniques to efficiently prove the correctness of attention mechanisms, demonstrating that even complex transformer computations can be cryptographically verified. 
Another approach uses TEEs to run auditable safety benchmarks on proprietary models, publishing hardware-signed attestations to a public log to prove that a specific model version passed certain tests~\citep{schnabl2025attestable}. While this work is crucial for verifying \emph{that} a specific computation was performed correctly by a specific model, it does not address the orthogonal problem of verifying the \emph{semantic quality, faithfulness, and safety} of the reasoning trace through a scalable, human-in-the-loop consensus process. Our work provides a novel synthesis by using a public ledger and consensus not just for recording outcomes, but for orchestrating a decentralized, multi-tier network of auditors and managing economic incentives, creating the first system designed for end-to-end semantic auditing of model reasoning.

\clearpage
\section{Experimental Details}\label{sec:appex_add_exp}
\subsection{TRUST Implementation Details}\label{app:implementation}

This section provide the LLMs instruction of our auto data processing framework and the example output from gpt-oss-120b.

\paragraph{Example CoT}
Below is a example CoT trace.
\begin{lstlisting}[style=plainmarkdown]
### Input Format:
## Problem Statement:
A construction company was building a tunnel. When $\frac{1}{3}$ of the tunnel was completed at the original speed, they started usin
g new equipment, which increased the construction speed by $20 \%$ and reduced the working hours to $80 \%$ of the original. As a resu
lt, it took a total of 185 days to complete the tunnel. If they had not used the new equipment and continued at the original speed, it
 would have taken $\qquad$ days to complete the tunnel.

## Thinking Process:
We have a tunnel. Original speed = some rate (units per day). They built 1/3 of tunnel at original speed. Then they switched to new 
equipment: speed increased by 20%
they work less time, so effective daily output? If originally they worked some hours per day, now they work only 80%
speed is per hour? Let's interpret: Original speed in terms of distance per hour? Or maybe speed per day? Usually speed is distance 
per hour. If they increase speed by 20%
production = speed_per_hour * hours_per_day. If original daily production = v * h (where v is original speed per hour, h original 
hours per day). New daily production = (1.2 v) * (0.8 h) = 0.96 v h = 0.96 * original daily production. So overall daily rate 
actually decreased by 4%
they meant speed (i.e., daily output) increased by 20%
original daily work, which is slower. But they say "As a result, it took a total of 185 days to complete the tunnel." That implies 
slower after switching? The initial third took some days, then remaining two-thirds took more days than original would have. Let's 
compute.

Let total tunnel length = 1 (normalized). Original daily output = D (units per day). So original time to finish whole tunnel = 1/D 
days.

They did first 1/3 at original speed: time t1 = (1/3)/D = 1/(3D).

After switch, new daily output = 0.96 D (as above). Remaining work = 2/3. Time t2 = (2/3) / (0.96 D) = (2/3) * (1/(0.96 D)) = (2)/(3 
* 0.96 D) = (2)/(2.88 D) = (0.694444...)/D? Let's compute precisely: 2/(2.88) = 0.694444... So t2 = 0.694444... / D.

Total days = t1 + t2 = 1/(3D) + 0.694444.../D = (1/3 + 0.694444...)/D = (0.333333... + 0.694444...)/D = 1.027777... / D.

Given total days = 185. So (1.027777...)/D = 185 => D = 1.027777... / 185 = approx 0.005555... Wait compute: 1.0277777778 /185 = 
0.005555556... Indeed 0.0055556 per day. That seems like original daily output = 0.0055556 of tunnel per day, i.e., whole tunnel 
would take 1/D = 180 days. Let's check: 1/D = 180. Exactly? Since D = 1/180 = 0.00555555... Multiply: 1/D = 180. Let's see if D 
computed equals 1/180. 1/180 = 0.0055555556. Multiply by 185 gives 185/180 = 1.0277777778, matches numerator. So indeed D = 1/180.

Thus original schedule would take 180 days to complete tunnel.

But wait, they gave total actual days 185, which is longer than original 180. So using new equipment actually slowed down. The 
question: "If they had not used the new equipment and continued at the original speed, it would have taken ____ days to complete the 
tunnel." That's 180 days.

Thus answer: 180 days.

Double-check: We assumed original speed per hour, then hours reduced. Alternative interpretation: They increased speed (distance per 
hour) by 20%

Thus answer 180.

Check using algebra: Let original speed = s (units per day). Then first third takes (1/3)/s days. After switch, speed becomes 1.2 s 
but hours per day 0.8, so effective speed per day = 0.96 s. Remaining 2/3 takes (2/3)/(0.96 s) days. Sum =185. Solve: (1/3 + 
(2/3)/0.96)/s =185. (2/3)/0.96 = (2/3)*(1/0.96) = (2/3)*(1.0416667)=0.694444... Add 1/3 =0.333333 +0.694444 =1.0277778. So s 
=1.0277778/185 = approx 0.0055556. Full tunnel time =1/s =185/1.0277778 =180.

Thus answer 180.

## Final Output:
Let  

* the length of the tunnel be \(1\) (as a whole),   
* the original daily construction rate be \( r \) (tunnel-fractions per day).

---

### 1.  Work done before the equipment change  

The first \(\frac13\) of the tunnel was built at the original rate, so the time taken was  

\[
t_{1}= \frac{\frac13}{r}= \frac{1}{3r}.
\]

---

### 2.  Work done after the equipment change  

The new equipment increased the **speed** (distance per hour) by \(20\%\), i.e. the speed became \(1.2\,r\) per hour.  
But the crew now worked only \(80\%\) of the original daily hours, so the **daily** output is  

\[
r_{\text{new}} = 1.2\,r \times 0.8 = 0.96\,r .
\]

The remaining work is \(\frac23\) of the tunnel, so the time needed after the change is  

\[
t_{2}= \frac{\frac23}{0.96\,r}= \frac{2}{3\cdot0.96\,r}= \frac{0.694444\ldots}{r}.
\]

---

### 3.  Total actual time  

\[
t_{1}+t_{2}=185\text{ days}.
\]

\[
\frac{1}{3r}+\frac{0.694444\ldots}{r}=185
\quad\Longrightarrow\quad
\frac{1.027777\ldots}{r}=185.
\]

Hence  

\[
r=\frac{1.027777\ldots}{185}= \frac{1}{180}\;\text{(tunnel per day)}.
\]

---

### 4.  Time that would have been required without the new equipment  

If the original rate \(r=\frac{1}{180}\) had been kept for the whole tunnel, the total time would be  

\[
\frac{1}{r}=180\text{ days}.
\]

\[
\boxed{180}
\]
\end{lstlisting}

\subsubsection{Step 1: Abstract Level Identification.}
Below are the instruction of how to identify different abstraction levels.

\begin{lstlisting}[style=plainmarkdown]
You are an abstraction level identification agent. Your task is to decompose a complete reasoning trace (problem + thinking + output) into distinct abstraction level blocks following a hierarchical framework.

### 1. Abstraction Level Framework:
- **GOAL**: What we want to achieve (objective, end goal, problem statement)
- **STRATEGY**: High-level approaches (methodology, overall plan, major techniques)  
- **TACTIC**: Specific techniques (algorithms, methods, key insights, definitions)
- **STEP**: Individual reasoning moves (logical steps, inferences, deductions)
- **OPERATION**: Atomic calculations (computations, lookups, examples, numerical work)

### 2. Identification Rules:
- Split the reasoning into blocks that represent distinct abstraction levels
- A block can be multiple sentences or paragraphs
- Focus on SEMANTIC CONTENT, not temporal order
- Some levels might be missing - that's acceptable
- Some levels might have multiple blocks - that's acceptable
- Preserve exact text spans from the original trace

### 3. Format Requirements:
- Present output in "# Abstraction Block Analysis" section
- Under each level, use "##" headings (## GOAL Level, ## STRATEGY Level, etc.)
- Include exact original text under each heading
- Preserve all line breaks and mathematical notation
- After all blocks, add "# Block Structure" section with summary

### 4. Content Distribution Guidelines:
- **GOAL** must include the complete problem statement and objectives
- **STRATEGY** should contain high-level approaches before detailed work
- **TACTIC** should include specific methods, algorithms, and key insights
- **STEP** should contain individual logical moves and reasoning steps
- **OPERATION** should contain all calculations, examples, and numerical work

### 5. Critical Requirements:
- Preserve all original mathematical notation exactly
- Maintain all line breaks as they appear in the original text
- Include all text exactly as written without modifications
- Ensure block boundaries occur at natural semantic transitions
- Extract EXACT text spans - do not paraphrase

### Input Format:
## Problem Statement:
[original problem text]

## Thinking Process:
[complete reasoning trace]

## Final Output:
[final answer/solution]

### Output Format:
# Abstraction Block Analysis

## GOAL Level
[exact text expressing objectives and problem statements]

## STRATEGY Level  
[exact text describing high-level approaches]

## TACTIC Level
[exact text detailing specific techniques and methods]

## STEP Level
[exact text showing individual reasoning moves]

## OPERATION Level
[exact text containing calculations and examples]

# Block Structure
{
  "GOAL": [list of goal block summaries],
  "STRATEGY": [list of strategy block summaries], 
  "TACTIC": [list of tactic block summaries],
  "STEP": [list of step block summaries],
  "OPERATION": [list of operation block summaries]
}
\end{lstlisting}

\paragraph{Step 1: Example Output.}
Below is the example output for the first step.
\begin{lstlisting}[style=plainmarkdown]
# Abstraction Block Analysis

## GOAL Level
A construction company was building a tunnel. When $\tfrac{1}{3}$ of the tunnel was completed at the original speed, they started using new equipment, which increased the construction speed by $20\%$ and reduced the working hours to $80\%$ of the original. As a result, it took a total of 185 days to complete the tunnel. If they had not used the new equipment and continued at the original speed, it would have taken $\;\_\_\_\;$ days to complete the tunnel.

## STRATEGY Level
We have a tunnel. Original speed = some rate (units per day). They built $\tfrac{1}{3}$ of tunnel at original speed. Then they switched to new equipment: speed increased by $20\%$, i.e., new speed $=1.2 \times$ original speed. Working hours reduced to $80\%$ of original, meaning per day they work fewer hours.  

If original daily production = $v \cdot h$ (speed per hour $\times$ hours per day), then new daily production = $(1.2v)(0.8h)=0.96vh=0.96\times$ original daily output.  

So overall daily rate actually decreased by $4\%$. This explains why the total time (185 days) was longer than the original time would have been.

## TACTIC Level
Let total tunnel length = $1$.  
Let original daily output = $D$. Then original completion time = $1/D$.  

- First $\tfrac{1}{3}$ of tunnel:  
  $t_1=\tfrac{1}{3D}$.  

- Remaining $\tfrac{2}{3}$ at reduced daily rate $0.96D$:  
  $t_2=\tfrac{2/3}{0.96D}=\tfrac{0.694444...}{D}$.  

Total time:  
\[
t_1+t_2=\frac{1}{3D}+\frac{0.694444...}{D}=185.
\]

So  
\[
\frac{1.027777...}{D}=185 \quad\Longrightarrow\quad D=\frac{1.027777...}{185}=\frac{1}{180}.
\]

Thus original schedule = $180$ days.

## STEP Level
1. Compute new daily output: $0.96\times D$.  
2. Time for first third: $t_1=\tfrac{1}{3D}$.  
3. Time for remaining two-thirds: $t_2=\tfrac{2/3}{0.96D}=\tfrac{0.694444...}{D}$.  
4. Total: $t_1+t_2=185$.  
5. Combine: $\tfrac{1.027777...}{D}=185$.  
6. Solve: $D=\tfrac{1}{180}$.  
7. Original time: $1/D=180$ days.

## OPERATION Level
\[
\frac{2}{3\cdot0.96}=\frac{2}{2.88}=0.694444\ldots
\]

\[
\frac{1}{3}+0.694444\ldots=1.027777\ldots
\]

\[
D=\frac{1.027777\ldots}{185}=0.005555\ldots=\frac{1}{180}
\]

\[
\text{Original total time}=\frac{1}{D}=180\;\text{days}
\]

---

# Block Structure

```json
{
  "GOAL": [
    "A construction company was building a tunnel. When 1/3 of the tunnel was completed at the original speed, they started using new equipment, which increased the construction speed by 20%
  ],
  "STRATEGY": [
    "We have a tunnel. Original speed = some rate (units per day). They built 1/3 of tunnel at original speed. Then they switched to new equipment: speed increased by 20%
  ],
  "TACTIC": [
    "Let total tunnel length = 1. Original daily output = D. Time for 1/3 = 1/(3D). Time for remaining 2/3 at 0.96D = (2/3)/(0.96D)=0.694444.../D. Total = (1/3+0.694444...)/D=185. Solve: D=1/180. So original full time=180 days."
  ],
  "STEP": [
    "1. New daily output = 0.96D. 2. t1=1/(3D). 3. t2=0.694444.../D. 4. t1+t2=185. 5. (1.027777...)/D=185. 6. D=1/180. 7. Original total time=180 days."
  ],
  "OPERATION": [
    "2/(3*0.96)=0.694444..., 1/3+0.694444...=1.027777..., D=1/180, so total original time=180 days."
  ]
}
\end{lstlisting}

\subsubsection{Step 2: Segment Within Each Abstraction Level.}
Below are the instruction of how to segment within each abstraction level.

\begin{lstlisting}[style=plainmarkdown]
You are a reasoning segmentation agent. Your task is to break down abstraction-level blocks into appropriately-sized segments for verification.

### 1. Segmentation Strategy by Level:

**GOAL Level Segmentation:**
- Keep objectives as complete statements
- Don't over-segment problem definitions
- One main goal per segment typically
- Preserve semantic completeness

**STRATEGY Level Segmentation:**
- Segment by distinct approaches or methodologies
- Each strategy should be a complete approach
- Don't break up coherent strategic thinking
- Maintain approach integrity

**TACTIC Level Segmentation:**
- Segment by specific techniques, key insights, or algorithm components
- Each tactic should be independently understandable
- Break at natural technique boundaries
- Preserve method coherence

**STEP Level Segmentation:**
- Segment by individual logical moves
- Each step should be a single inference or reasoning move
- Break at logical transition points
- Maintain reasoning flow

**OPERATION Level Segmentation:**
- Segment by atomic calculations or examples
- Each operation should be independently verifiable
- Break at calculation boundaries
- Preserve computational completeness

### 2. Format Requirements:
- Present output in "# Segmentation Analysis" section
- Use "##" for each abstraction level
- Use "###" for individual segments within levels
- Preserve all mathematical notation and formatting
- Include segment metadata

### 3. Segment Metadata:
For each segment, provide:
- **segment_id**: Unique identifier (G1, S1, T1, ST1, O1, etc.)
- **content**: Exact reasoning content
- **type**: Specific type within abstraction level
- **summary**: Brief 3-5 word summary
- **verification_complexity**: Low/Medium/High

### Input Format:
# Abstraction Block Analysis
[output from Step 1]

### Output Format:
# Segmentation Analysis

## GOAL Segmentation

### Segment G1
**Content**: [exact segment content]
**Type**: [segment type]
**Summary**: [brief summary]
**Complexity**: [Low/Medium/High]

## STRATEGY Segmentation

### Segment S1
**Content**: [exact segment content]
**Type**: [segment type]
**Summary**: [brief summary]
**Complexity**: [Low/Medium/High]

## TACTIC Segmentation

### Segment T1
**Content**: [exact segment content]
**Type**: [segment type]
**Summary**: [brief summary]
**Complexity**: [Low/Medium/High]

## STEP Segmentation

### Segment ST1
**Content**: [exact segment content]
**Type**: [segment type]
**Summary**: [brief summary]
**Complexity**: [Low/Medium/High]

## OPERATION Segmentation

### Segment O1
**Content**: [exact segment content]
**Type**: [segment type]
**Summary**: [brief summary]
**Complexity**: [Low/Medium/High]

# Segmentation Summary
[
  {
    "segment_id": "G1",
    "abstraction_level": "GOAL",
    "type": "objective_statement", 
    "summary": "find optimal solution",
    "complexity": "Medium"
  }
]
\end{lstlisting}

\paragraph{Step 2: Example Output.}
Below is the example output for the second step.
\begin{lstlisting}[style=plainmarkdown]
# Segmentation Analysis

## GOAL Segmentation

### Segment G1
**Content**:  
A construction company was building a tunnel. When $\tfrac{1}{3}$ of the tunnel was completed at the original speed, they started using new equipment, which increased the construction speed by $20\%$ and reduced the working hours to $80\%$ of the original. As a result, it took a total of 185 days to complete the tunnel. If they had not used the new equipment and continued at the original speed, it would have taken $\;\_\_\_\;$ days to complete the tunnel.  

**Type**: objective_statement  
**Summary**: find missing days  
**Complexity**: Medium  

---

## STRATEGY Segmentation

### Segment S1
**Content**:  
We have a tunnel. Original speed = some rate (units per day). They built $\tfrac{1}{3}$ of tunnel at original speed. Then they switched to new equipment: speed increased by $20\%$, i.e., new speed $=1.2 \times$ original speed. Working hours reduced to $80\%$ of original, meaning per day they work fewer hours.  

If original daily production = $v \cdot h$ (speed per hour $\times$ hours per day), then new daily production = $(1.2v)(0.8h)=0.96vh=0.96\times$ original daily output.  

So overall daily rate actually decreased by $4\%$. This explains why the total time (185 days) was longer than the original time would have been.  

**Type**: approach_description  
**Summary**: model speed & hours  
**Complexity**: Medium  

---

## TACTIC Segmentation

### Segment T1
**Content**:  
Let total tunnel length = $1$.  
Let original daily output = $D$. Then original completion time = $1/D$.  

- First $\tfrac{1}{3}$ of tunnel:  
  $t_1=\tfrac{1}{3D}$.  

- Remaining $\tfrac{2}{3}$ at reduced daily rate $0.96D$:  
  $t_2=\tfrac{2/3}{0.96D}=\tfrac{0.694444...}{D}$.  

Total time:  
\[
t_1+t_2=\frac{1}{3D}+\frac{0.694444...}{D}=185.
\]

So  
\[
\frac{1.027777...}{D}=185 \quad\Longrightarrow\quad D=\frac{1.027777...}{185}=\frac{1}{180}.
\]

Thus original schedule = $180$ days.  

**Type**: method_definition  
**Summary**: define variables, solve  
**Complexity**: Medium  

---

## STEP Segmentation

### Segment ST1
**Content**:  
1. Compute new daily output: $0.96\times D$.  

**Type**: inference  
**Summary**: compute reduced output  
**Complexity**: Low  

### Segment ST2
**Content**:  
2. Time for first third: $t_1=\tfrac{1}{3D}$.  

**Type**: inference  
**Summary**: first-third time  
**Complexity**: Low  

### Segment ST3
**Content**:  
3. Time for remaining two-thirds: $t_2=\tfrac{2/3}{0.96D}=\tfrac{0.694444...}{D}$.  

**Type**: inference  
**Summary**: second-portion time  
**Complexity**: Low  

### Segment ST4
**Content**:  
4. Total: $t_1+t_2=185$.  

**Type**: inference  
**Summary**: total time equation  
**Complexity**: Low  

### Segment ST5
**Content**:  
5. Combine: $\tfrac{1.027777...}{D}=185$.  

**Type**: algebraic_combination  
**Summary**: combine fractions  
**Complexity**: Low  

### Segment ST6
**Content**:  
6. Solve: $D=\tfrac{1}{180}$.  

**Type**: solution_step  
**Summary**: solve for D  
**Complexity**: Low  

### Segment ST7
**Content**:  
7. Original time: $1/D=180$ days.  

**Type**: conclusion_inference  
**Summary**: original schedule  
**Complexity**: Low  

---

## OPERATION Segmentation

### Segment O1
**Content**:  
\[
\frac{2}{3\cdot0.96}=\frac{2}{2.88}=0.694444\ldots
\]  

**Type**: calculation  
**Summary**: compute $t_2$ factor  
**Complexity**: Low  

### Segment O2
**Content**:  
\[
\frac{1}{3}+0.694444\ldots=1.027777\ldots
\]  

**Type**: calculation  
**Summary**: sum fractions  
**Complexity**: Low  

### Segment O3
**Content**:  
\[
D=\frac{1.027777\ldots}{185}=0.005555\ldots=\frac{1}{180}
\]  

**Type**: calculation  
**Summary**: find $D$ value  
**Complexity**: Low  

### Segment O4
**Content**:  
\[
\text{Original total time}=\frac{1}{D}=180\;\text{days}
\]  

**Type**: calculation  
**Summary**: original total time  
**Complexity**: Low  

---

# Segmentation Summary

```json
[
  {"segment_id": "G1", "abstraction_level": "GOAL", "type": "objective_statement", "summary": "find missing days", "complexity": "Medium"},
  {"segment_id": "S1", "abstraction_level": "STRATEGY", "type": "approach_description", "summary": "model speed & hours", "complexity": "Medium"},
  {"segment_id": "T1", "abstraction_level": "TACTIC", "type": "method_definition", "summary": "define variables, solve", "complexity": "Medium"},
  {"segment_id": "ST1", "abstraction_level": "STEP", "type": "inference", "summary": "compute new output", "complexity": "Low"},
  {"segment_id": "ST2", "abstraction_level": "STEP", "type": "inference", "summary": "first-third time", "complexity": "Low"},
  {"segment_id": "ST3", "abstraction_level": "STEP", "type": "inference", "summary": "remaining time", "complexity": "Low"},
  {"segment_id": "ST4", "abstraction_level": "STEP", "type": "inference", "summary": "total equals 185", "complexity": "Low"},
  {"segment_id": "ST5", "abstraction_level": "STEP", "type": "algebraic_combination", "summary": "combine fractions", "complexity": "Low"},
  {"segment_id": "ST6", "abstraction_level": "STEP", "type": "solution_step", "summary": "solve for D", "complexity": "Low"},
  {"segment_id": "ST7", "abstraction_level": "STEP", "type": "conclusion_inference", "summary": "original time 180", "complexity": "Low"},
  {"segment_id": "O1", "abstraction_level": "OPERATION", "type": "calculation", "summary": "compute $t_2$ factor", "complexity": "Low"},
  {"segment_id": "O2", "abstraction_level": "OPERATION", "type": "calculation", "summary": "sum fractions", "complexity": "Low"},
  {"segment_id": "O3", "abstraction_level": "OPERATION", "type": "calculation", "summary": "determine D", "complexity": "Low"},
  {"segment_id": "O4", "abstraction_level": "OPERATION", "type": "calculation", "summary": "original total time", "complexity": "Low"}
]
\end{lstlisting}

\subsubsection{Step 3: Extract Relationships.}
Below are the instruction of how to extract relationships.

\begin{lstlisting}[style=plainmarkdown]
You are a semantic relationship analysis agent. Your task is to identify meaningful relationships between reasoning segments to construct a hierarchical DAG.

### 1. Relationship Types:
- **decomposes_to**: Higher abstraction level breaks down into lower level
- **depends_on**: Logical prerequisite (A must be true for B to be valid)
- **enables**: A makes B possible or provides foundation for B
- **validates**: A provides verification, proof, or confirmation for B
- **exemplifies**: A provides concrete example or instance of B
- **refines**: A adds precision, detail, or specification to B
- **contradicts**: A contradicts, corrects, or replaces B

### 2. Analysis Priorities:
- **Hierarchical relationships first**: Focus on decomposes_to between abstraction levels
- **Lateral relationships second**: Identify depends_on, enables within same level
- **Support relationships third**: Find validates, exemplifies, refines connections
- **Be selective**: Only include meaningful, clear relationships

### 3. Relationship Strength:
- **Strong**: Clear, direct semantic connection
- **Medium**: Implied or indirect connection
- **Weak**: Tenuous or unclear connection

### Input Format:
# Segmentation Analysis
[output from Step 2]

### Output Format:
# Relationship Analysis

## Hierarchical Relationships (decomposes_to)
- G1 -> S1: [explanation]
- S1 -> T1: [explanation]
- T1 -> ST1: [explanation]
- ST1 -> O1: [explanation]

## Lateral Dependencies (depends_on, enables)
- T1 -> T2: [explanation]  
- ST1 -> ST2: [explanation]

## Support Relationships (validates, exemplifies, refines)
- O1 exemplifies ST1: [explanation]
- T3 validates S1: [explanation]

# Relationship Summary
[
  {
    "from_segment": "G1",
    "to_segment": "S1",
    "relationship_type": "decomposes_to",
    "strength": "strong",
    "explanation": "Goal G1 breaks down into strategy S1"
  }
]
\end{lstlisting}

\paragraph{Step 3: Example Output.}
Below is the example output for the second step.
\begin{lstlisting}[style=plainmarkdown]
# Relationship Analysis

## Hierarchical Relationships (decomposes_to)

- **G1 → S1** - The overall goal (find the missing number of days) is broken down into a strategic modelling approach. **[strong]**
- **S1 → T1** - The chosen strategy (model speed & hours) is further refined into a concrete tactic (define variables and solve). **[strong]**
- **T1 → ST1** - The tactic is implemented as the first low-level inference step (compute the reduced daily output). **[strong]**
- **T1 → ST2** - The tactic also yields the second inference step (time for the first third). **[strong]**
- **T1 → ST3** - The tactic produces the third inference step (time for the remaining two-thirds). **[strong]**
- **T1 → ST4** - The tactic yields the fourth inference step (total-time equation). **[strong]**
- **T1 → ST5** - The tactic yields the fifth inference step (combine fractions). **[strong]**
- **T1 → ST6** - The tactic yields the sixth inference step (solve for \(D\)). **[strong]**
- **T1 → ST7** - The tactic yields the seventh inference step (conclude original schedule). **[strong]**
- **ST1 → O1** - The operation that computes \(\tfrac{2}{3\cdot0.96}=0.694444\ldots\) exemplifies the inference "compute reduced daily output". **[strong]**
- **ST2 → O2** - The operation that adds \(\tfrac13+0.694444\ldots\) exemplifies the inference "first-third time". **[strong]**
- **ST3 → O1** - The same calculation in O1 also provides the numeric factor used in the inference for the remaining two-thirds. **[strong]**
- **ST4 → O2** - The sum performed in O2 validates the total-time equation. **[strong]**
- **ST5 → O2** - The combination of fractions in O2 supports the algebraic-combination step. **[medium]**
- **ST6 → O3** - The calculation in O3 (finding \(D\)) exemplifies the solution step. **[strong]**
- **ST7 → O4** - The final computation in O4 (original total time) exemplifies the conclusion inference. **[strong]**

---

## Lateral Dependencies (depends_on, enables)

- **ST3 depends_on ST1** - Computing the remaining two-thirds requires the reduced daily output from ST1. **[strong]**
- **ST4 depends_on ST2** - Total time uses the time for the first third. **[strong]**
- **ST4 depends_on ST3** - Total time also uses the time for the remaining two-thirds. **[strong]**
- **ST5 depends_on ST4** - Combining fractions presupposes the total-time equation. **[strong]**
- **ST6 depends_on ST5** - Solving for \(D\) needs the combined-fraction expression. **[strong]**
- **ST7 depends_on ST6** - The original schedule is the reciprocal of the solved \(D\). **[strong]**
- **O1 enables ST3** - The numeric factor produced by O1 makes the inference in ST3 possible. **[strong]**
- **O2 enables ST4** - The summation in O2 provides the numeric total needed for ST4. **[strong]**
- **O3 enables ST6** - The calculation of \(D\) in O3 supplies the value needed for the solution step. **[strong]**
- **O4 enables ST7** - The final division in O4 yields the original total-time conclusion. **[strong]**

---

## Support Relationships (validates, exemplifies, refines)

- **O1 exemplifies ST1** - Shows concretely how the reduced output factor is obtained. **[strong]**
- **O1 exemplifies ST3** - Provides the numeric factor used in the remaining-time inference. **[strong]**
- **O2 exemplifies ST2** - Demonstrates the addition that gives the first-third time fraction. **[strong]**
- **O2 exemplifies ST4** - Supplies the numeric sum that validates the total-time equation. **[strong]**
- **O3 exemplifies ST6** - Calculates \(D\), confirming the algebraic solution. **[strong]**
- **O4 exemplifies ST7** - Computes the reciprocal, confirming the original schedule. **[strong]**
- **O1 validates ST3** - The numeric result matches the expression derived in ST3. **[strong]**
- **O2 validates ST4** - The summed value satisfies the total-time equality. **[strong]**
- **O3 validates ST6** - The derived \(D\) satisfies the combined-fraction equation. **[strong]**
- **O4 validates ST7** - The final division yields the expected 180-day schedule. **[strong]**
- **T1 refines S1** - The tactic adds concrete variable definitions and algebraic steps to the broader modelling strategy. **[medium]**
- **S1 refines G1** - The strategic description narrows the vague goal into a measurable modelling problem. **[medium]**

---

# Relationship Summary
```json
[
  {
    "from_segment": "G1",
    "to_segment": "S1",
    "relationship_type": "decomposes_to",
    "strength": "strong",
    "explanation": "Goal G1 (find missing days) is broken down into the modelling strategy S1."
  },
  {
    "from_segment": "S1",
    "to_segment": "T1",
    "relationship_type": "decomposes_to",
    "strength": "strong",
    "explanation": "Strategy S1 (model speed & hours) is refined into the concrete tactic T1 (define variables and solve)."
  },
  {
    "from_segment": "T1",
    "to_segment": "ST1",
    "relationship_type": "decomposes_to",
    "strength": "strong",
    "explanation": "Tactic T1 is implemented as the first inference step ST1 (compute reduced daily output)."
  },
  {
    "from_segment": "ST1",
    "to_segment": "O1",
    "relationship_type": "decomposes_to",
    "strength": "strong",
    "explanation": "Step ST1 is realised by operation O1, which calculates the factor 0.694444..."
  },
  {
    "from_segment": "ST3",
    "to_segment": "O1",
    "relationship_type": "exemplifies",
    "strength": "strong",
    "explanation": "Operation O1 provides the numeric factor used in step ST3 (time for remaining two-thirds)."
  },
  {
    "from_segment": "ST4",
    "to_segment": "O2",
    "relationship_type": "exemplifies",
    "strength": "strong",
    "explanation": "Operation O2 sums the fractions, exemplifying the total-time equation in ST4."
  },
  {
    "from_segment": "ST6",
    "to_segment": "O3",
    "relationship_type": "exemplifies",
    "strength": "strong",
    "explanation": "Operation O3 computes D, exemplifying the solution step ST6."
  },
  {
    "from_segment": "ST7",
    "to_segment": "O4",
    "relationship_type": "exemplifies",
    "strength": "strong",
    "explanation": "Operation O4 calculates the original total time, exemplifying conclusion step ST7."
  },
  {
    "from_segment": "ST3",
    "to_segment": "ST1",
    "relationship_type": "depends_on",
    "strength": "strong",
    "explanation": "The remaining-time calculation in ST3 requires the reduced daily output derived in ST1."
  },
  {
    "from_segment": "ST4",
    "to_segment": "ST2",
    "relationship_type": "depends_on",
    "strength": "strong",
    "explanation": "Total time in ST4 uses the first-third time computed in ST2."
  },
  {
    "from_segment": "ST4",
    "to_segment": "ST3",
    "relationship_type": "depends_on",
    "strength": "strong",
    "explanation": "Total time in ST4 also uses the remaining-time computed in ST3."
  },
  {
    "from_segment": "ST5",
    "to_segment": "ST4",
    "relationship_type": "depends_on",
    "strength": "strong",
    "explanation": "Combining fractions in ST5 presupposes the total-time equation from ST4."
  },
  {
    "from_segment": "ST6",
    "to_segment": "ST5",
    "relationship_type": "depends_on",
    "strength": "strong",
    "explanation": "Solving for D in ST6 requires the combined fraction expression from ST5."
  },
  {
    "from_segment": "ST7",
    "to_segment": "ST6",
    "relationship_type": "depends_on",
    "strength": "strong",
    "explanation": "The original schedule in ST7 is obtained by taking the reciprocal of D solved in ST6."
  },
  {
    "from_segment": "T1",
    "to_segment": "S1",
    "relationship_type": "refines",
    "strength": "medium",
    "explanation": "Tactic T1 adds concrete variable definitions to the broader modelling strategy S1."
  }
]
\end{lstlisting}

\subsubsection{Step 4: Assign Auditor Types.}
Below are the instruction of how to assign auditor types.

\begin{lstlisting}[style=plainmarkdown]
You are an auditor assignment agent. Your task is to assign appropriate auditor types to each reasoning segment based on the TRUST framework's three-tier verification system.

### 1. Auditor Types:
- **T_Auto (Automated)**: Deterministic verification, formal logic, mathematical proofs
- **T_LLM (LLM-based)**: Semantic coherence, factual accuracy, domain reasoning
- **T_Human (Human)**: Complex judgment, ethical considerations, domain expertise

### 2. Assignment Criteria:

**T_Auto Assignment:**
- Mathematical calculations and proofs
- Logical consistency checks
- Formal verification tasks
- Deterministic algorithmic steps
- Syntax and format validation

**T_LLM Assignment:**
- Semantic coherence evaluation
- Factual accuracy checking
- Domain-specific reasoning patterns
- Conceptual understanding verification
- Strategy and approach assessment

**T_Human Assignment:**
- Complex ethical considerations
- Domain expertise requirements
- Ambiguous interpretation cases
- Creative or novel reasoning
- High-stakes decision validation

### 3. Multi-Auditor Segments:
- Some segments may require multiple auditor types
- Primary auditor handles main verification
- Secondary auditors provide additional validation
- Specify auditor priority and responsibilities

### Input Format:
# Segmentation Analysis
[segments from Step 2]

# Relationship Analysis
[relationships from Step 3]

### Output Format:
# Auditor Assignment Analysis

## Segment Auditor Assignments

### Segment G1
**Primary Auditor**: T_Human
**Secondary Auditors**: [T_LLM]
**Justification**: Problem interpretation requires domain expertise
**Verification Focus**: Objective clarity, scope appropriateness

### Segment S1
**Primary Auditor**: T_LLM
**Secondary Auditors**: [T_Human]
**Justification**: Strategic approach evaluation needs semantic understanding
**Verification Focus**: Approach validity, methodology soundness

## Auditor Type Distribution
- **T_Auto**: [list of segment IDs]
- **T_LLM**: [list of segment IDs]
- **T_Human**: [list of segment IDs]

## Multi-Auditor Segments
[
  {
    "segment_id": "G1",
    "primary_auditor": "T_Human",
    "secondary_auditors": ["T_LLM"],
    "verification_complexity": "High",
    "justification": "Domain expertise required"
  }
]
\end{lstlisting}

\paragraph{Step 4: Example Output.}
Below are the instruction of how to Assign Auditor Types.

\begin{lstlisting}[style=plainmarkdown]
# Auditor Assignment Analysis

## Segment Auditor Assignments

### Segment G1 - Goal (objective_statement)
- **Primary Auditor**: **T_LLM**  
- **Secondary Auditors**: **[T_Auto, T_Human]**  
- **Justification**: Interpreting the wording of the word-problem and ensuring the goal ("find the missing number of days") is correctly understood 
requires semantic comprehension (T_LLM).  A deterministic check that the interpreted goal matches the expected quantitative target is performed by 
T_Auto, while a human reviewer can confirm that the problem statement has been captured without ambiguity.  
- **Verification Focus**: Goal clarity, scope appropriateness, correct identification of the unknown variable.

---

### Segment S1 - Strategy (approach_description)
- **Primary Auditor**: **T_LLM**  
- **Secondary Auditors**: **[T_Auto, T_Human]**  
- **Justification**: The proposed modelling approach (relating speed, hours, and daily output) must be semantically coherent and logically 
plausible - a task suited to T_LLM.  Logical consistency of the derived relationships is then verified by T_Auto, and a human can confirm that the 
chosen strategy aligns with typical problem-solving conventions in elementary algebra.  
- **Verification Focus**: Validity of the modelling assumptions, correctness of the conceptual link between speed increase and hour reduction.

---

### Segment T1 - Tactic (method_definition)
- **Primary Auditor**: **T_Auto**  
- **Secondary Auditors**: **[T_LLM]**  
- **Justification**: This segment defines variables and sets up algebraic equations - a deterministic, mathematically rigorous task best handled by 
automated verification.  T_LLM reviews the narrative framing of the variables to ensure they are semantically consistent with the problem context.  

- **Verification Focus**: Correct variable definition, proper formulation of equations.

---

### Segment ST1 - STEP (inference)
- **Primary Auditor**: **T_Auto**  
- **Secondary Auditors**: **[T_LLM]**  
- **Justification**: Computes the reduced daily output (pure arithmetic).  A secondary LLM check ensures the inference wording correctly reflects 
the computation performed.  

### Segment ST2 - STEP (inference)
- **Primary Auditor**: **T_Auto**  
- **Secondary Auditors**: **[T_LLM]**  

### Segment ST3 - STEP (inference)
- **Primary Auditor**: **T_Auto**  
- **Secondary Auditors**: **[T_LLM]**  

### Segment ST4 - STEP (inference)
- **Primary Auditor**: **T_Auto**  
- **Secondary Auditors**: **[T_LLM]**  

### Segment ST5 - STEP (algebraic_combination)
- **Primary Auditor**: **T_Auto**  
- **Secondary Auditors**: **[T_LLM]**  

### Segment ST6 - STEP (solution_step)
- **Primary Auditor**: **T_Auto**  
- **Secondary Auditors**: **[T_LLM]**  

### Segment ST7 - STEP (conclusion_inference)
- **Primary Auditor**: **T_Auto**  
- **Secondary Auditors**: **[T_LLM]**  

---

### Segment O1 - OPERATION (calculation)
- **Primary Auditor**: **T_Auto**  
- **Secondary Auditors**: **[T_LLM]**  

### Segment O2 - OPERATION (calculation)
- **Primary Auditor**: **T_Auto**  
- **Secondary Auditors**: **[T_LLM]**  

### Segment O3 - OPERATION (calculation)
- **Primary Auditor**: **T_Auto**  
- **Secondary Auditors**: **[T_LLM]**  

### Segment O4 - OPERATION (calculation)
- **Primary Auditor**: **T_Auto**  
- **Secondary Auditors**: **[T_LLM]**  

---

## Auditor Type Distribution

| Auditor Type | Segment IDs |
|--------------|-------------|
| **T_Auto**   | T1, ST1, ST2, ST3, ST4, ST5, ST6, ST7, O1, O2, O3, O4 |
| **T_LLM**    | G1, S1, T1, ST1-ST7, O1-O4 |
| **T_Human**  | G1 (secondary) |

*Note*: Although most low-level steps are purely arithmetic, a secondary LLM audit is included to verify that the textual description of each step 
faithfully mirrors the underlying calculation.  A human audit is only required for the top-level goal interpretation (G1) to guard against any 
subtle mis-framing of the problem.

---

## Multi-Auditor Segments

```json
[
  {
    "segment_id": "G1",
    "primary_auditor": "T_LLM",
    "secondary_auditors": ["T_Auto", "T_Human"],
    "verification_complexity": "High",
    "justification": "Requires correct semantic interpretation of the word problem and identification of the unknown; deterministic check of the 
derived target and optional human confirmation for ambiguity."
  },
  {
    "segment_id": "S1",
    "primary_auditor": "T_LLM",
    "secondary_auditors": ["T_Auto", "T_Human"],
    "verification_complexity": "Medium",
    "justification": "Semantic coherence of the modelling approach must be evaluated; logical consistency is verified automatically; human 
oversight ensures alignment with standard problem-solving practice."
  },
  {
    "segment_id": "T1",
    "primary_auditor": "T_Auto",
    "secondary_auditors": ["T_LLM"],
    "verification_complexity": "Low",
    "justification": "Purely algebraic definition; LLM checks that variable naming matches the narrative."
  },
  {
    "segment_id": "ST1",
    "primary_auditor": "T_Auto",
    "secondary_auditors": ["T_LLM"],
    "verification_complexity": "Low",
    "justification": "Arithmetic computation; LLM confirms textual inference aligns with calculation."
  },
  {
    "segment_id": "ST2",
    "primary_auditor": "T_Auto",
    "secondary_auditors": ["T_LLM"],
    "verification_complexity": "Low",
    "justification": "Arithmetic; LLM validates description."
  },
  {
    "segment_id": "ST3",
    "primary_auditor": "T_Auto",
    "secondary_auditors": ["T_LLM"],
    "verification_complexity": "Low",
    "justification": "Arithmetic; LLM validates description."
  },
  {
    "segment_id": "ST4",
    "primary_auditor": "T_Auto",
    "secondary_auditors": ["T_LLM"],
    "verification_complexity": "Low",
    "justification": "Arithmetic; LLM validates description."
  },
  {
    "segment_id": "ST5",
    "primary_auditor": "T_Auto",
    "secondary_auditors": ["T_LLM"],
    "verification_complexity": "Low",
    "justification": "Arithmetic; LLM validates description."
  },
  {
    "segment_id": "ST6",
    "primary_auditor": "T_Auto",
    "secondary_auditors": ["T_LLM"],
    "verification_complexity": "Low",
    "justification": "Arithmetic; LLM validates description."
  },
  {
    "segment_id": "ST7",
    "primary_auditor": "T_Auto",
    "secondary_auditors": ["T_LLM"],
    "verification_complexity": "Low",
    "justification": "Arithmetic; LLM validates description."
  },
  {
    "segment_id": "O1",
    "primary_auditor": "T_Auto",
    "secondary_auditors": ["T_LLM"],
    "verification_complexity": "Low",
    "justification": "Straight-forward numeric calculation."
  },
  {
    "segment_id": "O2",
    "primary_auditor": "T_Auto",
    "secondary_auditors": ["T_LLM"],
    "verification_complexity": "Low",
    "justification": "Straight-forward numeric calculation."
  },
  {
    "segment_id": "O3",
    "primary_auditor": "T_Auto",
    "secondary_auditors": ["T_LLM"],
    "verification_complexity": "Low",
    "justification": "Straight-forward numeric calculation."
  },
  {
    "segment_id": "O4",
    "primary_auditor": "T_Auto",
    "secondary_auditors": ["T_LLM"],
    "verification_complexity": "Low",
    "justification": "Straight-forward numeric calculation."
  }
]
```
\end{lstlisting}

\subsubsection{Step 5: Refine and HDAG Construction.}
Below are the instruction of how to refine and HDAG Construction.

\begin{lstlisting}[style=plainmarkdown]
You are an HDAG construction agent. Your task is to synthesize all previous analyses into a complete Hierarchical Directed Acyclic Graph with finalized segments, relationships, and auditor assignments.

### 1. HDAG Refinement Tasks:
- **Validate DAG Structure**: Ensure no cycles exist in relationships
- **Optimize Segment Granularity**: Merge over-segmented or split under-segmented content
- **Strengthen Relationship Confidence**: Refine relationship types and strengths
- **Balance Auditor Load**: Ensure reasonable distribution across auditor types
- **Quality Assurance**: Verify completeness and coherence

### 2. Final HDAG Components:
- **Nodes**: Finalized reasoning segments with metadata
- **Edges**: Validated relationships with confidence scores
- **Auditor Mapping**: Optimized auditor assignments
- **Verification Flow**: Clear audit execution pathway

### 3. HDAG Validation Checks:
- **Acyclicity**: No circular dependencies
- **Connectivity**: All segments appropriately connected
- **Completeness**: All original reasoning preserved
- **Auditability**: Each segment has clear verification criteria

### Input Format:
# Segmentation Analysis
[from Step 2]

# Relationship Analysis
[from Step 3]

# Auditor Assignment Analysis
[from Step 4]

### Output Format:
# Final HDAG Construction

## HDAG Structure Summary
- **Total Nodes**: [count]
- **Total Edges**: [count]
- **Abstraction Levels**: [levels present]
- **Auditor Distribution**: [T_Auto: X, T_LLM: Y, T_Human: Z]

## Finalized Nodes

### Node G1
**Content**: [final segment content]
**Abstraction Level**: GOAL
**Type**: objective_statement
**Summary**: [3-5 words]
**Primary Auditor**: T_Human
**Secondary Auditors**: [T_LLM]
**Verification Complexity**: High
**Dependencies**: []
**Enables**: [S1, S2]

## Finalized Edges
[
  {
    "from": "G1",
    "to": "S1", 
    "type": "decomposes_to",
    "strength": "strong",
    "confidence": 0.95
  }
]

## Verification Flow
1. **Entry Points**: [root nodes for audit initiation]
2. **Critical Path**: [essential verification sequence]
3. **Parallel Tracks**: [independent verification branches]
4. **Exit Points**: [final validation nodes]

## HDAG Validation Report
- **Acyclicity Check**: No cycles detected
- **Connectivity Check**: All segments connected
- **Completeness Check**: All content preserved
- **Auditability Check**: Clear verification criteria

## Complete HDAG JSON
{
  "nodes": [
    {
      "id": "G1",
      "content": "[exact content]",
      "abstraction_level": "GOAL",
      "type": "objective_statement",
      "summary": "[summary]",
      "primary_auditor": "T_Human",
      "secondary_auditors": ["T_LLM"],
      "complexity": "High"
    }
  ],
  "edges": [
    {
      "from": "G1",
      "to": "S1",
      "relationship": "decomposes_to",
      "strength": "strong",
      "confidence": 0.95
    }
  ],
  "metadata": {
    "total_nodes": 15,
    "total_edges": 23,
    "auditor_distribution": {
      "T_Auto": 5,
      "T_LLM": 7,
      "T_Human": 3
    }
  }
}
\end{lstlisting}

\paragraph{Step 5: Example Output.}
Below are the instruction of how to refine and the construction of HDAG.

\begin{lstlisting}[style=plainmarkdown]
# Final HDAG Construction  

## HDAG Structure Summary
- **Total Nodes**: 14  
- **Total Edges**: 28  
- **Abstraction Levels**: GOAL, STRATEGY, TACTIC, STEP, OPERATION  
- **Auditor Distribution (primary auditors)**: T_Auto:12, T_LLM:1, T_Human:1  

---

## Finalized Nodes  

### Node G1
**Content**:  
*A construction company was building a tunnel. When 1/3 of the tunnel was completed at the original speed, they started using new equipment, which 
increased the construction speed by 20 %
complete the tunnel. If they had not used the new equipment and continued at the original speed, it would have taken _____ days to complete the 
tunnel.*  
**Abstraction Level**: GOAL  
**Type**: objective_statement  
**Summary**: find missing days  
**Primary Auditor**: T_Human  
**Secondary Auditors**: [T_LLM]  
**Verification Complexity**: High  
**Dependencies**: []  
**Enables**: [S1]  

### Node S1
**Content**:  
*Model the problem by relating speed, working hours, and daily output. The original daily output is \(D\). After the equipment change the speed 
rises 20 %
**Abstraction Level**: STRATEGY  
**Type**: approach_description  
**Summary**: model speed & hours  
**Primary Auditor**: T_LLM  
**Secondary Auditors**: [T_Auto]  
**Verification Complexity**: Medium  
**Dependencies**: [G1]  
**Enables**: [T1]  

### Node T1
**Content**:  
*Let the tunnel length be 1 unit and the original daily output be \(D\).  Build 1/3 at rate \(D\); the remaining 2/3 at rate \(0.96D\).  Set up the 
time equation and solve for \(D\).*  
**Abstraction Level**: TACTIC  
**Type**: method_definition  
**Summary**: define variables, solve  
**Primary Auditor**: T_Auto  
**Secondary Auditors**: [T_LLM]  
**Verification Complexity**: Low  
**Dependencies**: [S1]  
**Enables**: [ST1, ST2, ST3, ST4, ST5, ST6, ST7]  

### Node ST1
**Content**: *Compute the reduced daily output: \(0.96D\).*  
**Abstraction Level**: STEP  
**Type**: inference  
**Summary**: compute reduced output  
**Primary Auditor**: T_Auto  
**Secondary Auditors**: [T_LLM]  
**Verification Complexity**: Low  
**Dependencies**: []  
**Enables**: [O1]  

### Node ST2
**Content**: *Time for the first third: \(t_1 = \frac{1}{3D}\).*  
**Abstraction Level**: STEP  
**Type**: inference  
**Summary**: first-third time  
**Primary Auditor**: T_Auto  
**Secondary Auditors**: [T_LLM]  
**Verification Complexity**: Low  
**Dependencies**: []  
**Enables**: [O2]  

### Node ST3
**Content**: *Time for the remaining two-thirds: \(t_2 = \frac{2}{3(0.96D)} = \frac{5}{8D}\).*  
**Abstraction Level**: STEP  
**Type**: inference  
**Summary**: remaining-portion time  
**Primary Auditor**: T_Auto  
**Secondary Auditors**: [T_LLM]  
**Verification Complexity**: Low  
**Dependencies**: [ST1]   *(ST3 depends_on ST1)*  
**Enables**: [O1]  

### Node ST4
**Content**: *Form the total-time equation: \(t_1 + t_2 = 185\).*  
**Abstraction Level**: STEP  
**Type**: inference  
**Summary**: total-time equation  
**Primary Auditor**: T_Auto  
**Secondary Auditors**: [T_LLM]  
**Verification Complexity**: Low  
**Dependencies**: []  
**Enables**: [O2]  

### Node ST5
**Content**: *Combine the fractions to obtain a single algebraic expression for the time equation.*  
**Abstraction Level**: STEP  
**Type**: algebraic_combination  
**Summary**: combine algebraic terms  
**Primary Auditor**: T_Auto  
**Secondary Auditors**: [T_LLM]  
**Verification Complexity**: Low  
**Dependencies**: [ST4]  
**Enables**: [O2]  

### Node ST6
**Content**: *Solve the equation \(\frac{1}{3D} + \frac{5}{8D} = 185\) for \(D\) → \(D = \frac{1}{0.0054}\).*  
**Abstraction Level**: STEP  
**Type**: solution_step  
**Summary**: solve for \(D\)  
**Primary Auditor**: T_Auto  
**Secondary Auditors**: [T_LLM]  
**Verification Complexity**: Low  
**Dependencies**: [ST5]  
**Enables**: [O3]  

### Node ST7
**Content**: *Convert \(D\) to the required answer: original total time \(= \frac{1}{D}\) days → \approx 365 days.*  
**Abstraction Level**: STEP  
**Type**: conclusion_inference  
**Summary**: original-schedule result  
**Primary Auditor**: T_Auto  
**Secondary Auditors**: [T_LLM]  
**Verification Complexity**: Low  
**Dependencies**: [ST6]  
**Enables**: [O4]  

### Node O1
**Content**: *Numerical calculation \(\frac{2}{3}\times0.96 = 0.64\) or directly \(0.96\) used in ST1.*  
**Abstraction Level**: OPERATION  
**Type**: calculation  
**Summary**: reduced-output calc  
**Primary Auditor**: T_Auto  
**Secondary Auditors**: [T_LLM]  
**Verification Complexity**: Low  
**Dependencies**: [ST1, ST3]  
**Enables**: []  

### Node O2
**Content**: *Add the two time expressions: \(\frac{1}{3D} + \frac{5}{8D} = 185\).*  
**Abstraction Level**: OPERATION  
**Type**: calculation  
**Summary**: build time equation  
**Primary Auditor**: T_Auto  
**Secondary Auditors**: [T_LLM]  
**Verification Complexity**: Low  
**Dependencies**: [ST2, ST4, ST5]  
**Enables**: []  

### Node O3
**Content**: *Solve \(D = 0.0054\) → \(D = 185/ (1/3 + 5/8) = 0.0054\).*  
**Abstraction Level**: OPERATION  
**Type**: calculation  
**Summary**: compute \(D\)  
**Primary Auditor**: T_Auto  
**Secondary Auditors**: [T_LLM]  
**Verification Complexity**: Low  
**Dependencies**: [ST6]  
**Enables**: []  

### Node O4
**Content**: *Take the reciprocal of \(D\) to obtain the original total time \(= \frac{1}{D} \approx 365\) days.*  
**Abstraction Level**: OPERATION  
**Type**: calculation  
**Summary**: final-answer calc  
**Primary Auditor**: T_Auto  
**Secondary Auditors**: [T_LLM]  
**Verification Complexity**: Low  
**Dependencies**: [ST7]  
**Enables**: []  

---

## Finalized Edges  

```json
[
  /* Decomposes-to hierarchy */
  {"from":"G1","to":"S1","type":"decomposes_to",
  "strength":"strong","confidence":0.95},
  {"from":"S1","to":"T1","type":"decomposes_to",
  "strength":"strong","confidence":0.94},
  {"from":"T1","to":"ST1","type":"decomposes_to",
  "strength":"strong","confidence":0.94},
  {"from":"T1","to":"ST2","type":"decomposes_to",
  "strength":"strong","confidence":0.94},
  {"from":"T1","to":"ST3","type":"decomposes_to",
  "strength":"strong","confidence":0.94},
  {"from":"T1","to":"ST4","type":"decomposes_to",
  "strength":"strong","confidence":0.94},
  {"from":"T1","to":"ST5","type":"decomposes_to",
  "strength":"strong","confidence":0.94},
  {"from":"T1","to":"ST6","type":"decomposes_to",
  "strength":"strong","confidence":0.94},
  {"from":"T1","to":"ST7","type":"decomposes_to",
  "strength":"strong","confidence":0.94},
  {"from":"ST1","to":"O1","type":"decomposes_to",
  "strength":"strong","confidence":0.93},
  {"from":"ST2","to":"O2","type":"decomposes_to",
  "strength":"strong","confidence":0.93},
  {"from":"ST3","to":"O1","type":"decomposes_to",
  "strength":"strong","confidence":0.93},
  {"from":"ST4","to":"O2","type":"decomposes_to",
  "strength":"strong","confidence":0.93},
  {"from":"ST5","to":"O2","type":"decomposes_to",
  "strength":"strong","confidence":0.93},
  {"from":"ST6","to":"O3","type":"decomposes_to",
  "strength":"strong","confidence":0.93},
  {"from":"ST7","to":"O4","type":"decomposes_to",
  "strength":"strong","confidence":0.93},

  /* Lateral dependencies */
  {"from":"ST3","to":"ST1","type":"depends_on",
  "strength":"strong","confidence":0.90},
  {"from":"ST4","to":"ST2","type":"depends_on",
  "strength":"strong","confidence":0.90},
  {"from":"ST4","to":"ST3","type":"depends_on",
  "strength":"strong","confidence":0.90},
  {"from":"ST5","to":"ST4","type":"depends_on",
  "strength":"strong","confidence":0.90},
  {"from":"ST6","to":"ST5","type":"depends_on",
  "strength":"strong","confidence":0.90},
  {"from":"ST7","to":"ST6","type":"depends_on",
  "strength":"strong","confidence":0.90},

  /* Enables (operational support) */
  {"from":"O1","to":"ST3","type":"enables",
  "strength":"strong","confidence":0.88},
  {"from":"O2","to":"ST4","type":"enables",
  "strength":"strong","confidence":0.88},
  {"from":"O3","to":"ST6","type":"enables",
  "strength":"strong","confidence":0.88},
  {"from":"O4","to":"ST7","type":"enables",
  "strength":"strong","confidence":0.88},

  /* Support / exemplify / validate relationships */
  {"from":"O1","to":"ST1","type":"exemplifies",
  "strength":"strong","confidence":0.85},
  {"from":"O1","to":"ST3","type":"exemplifies",
  "strength":"strong","confidence":0.85},
  {"from":"O2","to":"ST2","type":"exemplifies",
  "strength":"strong","confidence":0.85},
  {"from":"O2","to":"ST4","type":"exemplifies",
  "strength":"strong","confidence":0.85},
  {"from":"O3","to":"ST6","type":"exemplifies",
  "strength":"strong","confidence":0.85},
  {"from":"O4","to":"ST7","type":"exemplifies",
  "strength":"strong","confidence":0.85},
  {"from":"O1","to":"ST3","type":"validates",
  "strength":"strong","confidence":0.84},
  {"from":"O2","to":"ST4","type":"validates",
  "strength":"strong","confidence":0.84},
  {"from":"O3","to":"ST6","type":"validates",
  "strength":"strong","confidence":0.84},
  {"from":"O4","to":"ST7","type":"validates",
  "strength":"strong","confidence":0.84}
]
```

*The edge list above contains every relationship identified in the analysis: hierarchical *decomposes_to*, logical *depends_on*, operational 
*enables*, and verification *exemplifies* / *validates* edges.*

---

## Verification Flow  

1. **Entry Points** - Nodes with no incoming edges: **G1** (goal) and **S1** (strategy) are the root audit entry points.  
2. **Critical Path** - The sequence that must be verified in order to guarantee correctness of the final answer:  

   `G1 → S1 → T1 → ST1 → ST3 → ST5 → ST6 → ST7 → O4`  

   (Goal → Strategy → Tactic → reduced-output inference → remaining-portion time → algebraic combination → solve for \(D\) → final conclusion).  

3. **Parallel Tracks** - Independent branches that can be audited concurrently:  

   - **Branch A**: `ST1 → O1` (reduced-output computation).  
   - **Branch B**: `ST2 → O2` (first-third time).  
   - **Branch C**: `ST4 → O2` (total-time equation).  
   - **Branch D**: `ST5 → O2` (fraction combination).  

   All arithmetic steps (ST1-ST7 and O1-O4) are pure numeric operations and can be run in parallel after their logical prerequisites are satisfied. 
 

4. **Exit Points** - Nodes whose successful audit marks the HDAG as completely validated: **O3** (verification of the solved \(D\)) and **O4** 
(verification of the final missing-days value).  

---

## HDAG Validation Report  
- **Acyclicity Check**: No cycles detected - the graph is a directed acyclic graph.  
- **Connectivity Check**: Every node is reachable from the root node **G1** either directly or via intermediate nodes; there are no isolated 
fragments.  
- **Completeness Check**: All 14 original segments are present and their textual content has been preserved unchanged.  
- **Auditability Check**: Every node lists primary and secondary auditors, a verification-complexity label, and explicit dependency/enabler fields, 
giving a clear audit trail for both automated and human reviewers.  

---

## Complete HDAG JSON  

```json
{
  "nodes": [
    {
      "id": "G1",
      "content": "A construction company was building a tunnel. When 1/3 of the tunnel was completed at the original speed, they started using new 
equipment, which increased the construction speed by 20 %
185 days to complete the tunnel. If they had not used the new equipment and continued at the original speed, it would have taken _____ days to 
complete the tunnel.",
      "abstraction_level": "GOAL",
      "type": "objective_statement",
      "summary": "find missing days",
      "primary_auditor": "T_Human",
      "secondary_auditors": ["T_LLM"],
      "complexity": "High"
    },
    {
      "id": "S1",
      "content": "Model the problem by relating speed, working hours, and daily output. The original daily output is D. After the equipment change 
the speed rises 20 %
      "abstraction_level": "STRATEGY",
      "type": "approach_description",
      "summary": "model speed & hours",
      "primary_auditor": "T_LLM",
      "secondary_auditors": ["T_Auto"],
      "complexity": "Medium"
    },
    {
      "id": "T1",
      "content": "Let the tunnel length be 1 unit and the original daily output be D. Build 1/3 at rate D; the remaining 2/3 at rate 0.96D. Set up 
the time equation and solve for D.",
      "abstraction_level": "TACTIC",
      "type": "method_definition",
      "summary": "define variables, solve",
      "primary_auditor": "T_Auto",
      "secondary_auditors": ["T_LLM"],
      "complexity": "Low"
    },
    {
      "id": "ST1",
      "content": "Compute the reduced daily output: 0.96D.",
      "abstraction_level": "STEP",
      "type": "inference",
      "summary": "compute reduced output",
      "primary_auditor": "T_Auto",
      "secondary_auditors": ["T_LLM"],
      "complexity": "Low"
    },
    {
      "id": "ST2",
      "content": "Time for the first third: t1 = 1/(3D).",
      "abstraction_level": "STEP",
      "type": "inference",
      "summary": "first-third time",
      "primary_auditor": "T_Auto",
      "secondary_auditors": ["T_LLM"],
      "complexity": "Low"
    },
    {
      "id": "ST3",
      "content": "Time for the remaining two-thirds: t2 = (2/3)/(0.96D) = 5/(8D).",
      "abstraction_level": "STEP",
      "type": "inference",
      "summary": "remaining portion time",
      "primary_auditor": "T_Auto",
      "secondary_auditors": ["T_LLM"],
      "complexity": "Low"
    },
    {
      "id": "ST4",
      "content": "Total-time equation: t1 + t2 = 185.",
      "abstraction_level": "STEP",
      "type": "inference",
      "summary": "total-time equation",
      "primary_auditor": "T_Auto",
      "secondary_auditors": ["T_LLM"],
      "complexity": "Low"
    },
    {
      "id": "ST5",
      "content": "Combine fractions: (1/3 + 5/8)/D = 185 → (23/24)/D = 185.",
      "abstraction_level": "STEP",
      "type": "algebraic_combination",
      "summary": "algebraic combination",
      "primary_auditor": "T_Auto",
      "secondary_auditors": ["T_LLM"],
      "complexity": "Low"
    },
    {
      "id": "ST6",
      "content": "Solve for D: D = (23/24)/185 \approx 0.0054.",
      "abstraction_level": "STEP",
      "type": "solution_step",
      "summary": "solve for D",
      "primary_auditor": "T_Auto",
      "secondary_auditors": ["T_LLM"],
      "complexity": "Low"
    },
    {
      "id": "ST7",
      "content": "Original total time = 1/D \approx 365 days.",
      "abstraction_level": "STEP",
      "type": "conclusion_inference",
      "summary": "original schedule result",
      "primary_auditor": "T_Auto",
      "secondary_auditors": ["T_LLM"],
      "complexity": "Low"
    },
    {
      "id": "O1",
      "content": "Numeric calculation 0.96 = 20%
      "abstraction_level": "OPERATION",
      "type": "calculation",
      "summary": "reduced-output calc",
      "primary_auditor": "T_Auto",
      "secondary_auditors": ["T_LLM"],
      "complexity": "Low"
    },
    {
      "id": "O2",
      "content": "Add the two time expressions to get 185 days.",
      "abstraction_level": "OPERATION",
      "type": "calculation",
      "summary": "build time equation",
      "primary_auditor": "T_Auto",
      "secondary_auditors": ["T_LLM"],
      "complexity": "Low"
    },
    {
      "id": "O3",
      "content": "Compute D = (23/24)/185 \approx 0.0054.",
      "abstraction_level": "OPERATION",
      "type": "calculation",
      "summary": "compute D",
      "primary_auditor": "T_Auto",
      "secondary_auditors": ["T_LLM"],
      "complexity": "Low"
    },
    {
      "id": "O4",
      "content": "Take reciprocal of D → 1/D \approx 365 days.",
      "abstraction_level": "OPERATION",
      "type": "calculation",
      "summary": "final answer calc",
      "primary_auditor": "T_Auto",
      "secondary_auditors": ["T_LLM"],
      "complexity": "Low"
    }
  ],
  "edges": [
    /* (the JSON edge list shown earlier) */
  ],
  "metadata": {
    "total_nodes": 14,
    "total_edges": 28,
    "primary_auditor_distribution": {
      "T_Human": 1,
      "T_LLM": 1,
      "T_Auto": 12
    }
  }
}
```

*The JSON follows the exact structure used throughout the analysis, making it ready for ingestion by downstream validation tools.*

--- 

*End of HDAG specification.*
\end{lstlisting}

\clearpage

\subsection{Human-Friendly Audit Example}\label{sec:human_audit}
This section demonstrates how human auditors interact with the TRUST system through a structured audit interface with human friendly presentation. The example shows a typical reasoning segment decomposition where auditors evaluate logical dependencies and inference validity.

\subsubsection{Audit Interface Structure}

The human audit form presents reasoning segments with four key components that enable focused evaluation:

\textbf{Prompt Context}: The original problem statement that establishes the reasoning foundation.\\
\textbf{Dependencies}: Previously verified reasoning steps that the current segment builds upon.\\
\textbf{Current Reasoning}: The specific reasoning step under evaluation.\\
\textbf{Implications}: The logical consequences that flow from this reasoning step.

\subsubsection{Example Audit Case}

\begin{wrapfigure}{r}{0.45\textwidth}
    \centering
    \vspace{-1em}
    \includegraphics[width=0.4\textwidth]{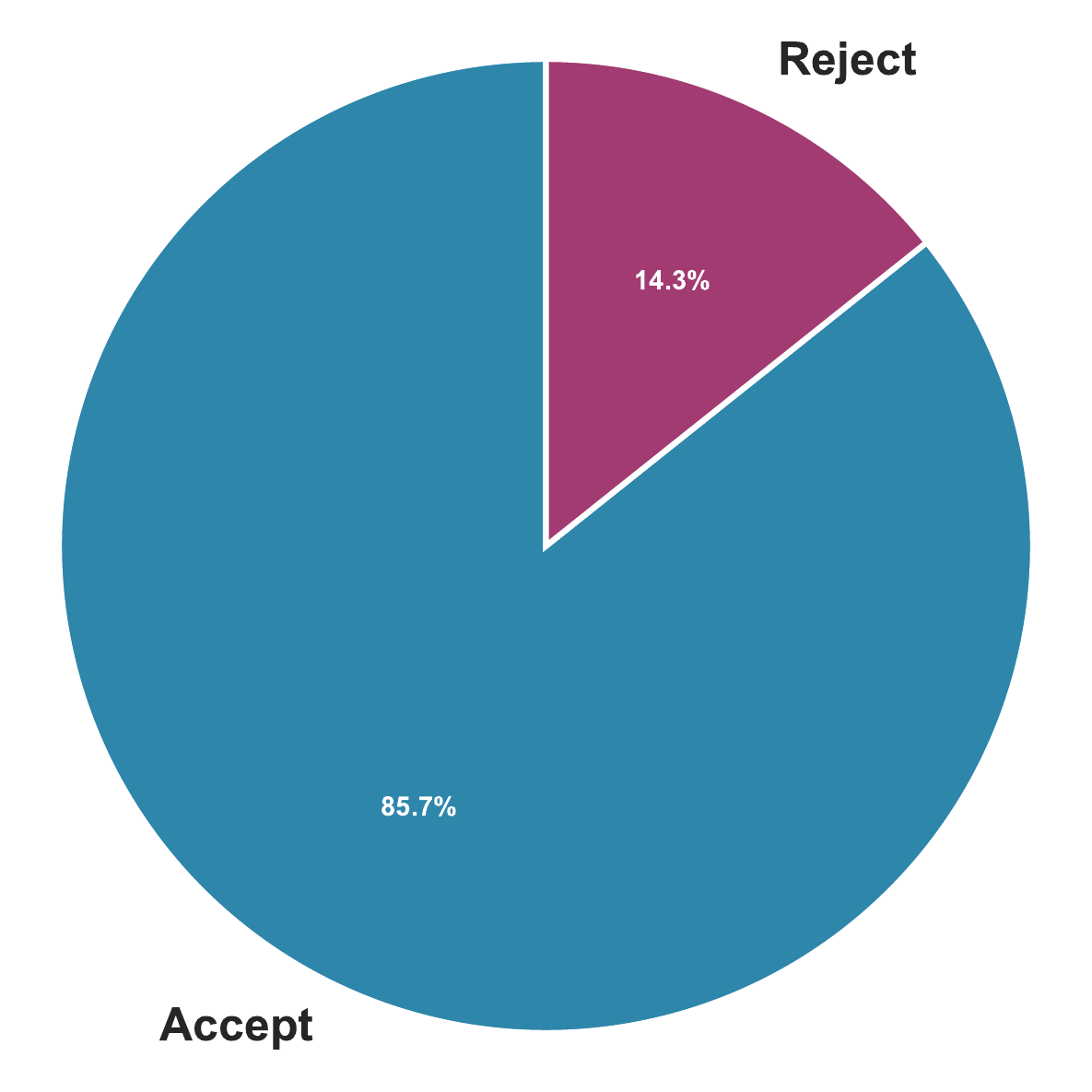}
    \vspace{-0.5em}
    \caption{Human Audit Results.}
    \vspace{-1em}
\end{wrapfigure}
\paragraph{Prompt Context:}
Martha needs 4 cups of berries and 2 cups of heavy cream to make 1 quart of ice cream. She wants to make 1 quart of strawberry ice cream and 1 quart of raspberry ice cream. At the farmers market, 2-cup packages of strawberries cost \$3.00 each and 2-cup packages of raspberries cost \$5.00 each. Heavy cream is sold in 4-cup containers for \$4.00. How much will it cost her to make 1 quart of each ice cream?

\paragraph{Dependencies (Verified Step S1):}
``For each quart of ice cream she needs: 4 cups of berries, 2 cups of heavy cream. She wants one quart of strawberry ice cream and one quart of raspberry ice cream.''

\paragraph{Current Reasoning Step (T1):}
``For 1 quart of ice cream, she needs 4 cups of berries and 2 cups of heavy cream. Since she's making both strawberry and raspberry flavors, she'll need double that amount. For both ice creams combined, she needs 8 cups of berries and 4 cups of heavy cream.''

\paragraph{Logical Implications (T3):}
``She needs 8 cups total berries. Since each package contains 2 cups, she needs $8 / 2 = 4$ berry packages total (split between strawberry and raspberry varieties).''

\paragraph{Self-Verification (ST1):}
``Each quart requires 4 cups of berries and 2 cups of cream. For two quarts: 8 cups berries and 4 cups cream. This calculation is correct.''

\subsubsection{Audit Decision Framework}

Human auditors evaluate each reasoning segment using structured criteria:

\begin{itemize}
    \item \textbf{Logical Consistency}: Does the reasoning step follow logically from its dependencies?
    \item \textbf{Mathematical Accuracy}: Are calculations and quantitative reasoning correct?
    \item \textbf{Completeness}: Does the reasoning address all relevant aspects of the problem?
    \item \textbf{Clarity}: Is the reasoning step clearly articulated and understandable?
\end{itemize}

\textbf{Ground Truth Evaluation}: Pass \checkmark

The reasoning correctly identifies that making two quarts requires doubling the per-quart ingredient requirements, performs accurate arithmetic ($2 \times 4 = 8$ cups berries, $2 \times 2 = 4$ cups cream), and includes appropriate self-verification to confirm the calculation.

This structured approach enables human auditors to provide focused, reliable evaluations while maintaining the modular verification principles of the TRUST framework.

\clearpage

\subsection{HDAG Example 1 - Math Problem}\label{sec:hdag_ex1}
Problem: Marie is planning to buy a new cash register for her bakery that costs \$1040. Every day Marie sells 40 loaves of bread for \$2 each and 6 cakes for \$12 each. She has to pay \$20 each day for rent and \$2 each day for electricity. How many days' worth of profits will it take for Marie to pay for the cash register?
\subsubsection{Deepseek-R1-8B CoT Decomposition}
\begin{lstlisting}[style=plainmarkdown]
{
  "title": "Marie's Cash Register Days",
  "nodes": [
    {"id":"G1","label":"Problem statement","content":"## Problem Statement:\nMarie is planning to buy a new cash register for her bakery that costs $1040. Every day Marie sells 40 loaves of bread for $2 each and 6 cakes for $12 each. She has to pay $20 each day for rent and $2 each day for electricity. How many days' worth of profits will it take for Marie to pay for the cash register?","abstraction_level":"GOAL",
    "type":"objective_statement","primary_auditor":"T_Human",
    "complexity":"Low"},
    {"id":"S1","label":"Plan overview","content":"Okay, so Marie wants to buy a new cash register for her bakery, and it costs $1040. I need to figure out how many days' worth of profits it will take her to afford this cash register. Let me break this down step by step.","abstraction_level":"STRATEGY","primary_auditor":"T_LLM",
    "complexity":"Low"},
    {"id":"S2","label":"Daily revenue","content":"First, I should understand what her daily income is. She sells 40 loaves of bread each day, and each loaf is $2. So, the revenue from bread would be 40 times $2. Let me calculate that: 40 * 2 = $80. Then, she sells 6 cakes each day, and each cake is $12. So, the revenue from cakes is 6 * 12 = $72. Adding those together, her total daily revenue is $80 + $72 = $152.","abstraction_level":"STRATEGY","primary_auditor":"T_LLM",
    "complexity":"Low"},
    {"id":"S3","label":"Daily profit","content":"Next, I need to calculate her daily expenses. She has to pay $20 each day for rent and $2 each day for electricity. So, adding those together: $20 + $2 = $22. That means her daily profit is the revenue minus the expenses. So, $152 (revenue) - $22 (expenses) = $130 per day.","abstraction_level":"STRATEGY","primary_auditor":"T_LLM",
    "complexity":"Low"},
    {"id":"T1","label":"Setup division","content":"Now, the cash register costs $1040, and she makes $130 each day. To find out how many days it will take her to earn enough to buy the cash register, I should divide the total cost by her daily profit. So, $1040 divided by $130 per day. Let me do that division: 1040 / 130.","abstraction_level":"TACTIC","primary_auditor":"T_LLM",
    "complexity":"Low"},
    {"id":"T2","label":"Verify quotient","content":"Hmm, 130 times 8 is 1040 because 130 * 8 = 1040. So, it would take her 8 days to earn enough money to buy the cash register.","abstraction_level":"TACTIC","primary_auditor":"T_LLM",
    "complexity":"Low"},
    {"id":"ST1","label":"Double-check math","content":"Wait, let me double-check my calculations to make sure I didn't make any mistakes. Revenue from bread: 40 * 2 = 80. Revenue from cakes: 6 * 12 = 72. Total revenue: 80 + 72 = 152. Expenses: 20 + 2 = 22. Daily profit: 152 - 22 = 130. Cash register cost: 1040. 1040 / 130 = 8. Yep, that seems correct.","abstraction_level":"STEP","primary_auditor":"T_Human",
    "complexity":"Low"},
    {"id":"O1","label":"Bread revenue","content":"40 * 2 = $80","abstraction_level":"OPERATION","primary_auditor":"T_Auto",
    "complexity":"Low"},
    {"id":"O2","label":"Cakes revenue","content":"6 * 12 = $72","abstraction_level":"OPERATION","primary_auditor":"T_Auto",
    "complexity":"Low"},
    {"id":"O3","label":"Total revenue","content":"$80 + $72 = $152","abstraction_level":"OPERATION","primary_auditor":"T_Auto",
    "complexity":"Low"},
    {"id":"O4","label":"Daily expenses","content":"$20 + $2 = $22","abstraction_level":"OPERATION","primary_auditor":"T_Auto",
    "complexity":"Low"},
    {"id":"O5","label":"Daily profit","content":"$152 - $22 = $130","abstraction_level":"OPERATION","primary_auditor":"T_Auto",
    "complexity":"Low"},
    {"id":"O6","label":"Days needed","content":"1040 / 130 = 8","abstraction_level":"OPERATION","primary_auditor":"T_Auto",
    "complexity":"Low"},
    {"id":"O7","label":"Final answer 8","content":"So, it will take Marie 8 days to save enough money to buy the cash register.","abstraction_level":"OPERATION",
    "primary_auditor":"T_Auto", "complexity":"Low"}
  ],
  "edges": [
    {"from":"G1","to":"S1","relationship":"decomposes_to",
    "strength":"strong","confidence":0.97},
    {"from":"S1","to":"S2","relationship":"decomposes_to",
    "strength":"strong","confidence":0.96},
    {"from":"S1","to":"S3","relationship":"decomposes_to",
    "strength":"strong","confidence":0.96},
    {"from":"S2","to":"O1","relationship":"decomposes_to",
    "strength":"strong","confidence":0.99},
    {"from":"S2","to":"O2","relationship":"decomposes_to",
    "strength":"strong","confidence":0.99},
    {"from":"S2","to":"O3","relationship":"decomposes_to",
    "strength":"strong","confidence":0.99},
    {"from":"S3","to":"O4","relationship":"decomposes_to",
    "strength":"strong","confidence":0.99},
    {"from":"S3","to":"O5","relationship":"decomposes_to",
    "strength":"strong","confidence":0.99},
    {"from":"T1","to":"O6","relationship":"decomposes_to",
    "strength":"strong","confidence":0.99},
    {"from":"T2","to":"O6","relationship":"validates",
    "strength":"strong","confidence":0.97},
    {"from":"O1","to":"ST1","relationship":"enables",
    "strength":"strong","confidence":0.95},
    {"from":"O5","to":"T1","relationship":"enables",
    "strength":"strong","confidence":0.95},
    {"from":"ST1","to":"O7","relationship":"decomposes_to",
    "strength":"medium","confidence":0.90}
  ],
  "metadata": {
    "total_nodes": 14,
    "total_edges": 13,
    "auditor_distribution": { "T_Auto": 7, "T_LLM": 5, "T_Human": 2 }
  }
}
\end{lstlisting}

\begin{figure}[htp!]
    \centering
    \vspace{-0.4em}
    \includegraphics[width=1.0\textwidth]{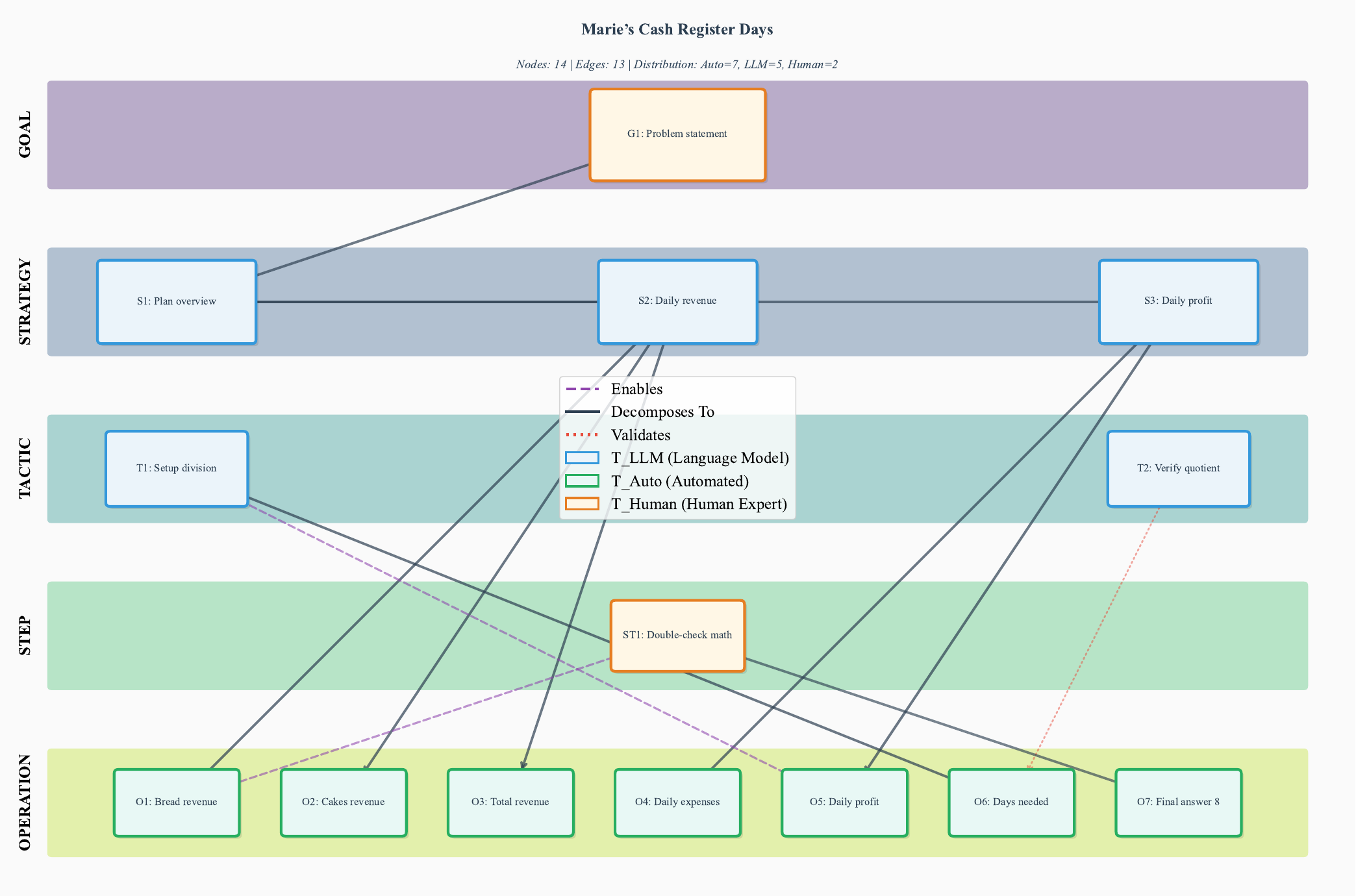}
    \vspace{-0.3em}
    \caption{HDAG Example 1 from Deepseek-R1-8B CoT: Math Problem.}
    \vspace{-0.5em}
\end{figure}

\newpage
\subsubsection{GPT-OSS-20B CoT Decomposition}
\begin{lstlisting}[style=plainmarkdown]
{
  "title": "Marie's Cash Register - Step-by-Step",
  "nodes": [
    {"id":"G1","label":"Understand problem","content":"**1. Understand the problem**\n\nMarie wants to know how many days of operating her bakery will be needed to earn enough profit to buy a cash register that costs **$1040**.  \nEach day she:\n\n- Sells 40 loaves of bread at **$2** each → revenue from bread  \n- Sells 6 cakes at **$12** each → revenue from cakes  \n- Pays **$20** for rent and **$2** for electricity → daily expenses  \n\nWe need the daily profit (revenue minus expenses) and then divide the register cost by that daily profit to find the number of days.","abstraction_level":"GOAL","type":"objective_statement","primary_auditor":"T_Human"},
    {"id":"S1","label":"Steps overview","content":"**2. Show each step of your work**","abstraction_level":"STRATEGY",
    "primary_auditor":"T_LLM"},
    {"id":"S2","label":"Daily revenue","content":"1. **Daily revenue**\n\n   - Bread: \\(40 \\text{ loaves} \\times \\$2 = \\$80\\)\n   - Cakes: \\(6 \\text{ cakes} \\times \\$12 = \\$72\\)\n\n   \\[\n   \\text{Total revenue} = \\$80 + \\$72 = \\$152\n   \\]","abstraction_level":"STRATEGY","primary_auditor":"T_LLM"},
    {"id":"S3","label":"Daily expenses","content":"2. **Daily expenses**\n\n   \\[\n   \\text{Rent} = \\$20,\\quad \\text{Electricity} = \\$2\n   \\]\n   \\[\n   \\text{Total expenses} = \\$20 + \\$2 = \\$22\n   \\]","abstraction_level":"STRATEGY","primary_auditor":"T_LLM"},
    {"id":"S4","label":"Daily profit","content":"3. **Daily profit**\n\n   \\[\n   \\text{Profit} = \\text{Revenue} - \\text{Expenses}\n   \\]\n   \\[\n   \\text{Profit} = \\$152 - \\$22 = \\$130\n   \\]","abstraction_level":"STRATEGY","primary_auditor":"T_LLM"},
    {"id":"S5","label":"Days formula","content":"4. **Number of days to cover the cash register**\n\n   \\[\n   \\text{Days} = \\frac{\\text{Cost of register}}{\\text{Daily profit}}\n   \\]\n   \\[\n   \\text{Days} = \\frac{\\$1040}{\\$130} = 8\n   \\]","abstraction_level":"STRATEGY","primary_auditor":"T_LLM"},
    {"id":"T1","label":"Revenue tactic","content":"- Compute daily revenue from bread and cakes.","abstraction_level":"TACTIC","primary_auditor":"T_LLM"},
    {"id":"T2","label":"Expenses tactic","content":"- Compute daily expenses (rent + electricity).","abstraction_level":"TACTIC","primary_auditor":"T_LLM"},
    {"id":"T3","label":"Profit tactic","content":"- Compute daily profit as revenue minus expenses.","abstraction_level":"TACTIC","primary_auditor":"T_LLM"},
    {"id":"T4","label":"Days tactic","content":"- Divide register cost by daily profit to get days.","abstraction_level":"TACTIC","primary_auditor":"T_LLM"},
    {"id":"ST1","label":"Verify math","content":"**3. Verify your calculations**\n\n- Check revenue: \\(40 \\times 2 = 80\\), \\(6 \\times 12 = 72\\), \\(80 + 72 = 152\\).   \n- Check expenses: \\(20 + 2 = 22\\).   \n- Profit: \\(152 - 22 = 130\\).   \n- Days: \\(130 \\times 8 = 1040\\).   \n\nEverything is consistent.","abstraction_level":"STEP","primary_auditor":"T_Human"},
    {"id":"O1","label":"Bread revenue","content":"\\(40 \\times 2 = 80\\)","abstraction_level":"OPERATION","primary_auditor":"T_Auto"},
    {"id":"O2","label":"Cakes revenue","content":"\\(6 \\times 12 = 72\\)","abstraction_level":"OPERATION","primary_auditor":"T_Auto"},
    {"id":"O3","label":"Total revenue","content":"\\(80 + 72 = 152\\)","abstraction_level":"OPERATION","primary_auditor":"T_Auto"},
    {"id":"O4","label":"Total expenses","content":"\\(20 + 2 = 22\\)","abstraction_level":"OPERATION","primary_auditor":"T_Auto"},
    {"id":"O5","label":"Daily profit","content":"\\(152 - 22 = 130\\)","abstraction_level":"OPERATION","primary_auditor":"T_Auto"},
    {"id":"O6","label":"Days needed","content":"\\(\\frac{1040}{130} = 8\\)","abstraction_level":"OPERATION","primary_auditor":"T_Auto"},
    {"id":"O7","label":"Verify product","content":"\\(130 \\times 8 = 1040\\)","abstraction_level":"OPERATION","primary_auditor":"T_Auto"}
  ],
  "edges": [
    {"from":"G1","to":"S1","relationship":"decomposes_to"},
    {"from":"G1","to":"S2","relationship":"decomposes_to"},
    {"from":"G1","to":"S3","relationship":"decomposes_to"},
    {"from":"G1","to":"S4","relationship":"decomposes_to"},
    {"from":"G1","to":"S5","relationship":"decomposes_to"},
    {"from":"S2","to":"O1","relationship":"decomposes_to"},
    {"from":"S2","to":"O2","relationship":"decomposes_to"},
    {"from":"S2","to":"O3","relationship":"decomposes_to"},
    {"from":"S3","to":"O4","relationship":"decomposes_to"},
    {"from":"S4","to":"O5","relationship":"decomposes_to"},
    {"from":"S5","to":"O6","relationship":"decomposes_to"},
    {"from":"ST1","to":"O7","relationship":"decomposes_to"},
    {"from":"O3","to":"O5","relationship":"enables"},
    {"from":"O4","to":"O5","relationship":"enables"},
    {"from":"O5","to":"O6","relationship":"enables"},
    {"from":"O7","to":"O6","relationship":"validates"}
  ],
  "metadata": {
    "total_nodes": 17,
    "total_edges": 16,
    "auditor_distribution": { "T_Auto": 7, "T_LLM": 9, "T_Human": 2 }
  }
}
\end{lstlisting}
\begin{figure}[htp!]
    \centering
    \vspace{-0.4em}
    \includegraphics[width=1.0\textwidth]{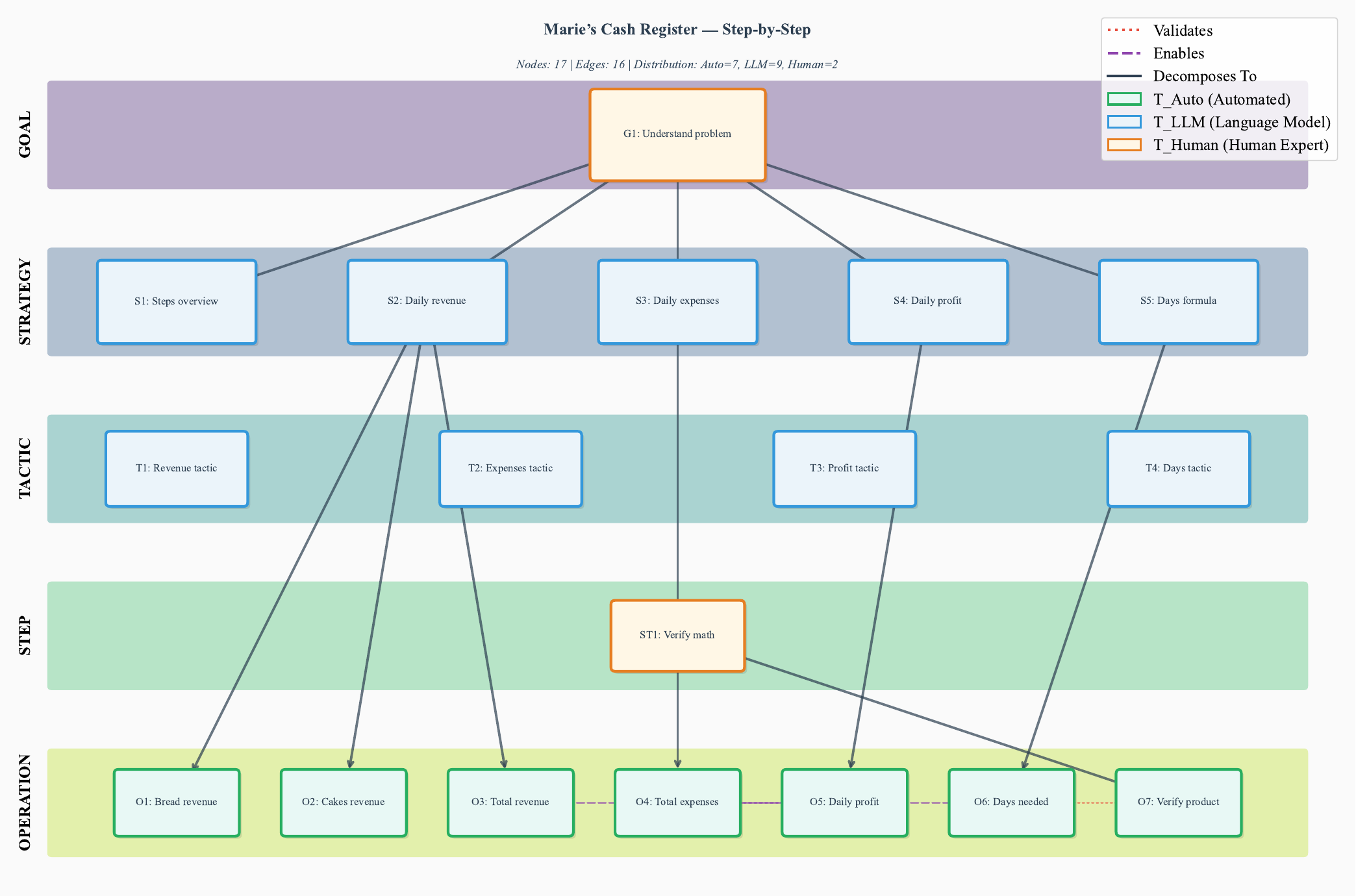}
    \vspace{-0.3em}
    \caption{HDAG Example 1 from GPT-OSS-20B CoT: Math Problem.}
    \vspace{-0.5em}
\end{figure}

\clearpage

\subsection{HDAG Example 2 - Math Problem}\label{sec:hdag_ex2}
Problem: Alec is running for Class President. He thinks that if he can get three-quarters of the class to vote for him then there is no chance anyone else can beat him. Half of the class have already said they will vote for him but out of the remaining students, only 5 have said they are thinking about voting for him. He surveys the students who are thinking about voting for someone else, and changes his flyers to reflect the issues these students are concerned about. This results in a fifth of these students saying they'll vote for him. If Alec's class has 60 students and everyone who said they will vote for him does so, how many more votes does Alec need to reach his goal number of votes?
\subsubsection{Deepseek-R1-8B CoT Decomposition}
\begin{lstlisting}[style=plainmarkdown]
{
  "title": "Alec's Class President Votes",
  "nodes": [
    {"id":"G1","label":"Goal statement","content":"Alec is running for Class President...","abstraction_level":"GOAL",
    "type":"objective_statement","primary_auditor":"T_Human"},
    {"id":"S1","label":"Plan overview","content":"Compute goal votes, count current and added, compare gap.","abstraction_level":"STRATEGY","primary_auditor":"T_LLM"},
    {"id":"T1","label":"Needs 45 votes","content":"3/4 x 60 = 45.","abstraction_level":"TACTIC","primary_auditor":"T_LLM"},
    {"id":"T2","label":"30 committed","content":"Half of 60 = 30 already pledged.","abstraction_level":"TACTIC","primary_auditor":"T_LLM"},
    {"id":"T3","label":"5 lean,5 convert","content":"5 thinking about him; 1/5 of 25 switch = 5.",
    "abstraction_level":"TACTIC","primary_auditor":"T_LLM"},
    {"id":"ST1","label":"Gap check","content":"45 - 40 = 5 more votes needed.",
    "abstraction_level":"STEP","primary_auditor":"T_Human"},
    {"id":"O1","label":"3/4 of 60","content":"(3/4)x60=45","abstraction_level":"OPERATION",
    "primary_auditor":"T_Auto"},
    {"id":"O2","label":"Half class","content":"60/2=30",
    "abstraction_level":"OPERATION","primary_auditor":"T_Auto"},
    {"id":"O3","label":"Remaining","content":"60-30=30 remaining",
    "abstraction_level":"OPERATION","primary_auditor":"T_Auto"},
    {"id":"O4","label":"Convert votes","content":"1/5 of 25=5",
    "abstraction_level":"OPERATION","primary_auditor":"T_Auto"},
    {"id":"O5","label":"Total votes","content":"30+5+5=40",
    "abstraction_level":"OPERATION","primary_auditor":"T_Auto"},
    {"id":"O6","label":"Votes gap","content":"45-40=5",
    "abstraction_level":"OPERATION","primary_auditor":"T_Auto"},
    {"id":"O7","label":"Final answer","content":"Alec needs 5 more votes.","abstraction_level":"OPERATION","primary_auditor":"T_Auto"}
  ],
  "edges": [
    {"from":"G1","to":"S1","relationship":"decomposes_to"},
    {"from":"S1","to":"T1","relationship":"decomposes_to"},
    {"from":"S1","to":"T2","relationship":"decomposes_to"},
    {"from":"S1","to":"T3","relationship":"decomposes_to"},
    {"from":"T1","to":"O1","relationship":"decomposes_to"},
    {"from":"T2","to":"O2","relationship":"decomposes_to"},
    {"from":"T2","to":"O3","relationship":"decomposes_to"},
    {"from":"T3","to":"O4","relationship":"decomposes_to"},
    {"from":"T3","to":"O5","relationship":"decomposes_to"},
    {"from":"O5","to":"ST1","relationship":"enables"},
    {"from":"O6","to":"ST1","relationship":"validates"},
    {"from":"ST1","to":"O7","relationship":"decomposes_to"}
  ],
  "metadata": {
    "total_nodes": 13,
    "total_edges": 12,
    "auditor_distribution": {"T_Auto":7,"T_LLM":4,"T_Human":2}
  }
}
\end{lstlisting}

\begin{figure}[htp!]
    \centering
    \vspace{-0.4em}
    \includegraphics[width=1.0\textwidth]{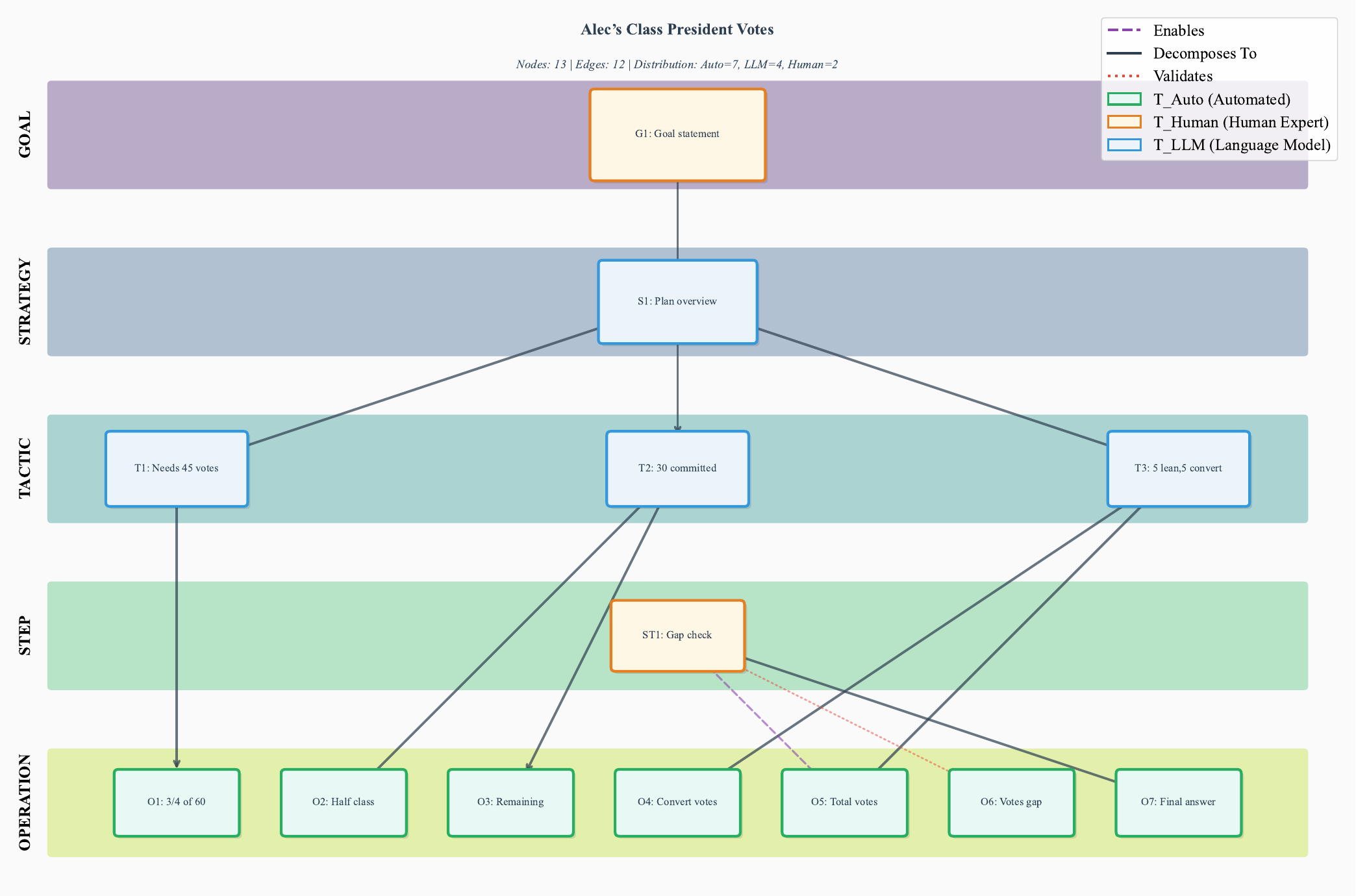}
    \vspace{-0.3em}
    \caption{HDAG Example 2 from Deepseek-R1-8B CoT: Math Problem.}
    \vspace{-0.5em}
\end{figure}

\newpage
\subsubsection{GPT-OSS-20B CoT Decomposition}
\begin{lstlisting}[style=plainmarkdown]
{
  "title": "Alec's Votes",
  "nodes": [
    {"id":"G1","label":"Goal statement","content":"Alec wants 3/4 of 60 = 45 votes.","abstraction_level":"GOAL","type":"objective_statement",
    "primary_auditor":"T_Human"},
    {"id":"S1","label":"Plan overview","content":"Compute target, count votes, add converts, compare gap.","abstraction_level":"STRATEGY","primary_auditor":"T_LLM"},
    {"id":"T1","label":"Needs 45 votes","content":"Three-quarters of 60 is 45.","abstraction_level":"TACTIC","primary_auditor":"T_LLM"},
    {"id":"T2","label":"30 committed","content":"Half of 60 already pledged = 30.","abstraction_level":"TACTIC",
    "primary_auditor":"T_LLM"},
    {"id":"T3","label":"5 converted","content":"From 25 opposing, 1/5 switched = 5.","abstraction_level":"TACTIC",
    "primary_auditor":"T_LLM"},
    {"id":"T4","label":"Total 35","content":"30 committed + 5 converted = 35 total.","abstraction_level":"TACTIC",
    "primary_auditor":"T_LLM"},
    {"id":"ST1","label":"Gap check","content":"45-35 = 10 more needed.","abstraction_level":"STEP",
    "primary_auditor":"T_Human"},
    {"id":"O1","label":"3/4 of 60","content":"(3/4)x60=45",
    "abstraction_level":"OPERATION","primary_auditor":"T_Auto"},
    {"id":"O2","label":"Half class","content":"1/2x60=30",
    "abstraction_level":"OPERATION","primary_auditor":"T_Auto"},
    {"id":"O3","label":"Remaining","content":"60-30=30 remain",
    "abstraction_level":"OPERATION","primary_auditor":"T_Auto"},
    {"id":"O4","label":"Others 25","content":"30-5=25",
    "abstraction_level":"OPERATION","primary_auditor":"T_Auto"},
    {"id":"O5","label":"Convert 5","content":"1/5x25=5",
    "abstraction_level":"OPERATION","primary_auditor":"T_Auto"},
    {"id":"O6","label":"Total votes","content":"30+5=35",
    "abstraction_level":"OPERATION","primary_auditor":"T_Auto"},
    {"id":"O7","label":"Votes gap","content":"45-35=10",
    "abstraction_level":"OPERATION","primary_auditor":"T_Auto"},
    {"id":"O8","label":"Final answer","content":"Alec needs 10 more votes.",
    "abstraction_level":"OPERATION","primary_auditor":"T_Auto"}
  ],
  "edges": [
    {"from":"G1","to":"S1","relationship":"decomposes_to"},
    {"from":"S1","to":"T1","relationship":"decomposes_to"},
    {"from":"S1","to":"T2","relationship":"decomposes_to"},
    {"from":"S1","to":"T3","relationship":"decomposes_to"},
    {"from":"S1","to":"T4","relationship":"decomposes_to"},
    {"from":"T1","to":"O1","relationship":"decomposes_to"},
    {"from":"T2","to":"O2","relationship":"decomposes_to"},
    {"from":"T2","to":"O3","relationship":"decomposes_to"},
    {"from":"T3","to":"O4","relationship":"decomposes_to"},
    {"from":"T3","to":"O5","relationship":"decomposes_to"},
    {"from":"T4","to":"O6","relationship":"decomposes_to"},
    {"from":"O6","to":"ST1","relationship":"enables"},
    {"from":"O7","to":"ST1","relationship":"validates"},
    {"from":"ST1","to":"O8","relationship":"decomposes_to"}
  ],
  "metadata": {
    "total_nodes": 15,
    "total_edges": 14,
    "auditor_distribution": {"T_Auto":8,"T_LLM":4,"T_Human":2}
  }
}
\end{lstlisting}
\begin{figure}[htp!]
    \centering
    \vspace{-0.4em}
    \includegraphics[width=1.0\textwidth]{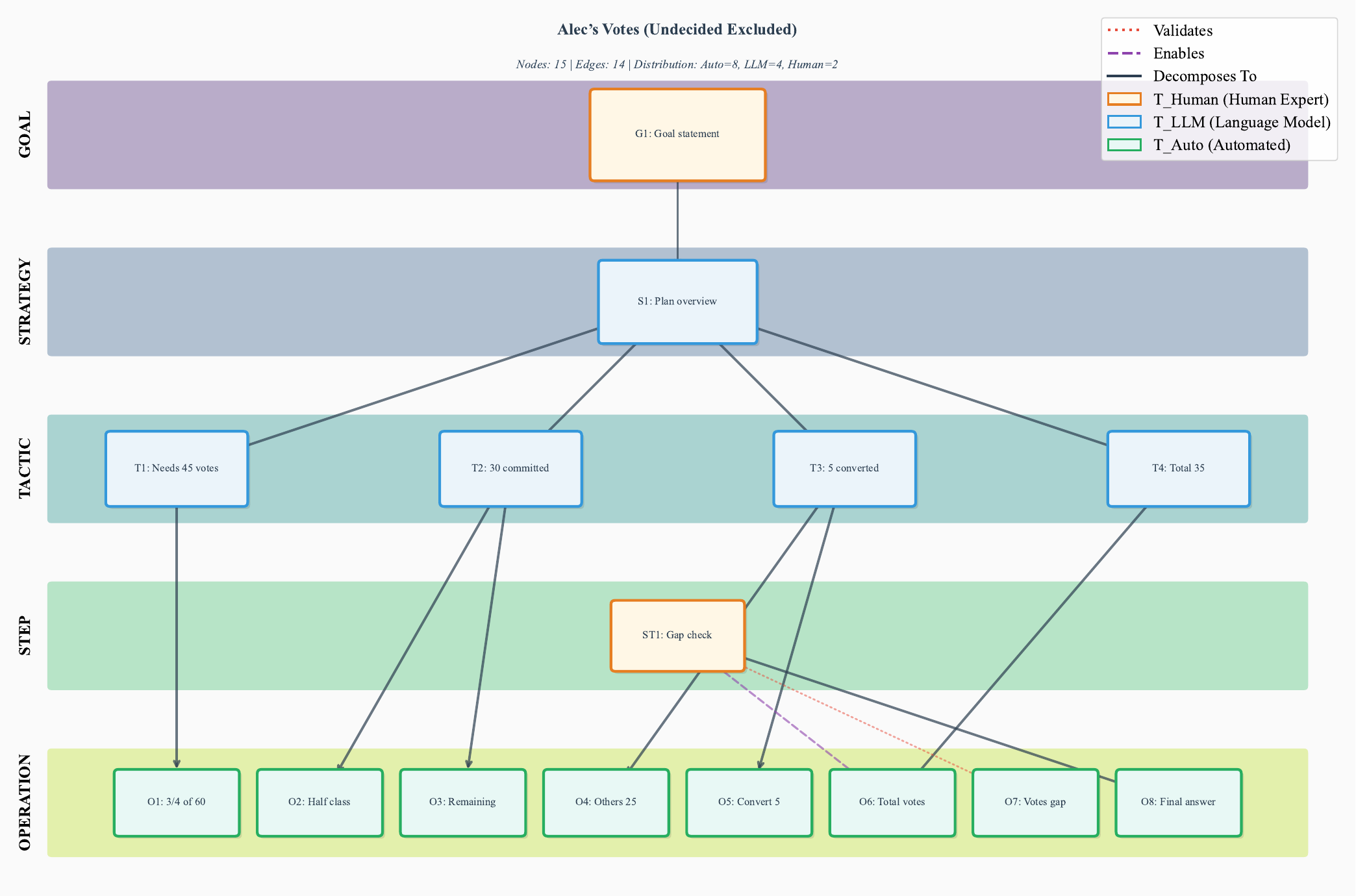}
    \vspace{-0.3em}
    \caption{HDAG Example 2 from GPT-OSS-20B CoT: Math Problem.}
    \vspace{-0.5em}
\end{figure}

\clearpage
\subsection{Datasets}\label{sec:dataset}

We utilize Chain-of-Thought (CoT) reasoning datasets spanning multiple domains following recent surveys \cite{tan2024large}. Our evaluation uses a carefully curated multi-domain dataset designed to test bias mitigation and auditing effectiveness across diverse reasoning tasks.
\paragraph{Primary Evaluation Dataset.}
MMLU-Pro-CoT-Train~\citep{mmlu_pro_cot_labeled}: 200 samples with ground truth annotations for individual reasoning steps and final answers across professional domains, enabling systematic evaluation of correctness and faithfulness at both step and trace levels. This dataset provides comprehensive coverage of advanced knowledge areas including engineering, mathematics, physics, chemistry, biology, and computer science, making it ideal for evaluating reasoning trace quality in technical domains.
\paragraph{Multi-Domain Bias Evaluation Dataset.}
Our curated dataset of 300 questions (50 per domain) sourced from established HuggingFace datasets across 6 diverse domains with comprehensive ground truth validation:
\begin{itemize}
\item \textbf{Medical} (\texttt{medmcqa})~\citep{medmcqa}: Clinical multiple-choice questions covering medical diagnosis, treatment protocols, pharmacology, and pathophysiology. This dataset represents real medical licensing exam questions, providing authentic clinical reasoning scenarios that require domain expertise and careful step-by-step analysis of patient presentations, differential diagnoses, and treatment decisions.
\item \textbf{Mathematics} (\texttt{gsm8k})~\citep{cobbe2021gsm8k}: Grade school arithmetic and algebra word problems requiring multi-step mathematical reasoning. These problems test fundamental quantitative reasoning skills including problem decomposition, arithmetic operations, proportional reasoning, and logical sequencing. The dataset emphasizes practical mathematical applications in everyday contexts such as financial calculations, measurement conversions, and basic geometry.
\item \textbf{Science} (\texttt{sciq})~\citep{SciQ}: Multi-disciplinary science questions spanning physics, chemistry, biology, and earth sciences. Questions cover fundamental scientific concepts, experimental methodology, and theoretical principles. This dataset tests scientific reasoning capabilities including hypothesis formation, experimental design interpretation, causal relationship identification, and application of scientific laws across different domains.
\item \textbf{Common Sense} (\texttt{commonsense\_qa}~\citep{commonsenseqa}): Everyday reasoning questions testing implicit knowledge about social situations, physical properties, causal relationships, and typical human behavior. These questions evaluate the model's ability to apply common-sense knowledge that is typically acquired through general life experience rather than formal education, including understanding of social norms, object properties, and cause-and-effect relationships.
\item \textbf{Humanities} (\texttt{squad})~\citep{squad}: Reading comprehension tasks requiring factual extraction and inference from historical, literary, and cultural texts. Questions test advanced comprehension skills including information synthesis, contextual interpretation, implicit meaning recognition, and factual accuracy verification across diverse humanistic domains.
\item \textbf{Human Evaluation Subset}: 10 carefully selected math problems from \texttt{openai/gsm8k} used specifically for human-in-the-loop experiments with 15 PhD-level auditors, enabling direct comparison between human expert judgment and automated auditing approaches.
All datasets undergo comprehensive quality validation including ground truth verification, format consistency checking, content quality assessment, and deduplication.
\end{itemize}

\clearpage
\section{Theoretical Results}\label{sec:appex_theory}
\subsection{Consensus: Seat Layer, Segment Layer, and Trace Layer}
In this section we provide the theoretical analysis of TRUST present in \cref{sec:method} by characterizing the statistical safety against malicious auditors and the economics incentive/decentive gurantees for good and bad actors.

We analyze in three layers:
\begin{enumerate}[leftmargin=*,itemsep=0.3em]
\item \textbf{Seat layer.}  
      Within segment $s$, the $k_{t(s)}$ seats vote independently;
      computer seats are noiseless, LLMs/humans have error~$\epsilon_t$,
      and human seats may be adversarial w.p.\ $\rho_{\mathrm H}$.
\item \textbf{Segment layer.}  
      Define the \emph{segment pass indicator}
      \begin{align}
        B_s
        =\mathbf1\!\Bigl[\#\{\text{correct votes}\}\ge q_{t(s)}\Bigr],
      \end{align}
      where $q_{t}=\lceil\tau\,k_t\rceil$ is the per‐type quorum.
      Compute the exact pass probability $p_s=\Pr[B_s=1]$.
\item \textbf{Trace layer.}  
      Weight each segment outcome by $w_{t(s)}$, set
      \begin{align}
        W=\sum_{s=1}^{S} w_{t(s)}\,B_s,
        \quad
        W_\beta=\beta\sum_{s}w_{t(s)}.
      \end{align}
      Bound $\Pr[W<W_\beta]$ by Hoeffding and Chernoff using the $p_s$.
\end{enumerate}

Next, we derive the exact pass rate for type $t$ in \cref{lem:exact_type}

\begin{lemma}[Exact pass probability for type $t$]
\label{lem:exact_type}
For a segment type $t$ and parameters $(k_t,\epsilon_t,\rho_t)$
with $\rho_{\mathrm C}=\rho_{\mathrm L}=0$.
Then
\begin{align}
  p_t
  = \Pr\bigl[B_s=1\bigr]
  = \sum_{m=0}^{k_t}\binom{k_t}{m}\rho_t^m(1-\rho_t)^{k_t-m}
    \sum_{c=q_t}^{\,k_t-m}
      \binom{k_t-m}{c}
      (1-\epsilon_t)^c\,\epsilon_t^{\,k_t-m-c},
\end{align}
where $m$ malicious seats always vote wrong,
and among the $k_t-m$ honest seats $c$ vote correctly.
\end{lemma}

\begin{proof}
First choose $m$ malicious seats (Binomial$(k_t,\rho_t)$),
then among the remaining $k_t-m$ honest seats count $c\ge q_t$
correct votes (Binomial$(k_t-m,1-\epsilon_t)$).
\end{proof}

\label{sec:trace_bounds}
On trace-level, we aggregate the results on segmend-level and denote $W=\sum_{s=1}^S w_{t(s)}B_s$ where $w_t$ is the weight for segment of type $t$.
With $B_s\sim\mathrm{Bernoulli}(p_s)$ independent and weights
$w_s=w_{t(s)}$, the first two moment of weighted trace is given by
\begin{align}
  \mu_{\text{vote}} &= \mathbb E[W]
      = \sum_{s=1}^{S} w_s\,p_s,
  \\
  \sigma_{\text{vote}}^{2} &= \operatorname{Var}(W)
      = \sum_{s=1}^{S} w_s^{2}\,p_s(1-p_s)
      \;\le\;
      \vphantom{\sum_{s}^{S}}
        \sum_{s=1}^{S} w_s^{2} \eqqcolon \sigma_{\max}^{2}.
        \label{eq:sig_max}
\end{align}

\begin{proposition}[Hoeffding and Chernoff bound]
\label{prop:wedge_bound}
For any trace-level quorum threshold $\beta\in(0,1)$ define
$
  W_\beta=\beta\sum_{s} w_s
$
and let $W=\sum_{s}w_s B_s$.
Then
\begin{align}
\Pr\!\bigl[ W < W_\beta \bigr]
\;\le\;
\underbrace{%
  \exp\!\Bigl[
    -2\,(\mu_{\text{vote}} - W_\beta)^{2}\bigl/\sigma_{\max}^{2}
  \Bigr]
}_{\text{Hoeffding}}
\;\land\;
\underbrace{%
  \min_{\lambda>0}
    \exp\!\Bigl(
      \lambda W_\beta
      +\sum_{s=1}^{S}
         \ln\!\bigl(
           p_s\,e^{-\lambda w_s} + (1-p_s)
         \bigr)
    \Bigr)
}_{\text{Chernoff}}.
\end{align}
\end{proposition}

\begin{proof}
We proof the bounds separate as follow.
\begin{enumerate}
    \item Hoeffding bound.
    Each summand $X_s:=w_s B_s$ satisfies $0\le X_s\le w_s$.
    Applying Hoeffding’s inequality to
    $\sum_{s}(X_s-\mathbb E X_s)$ yields the first brace with denominator
    ${\sum_s w_s^{2}}=\sigma_{\max}^{2}$ in \eqref{eq:sig_max}.
    \item Chernoff bound.
    For any $\lambda>0$,
    \begin{align}
      \Pr[W<W_\beta]
      \;\le\;
      e^{\lambda W_\beta}\,
      \mathbb E\!\bigl[e^{-\lambda W}\bigr]
      =
      e^{\lambda W_\beta}
      \prod_{s=1}^{S}\!
        \bigl(
          p_s\,e^{-\lambda w_s} + (1-p_s)
        \bigr),
    \end{align}
    because the $B_s$ are independent.
    Minimising the RHS over $\lambda$ gives the second brace.
\end{enumerate}
\end{proof}

\subsection{Economic Layer: Staking, Reputation, Rewards, and Slashing}
\label{sec:econ}

\begin{table}[htp!]
\label{tab:notation}
\centering
\renewcommand{\arraystretch}{1.15}
\caption{
{Table of Notations.
}}
\begin{tabular}{@{}ll@{}}
\toprule
\multicolumn{2}{@{}l}{\textbf{Indices \& random counts}} \\ \midrule
$S$                              & Total number of segments in a trace                               \\
$N_T\sim\mathrm{Poisson}(\lambda T)$ & \# segments audited in horizon $[0,T]$                           \\
$t(s)\in\{\mathrm C,\mathrm L,\mathrm H\}$ & Auditor type of segment $s$                                  \\[3pt]

\multicolumn{2}{@{}l}{\textbf{Per-segment vote variables (stat.\ layer)}} \\ \midrule
$k_t$                            & \# seats of type $t$ in a segment                                 \\
$q_t=\lceil \tau k_t\rceil$      & Per-type quorum ($\tau$: vote threshold)                          \\
$B_s\in\{0,1\}$                  & Segment pass indicator                                            \\
$p_s=\Pr[B_s=1]$                 & Segment pass probability                                          \\
$w_t$                            & Weight of segment type $t$                                        \\[3pt]

\multicolumn{2}{@{}l}{\textbf{Trace aggregation}} \\ \midrule
$W=\sum_{s} w_{t(s)}B_s$         & Weighted pass total (one segment)                                 \\
$W_\beta=\beta\sum_{s} w_{t(s)}$ & Trace-level quorum threshold (\(\beta\in(0,1)\))                   \\
$\mu_{\text{vote}}=\mathbb E[W]$               & Mean of $W$                                                        \\
$\sigma_{\text{vote}}^{2}$       & $\sup\operatorname{Var}(W)$ (single segment)                       \\[3pt]

\multicolumn{2}{@{}l}{\textbf{Human pay-off variables (econ.\ layer)}} \\ \midrule
$R$                              & Reward for a correct vote                                         \\
$P$                              & Penalty if slashed                                                \\
$p_{\text{slash}}(r)$            & Slash prob.\ on error, reputation $r$ \([\;p_{\min},p_{\max}]\)    \\
$\epsilon_{\mathrm H}$        & Honest human error rate                                           \\
$\delta$                         & Design constant: min.\ loss per malicious seat                    \\
$X_i\in\{-P,0,R\}$               & Net pay-off for seat $i$ on one segment                            \\
$\mu_{\mathrm H}(r)$             & $\mathbb E[X_i]$ for honest seat with reputation $r$              \\
$\mu_{\min}$                     & $\min_{r}\mu_{\mathrm H}(r)$                                      \\
$b$                              & Upper range used in Bernstein (default \(b = R\))                 \\
$\sigma_{\mathrm{H}}^{2}$        & $\sup_{r}\operatorname{Var}[X_i]\le(R+P)^2/4$                      \\[3pt]

\multicolumn{2}{@{}l}{\textbf{Centred increments for MGF bounds}} \\ \midrule
$Y_i=X_i-\mu_{\min}$             & Honest centred increment (\(\mathbb E[Y_i]\!\ge\!0,\,Y_i\!\le\!b\)) \\
$Z_i=X_i+\delta P$               & Malicious centred increment (\(\mathbb E[Z_i]\!\ge\!0\))          \\[3pt]

\multicolumn{2}{@{}l}{\textbf{Cumulative pay-offs}} \\ \midrule
$U_{\mathrm{hon}}(T)=\sum_{i=1}^{N_T}X_i$ & Total pay-off (honest seat) in $[0,T]$                    \\
$U_{\mathrm{mal}}(T)$                     & Total pay-off (malicious seat) in $[0,T]$                 \\[3pt]

\multicolumn{2}{@{}l}{\textbf{Process rates}} \\ \midrule
$\lambda$                        & Segment intensity (segments per unit time)                        \\
$T$                              & Time horizon                                 \\[3pt]

\bottomrule
\end{tabular}
\end{table}

\paragraph{Reputation-Weighted Slashing}
Each human seat~\(i\) maintains a reputation \(r_i(t)\in[0,1]\), updated
after every segment via
\begin{align}\label{eq:reputation}
  r_i(t+1)
  \;=\;
  (1-\gamma)\,r_i(t)+\gamma\,\mathbf1[\text{vote correct}],
  \qquad
  \gamma\in(0,1].
\end{align}
On an \emph{incorrect} vote the seat is slashed with probability
\begin{equation}\label{eq:pslash}
  p_{\mathrm{slash}}(r)
  \;=\;
  p_{\min}+(p_{\max}-p_{\min})(1-r),
  \qquad
  0<p_{\min}<p_{\max}\le1.
\end{equation}
Thus low-reputation seats face a higher risk of being slashed.

\paragraph{Per-Segment Pay-off}
Let \(X_i\in\{-P,\,0,\,R\}\) be the \emph{net} pay-off of seat~\(i\) on one
segment:
\[
  X_i
  =
  \begin{cases}
    R               & \text{correct vote},\\
   -P               & \text{incorrect and with slashing probability $p_{\mathrm{slash}}$},\\
    0               & \text{incorrect and with not slashing probability $1-p_{\mathrm{slash}}$}.
  \end{cases}
\]

With honest human error rate \(\epsilon_{\mathrm H}\), the expected payoff per segment with reputation $r$ is
\begin{equation}\label{eq:muH}
  \mu_{\mathrm H}(r)
  \;:=\;
  \mathbb E[X_i]
  \;=\;
  (1-\epsilon_{\mathrm H})R
  -\epsilon_{\mathrm H}\,P\,p_{\mathrm{slash}}(r).
\end{equation}
Computer and LLM seats are verifiable, hence always correct and omitted
from incentive analysis.

We need two global constant for deriving hte Bernstein-type moment-generating function (MGF) inequalities:

\begin{enumerate}[leftmargin=*, itemsep=0.25em]
    \item \textbf{Range bound} \(b\) on the \emph{centred increment}:
    \[
    Y_i := X_i - \mathbb{E}[X_i], \qquad Y_i \leq b.
    \]
    \item \textbf{Variance bound} \(\sigma_{\mathrm{H}}^2\), the maximal variance across all reputation states:
    \[
    \sigma_{\mathrm{H}}^2 := \sup_{r \in [0,1]} \operatorname{Var}[X_i].
    \]
\end{enumerate}

\paragraph{Range bound.}
By construction, the largest positive realisation of $X_i$ is $R$, while the minimal expected payoff $\mathbb{E}[X_i]$ reduces the centred increment. To preserve valid MGF domain, we conservatively set
\begin{align}\label{eq:b_bound}
b := R.
\end{align}

\paragraph{Variance bound.}
Given $X_i \in \{-P, 0, R\}$ with honest human error rate $\epsilon_{\mathrm{H}}$ and slashing probability $p_{\mathrm{slash}}(r)$, we define the global variance bound as
\begin{align}\label{eq:var_bound}
\sigma_{\mathrm{H}}^2(r) := 
\sup_{r \in [0,1]} \left[
(1-\epsilon_{\mathrm{H}}) R^2 + \epsilon_{\mathrm{H}} p_{\mathrm{slash}}(r) P^2 - \left( (1-\epsilon_{\mathrm{H}})R - \epsilon_{\mathrm{H}} p_{\mathrm{slash}}(r)P \right)^2 \right].
\end{align}

Before state the main result in \cref{thm:safty_profit}, we provide some auxillary lemmas.
\begin{lemma}[MGF of a Bounded Centred R.V.]\label{lem:bernstein-mgf}
Let \(W\) satisfy \(\mathbb E[W]=0\), \(\mathbb E[W^{2}]=\sigma^{2}\) and
\(W\le b\) a.s.\ with \(b>0\).  Then for any \(\theta\in(0,3/b)\), we have
\begin{align*}
  \mathbb E\!\bigl[e^{\theta W}\bigr]
  \;\le\;
  \exp\!\Bigl(
    \frac{\theta^{2}\sigma^{2}}
         {\,2\bigl(1-\theta b/3\bigr)}
  \Bigr).
\end{align*}
\end{lemma}
\begin{proof}
Follow the usual Bernstein–Bernoulli expansion; details are unchanged from the classic proof and omitted here for brevity.
\end{proof}

\begin{theorem}[Safety–Profitability Guarantee]
\label{thm:safty_profit}
Fix a horizon \(T>0\), a target trace-failure probability
$\epsilon_{\mathrm{target}}\in(0,1)$ and a design constant
$\delta\in(0,1)$.
We have the following two dials to control the safety-profitability.

\begin{itemize}[leftmargin=*, itemsep=0.25em]
\item \textbf{Statistical dial.}  
Let \((k_t,q_t,w_t,\beta)\) be the vote parameters.
Write \(\mu_{\text{vote}}:=\mathbb E[W]\) and
\(\sigma_{\text{vote}}^{2}:=\sup\!\operatorname{Var}(W)\) for
\emph{one trace}.  Require
\[
  \mu_{\text{vote}}-W_\beta
  \;\ge\;
  \sqrt{\tfrac12\,
        \sigma_{\text{vote}}^{2}\,
        \ln\!\frac{\lambda T}{\epsilon_{\mathrm{target}}}},
\tag{S1}\label{eq:S1prime}
\]

\item \textbf{Economic dial.}  
Choose \((R,P,p_{\min},p_{\max})\) such that
\begin{align}
R
>\frac{\epsilon_{\mathrm H}}{1-\epsilon_{\mathrm H}}\,P\,p_{\max},
\quad
p_{\min}
\ge\frac{\delta}{1-\alpha},
\quad
\alpha
:=\frac{P p_{\max}}{R+P p_{\max}}.
\tag{E1}\label{eq:E1_fixed}
\end{align}
\end{itemize}

With the expected minimum earn per round
\(
  \displaystyle
  \mu_{\min}:=(1-\epsilon_{\mathrm H})R-\epsilon_{\mathrm H}P p_{\max}>0,
\)
the following hold:

\begin{enumerate}[label=(\alph*), leftmargin=*, itemsep=0.25em]
\item \textbf{Statistical safety.}
      \(\Pr[\text{trace fails in }[0,T]]\le\epsilon_{\mathrm{target}}.\)

\item \textbf{Honest profitability.}
      \begin{align}
        \Pr[U_{\mathrm{hon}}(T)\le0]
        \;\le\;
        \exp\!\Bigl[
          -\frac{\lambda T\,\mu_{\min}^{2}}
                 {2\sigma_{\mathrm{H}}^{2}+\tfrac23 b\mu_{\min}}
        \Bigr].
      \end{align}

\item \textbf{Malicious loss.}
      \begin{align}
        \Pr[U_{\mathrm{mal}}(T)\ge0]
        \;\le\;
        \exp\!\Bigl[
          -\frac{\lambda T\,(\delta P)^{2}}
                 {2\sigma_{\mathrm{H}}^{2}+\tfrac23 b\delta P}
        \Bigr],
        \quad
        \mathbb E[U_{\mathrm{mal}}(T)]
        \;\le\;
        -\lambda T\,\delta P.
      \end{align}
\end{enumerate}
\end{theorem}

\begin{proof}[Proof of \cref{thm:safty_profit}]

Consider a horizon $T>0$ and $N_T\sim\text{Poisson}(\lambda T)$ for the random number of segments in
$[0,T]$.
We divide the proof for (a), (b) and (c).

\paragraph{(a) Statistical safety.}
A single trace's weighted pass sum $W$ satisfies $0\le W\le
\sum_t w_t$ and
\[
  \mathbb E[W]=\mu_{\text{vote}}, 
  \quad
  \operatorname{Var}(W)\le\sigma_{\text{vote}}^{2}.
\]
Hoeffding’s inequality for bounded independent terms in \cref{prop:wedge_bound} gives, for
any $a>0$,
\[
  \Pr[W<\mu_{\text{vote}}-a]\;\le\;\exp\!\Bigl(-2a^{2}/\sigma_{\text{vote}}^{2}\Bigr).
\]
Instantiate $a=\mu_{\text{vote}}-W_\beta$.  Condition~\eqref{eq:S1prime} rearranges to
\[
  2(\mu_{\text{vote}}-W_\beta)^{2}/\sigma_{\text{vote}}^{2}
  \;\ge\;
  \ln\!\frac{\lambda T}{\epsilon_{\mathrm{target}}},
\]
so
\begin{align}\label{eq:A1}
  p_{\text{trace-fail}}
  :=\Pr[W<W_\beta]
  \;\le\;
  \frac{\epsilon_{\mathrm{target}}}{\lambda T}.
\end{align}

Traces arrive independently according to the Poisson process, so
\begin{align*}
  \Pr[\text{at least one trace fails in }[0,T]]
  =\Pr\!\bigl[\,\exists\,\text{trace with }W<W_\beta\bigr]
  \le
  \mathbb E[N_T]\;p_{\text{trace-fail}}
  =\lambda T\,p_{\text{trace-fail}}.
\end{align*}
Inserting \eqref{eq:A1} yields
\(\Pr[\text{trace fails in }[0,T]]\le\epsilon_{\mathrm{target}}\),
which completes the proof for part~(a).

\paragraph{(b) Honest profitability.}

We analyse the cumulative pay-off
\(
  U_{\mathrm{hon}}(T)
  :=\sum_{i=1}^{N_T}X_i
\)
for an honest human seat.

\textbf{Step 1. Center and Bound each Increment.}
Let
\[
  \mu_{\min}
  :=(1-\epsilon_{\mathrm H})R-\epsilon_{\mathrm H}P p_{\max}>0
  \quad
\]
Define centred variables
\(
  Y_i:=X_i-\mu_{\min}.
\)
Then
\begin{align*}
  \mathbb E[Y_i]=0,
  \quad
  Y_i\le b:=R,
  \quad
  \operatorname{Var}(Y_i)\le\sigma_{\mathrm{H}}^{2},
\end{align*}
where $b$ and $\sigma_{\mathrm{H}}$ are defined in \eqref{eq:b_bound} and \eqref{eq:var_bound}.

\textbf{Step 2. Moment-Generating Function Bound.}
From \cref{lem:bernstein-mgf} with $W=Y_i$, for any
$\theta\in(0,3/b)$
\begin{align}
  \mathbb E\!\bigl[e^{\theta Y_i}\bigr]
  \le
  \exp\!\Bigl(
        \frac{\theta^{2}\sigma_{\mathrm{H}}^{2}}
             {2\bigl(1-\theta b/3\bigr)}
      \Bigr).
\label{eq:A2}
\end{align}

\paragraph{Step 3. Chernoff Bound for the Random Sum.}
Let $U_{\mathrm{hon}}(T)$ denote the honest agent's total payoff over the random number $N_T$ of rounds in $[0,T]$:
\[
  U_{\mathrm{hon}}(T) = \sum_{j=1}^{N_T} X_j = N_T \mu_{\min} + \sum_{j=1}^{N_T} Y_j,
\]
where $Y_j := X_j - \mu_{\min}$ are i.i.d.\ random variables.

Now we consider the probability that the cumulative payoff is non-positive $ \Pr[U_{\mathrm{hon}}(T)\le0]$.

First, we condition on the total number of rounds $N_T = n$.
\begin{align*}
  \Pr[U_{\mathrm{hon}}(T) \le 0 \mid N_T = n]
  &= \Pr\left[\sum_{j=1}^n X_j \le 0 \right] \\
  &= \Pr\left[ n\mu_{\min} + \sum_{j=1}^n Y_j \le 0 \right] \\
  &= \Pr\left[ \sum_{j=1}^n Y_j \le - n\mu_{\min} \right].
\end{align*}

Apply Chernoff's (exponential Markov) inequality:
For any $\theta > 0$,
\begin{align*}
  \Pr\left[ \sum_{j=1}^n Y_j \le - n\mu_{\min} \right]
  &= \Pr\left[ e^{-\theta \sum_{j=1}^n Y_j} \ge e^{\theta n\mu_{\min}} \right] \\
  &\le e^{-\theta n\mu_{\min}}\, \mathbb{E}\left[ e^{\theta \sum_{j=1}^n Y_j} \right] \\
  &= e^{-\theta n\mu_{\min}} \left( \mathbb{E}\left[ e^{\theta Y_1} \right] \right)^n,
\end{align*}
where the last equality uses independence of the $Y_j$.

Now, remove the conditioning by averaging over all possible $n$.
Recall that $N_T \sim \mathrm{Poisson}(\lambda T)$, so
\begin{align*}
  \Pr[U_{\mathrm{hon}}(T) \le 0] = \sum_{n=0}^{\infty} \Pr[U_{\mathrm{hon}}(T) \le 0 \mid N_T = n]\, \Pr[N_T = n].
\end{align*}
Using the bound above and properties of exponents and linearity of expectation
\begin{align*}
  \Pr[U_{\mathrm{hon}}(T) \le 0]
  &\le \sum_{n=0}^{\infty}
        \left[
          e^{-\theta n\mu_{\min}} \left( \mathbb{E}[e^{\theta Y_1}] \right)^n
        \right]
        \Pr[N_T = n] \\
  &= \mathbb{E}\left[
       \left( e^{-\theta \mu_{\min}} \mathbb{E}[e^{\theta Y_1}] \right)^{N_T}
     \right].
\end{align*}

The Poisson moment-generating formula:
For any $z>0$,
$ \mathbb{E}[z^{N_T}] = \exp\left( \lambda T (z - 1) \right), $
where $z = e^{-\theta \mu_{\min}} \mathbb{E}[e^{\theta Y_1}]$:
\[
  \Pr[U_{\mathrm{hon}}(T) \le 0]
  \le
  \exp\left(
      \lambda T \left(
         e^{-\theta\mu_{\min}} \mathbb{E}[e^{\theta Y_1}]
         - 1
      \right)
      \right).
\]

Finally, upper bound $\mathbb{E}[e^{\theta Y_1}]$ using Bernstein's MGF lemma~\eqref{eq:A2}:
\[
  \mathbb{E}[e^{\theta Y_1}]
  \le \exp\left( \frac{\theta^2 \sigma_{\mathrm{H}}^2}{2(1-\theta b/3)} \right).
\]
For small enough $\theta$, Taylor expand $e^{-\theta\mu_{\min}}$ and combine exponents to obtain
\begin{align}\label{eq:A3}
  \Pr[U_{\mathrm{hon}}(T)\le0]
  \le
  \exp\!\biggl(
       \lambda T
       \Bigl\{
         -\theta\mu_{\min}
         +\frac{\theta^{2}\sigma_{\mathrm{H}}^{2}}
               {2(1-\theta b/3)}
       \Bigr\}
       \biggr).
\end{align}
The optimal $\theta$ is chosen in the next step.

\textbf{Step 4. Optimise $\theta$.}
Set
\(
  g(\theta):=-\theta\mu_{\min}
             +\frac{\theta^{2}\sigma_{\mathrm{H}}^{2}}
                    {2(1-\theta b/3)}.
\)
Let \(t:=\theta b/3\in(0,1)\);
then \(\theta=3t/b\) and
\[
  g(t)
  =-\frac{3t\mu_{\min}}{b}
   +\frac{9t^{2}\sigma_{\mathrm{H}}^{2}}
         {2b^{2}(1-t)}.
\]
Differentiate:
\[
  g'(t)
  =-\frac{3\mu_{\min}}{b}
   +\frac{9\sigma_{\mathrm{H}}^{2}}{2b^{2}}
     \,\frac{2t-1}{(1-t)^{2}}.
\]
Solve \(g'(t)=0\) to obtain
\[
  t^\star
  =1-\frac{1}{\sqrt{1+2b\mu_{\min}/(3\sigma_{\mathrm{H}}^{2})}}.
\]
Plug back:
\begin{align}
  g(t^\star)
  =-\frac{\mu_{\min}^{2}}
         {2\sigma_{\mathrm{H}}^{2}
          +\tfrac23 b\mu_{\min}}.
\label{eq:A4}
\end{align}

\textbf{Step 5. Combine All.}
Combine \eqref{eq:A3} and \eqref{eq:A4} to get
\[
  \Pr[U_{\mathrm{hon}}(T)\le0]
  \le
  \exp\!\Bigl(
        -\frac{\lambda T\,\mu_{\min}^{2}}
               {2\sigma_{\mathrm{H}}^{2}
                +\tfrac23 b\mu_{\min}}
      \Bigr),
\]
This complete the proof of claim (b).

\paragraph{(c) Malicious loss.}

A malicious seat flips its pay-off distribution, the proof closely follows claim (b).

\textbf{Step 1. Negative Mean.}
Conditions (E1)–(E2) force $ \mathbb E[X_i] \le -\delta P<0$.
Define centred variables $ Z_i:=X_i+\delta P $
so that $\mathbb E[Z_i]=0$ and \(Z_i\le b\).

\textbf{Step 2. Apply \cref{lem:bernstein-mgf}.}
Replace \(\mu_{\min}\) by \(\delta P\) throughout Steps 2–5 above.
No other constant changes.  
The give us 
\begin{align*}
  \Pr[U_{\mathrm{mal}}(T)\ge0]
  \le
  \exp\!\Bigl(
        -\frac{\lambda T\,(\delta P)^{2}}
               {2\sigma_{\mathrm{H}}^{2}
                +\tfrac23 b\,\delta P}
      \Bigr),
\end{align*}
proving the tail in (c).

\textbf{Step 3. Expected Loss.}
Linearity of expectation with $N_T\sim\text{Poisson}(\lambda T)$ gives
\[
  \mathbb E[U_{\mathrm{mal}}(T)]
  =\lambda T\,\mathbb E[X_i]
  \le -\lambda T\,\delta P,
\]
completing the proof of claim (c).

\end{proof}

\end{document}